\pdfoutput=1
\documentclass[11pt]{article} 
\usepackage[utf8]{inputenc}
\usepackage{amsmath, amsfonts, amsthm, amssymb, graphicx}
\usepackage{mathtools,verbatim}
\usepackage{bbm}
\usepackage{bm}
\usepackage{mathrsfs}  
\usepackage{xr}
\usepackage{thm-restate}
\usepackage{xfrac}
\usepackage{breqn}  
\usepackage{etoolbox}
\usepackage{enumitem}
\usepackage{natbib}
\usepackage{pdfpages,tikz}
\usetikzlibrary{bayesnet}
\usepackage[noend]{algpseudocode}
\usepackage[linesnumbered,ruled]{algorithm2e}
\usepackage{cleveref}

\usepackage{fancyhdr, lastpage}

\newtoggle{nips}

\usepackage[vmargin=1.00in,hmargin=1.00in,centering,letterpaper]{geometry}

\newcommand{\calA}{\mathcal{A}}

\newcommand{\calD}{\mathcal{D}}
\newcommand{\calE}{\mathcal{E}}
\newcommand{\calF}{\mathcal{F}}

\newcommand{\calM}{\mathcal{M}}

\newcommand{\calO}{\mathcal{O}}

\newcommand{\calS}{\mathcal{S}}

\newcommand{\calZ}{\mathcal{Z}}

\newcommand{\matQ}{\mathbf{Q}}

\newcommand{\matV}{\mathbf{V}}

\newcommand{\matZ}{\mathbf{Z}}

\newcommand{\maxs}[1]{{\color{magenta}{#1}}}
\newcommand\numberthis{\addtocounter{equation}{1}\tag{\theequation}}
\newcommand{\rhatk}{\rhat_k}
\newcommand{\boundfut}{\mathbf{B}^{\mathrm{fut}}}
\newcommand{\boundfutk}{\boundfut_k}
\newcommand{\Mtil}{\widetilde{M}}
\newcommand{\Hsample}{H_{\mathrm{sample}}}

\newcommand{\fmax}{f_{\max}}
\newcommand{\Lfactortil}{\widetilde{\Lfactor}}
\newcommand{\boundlead}{\mathbf{B}^{\mathrm{lead}}}

\newcommand{\secemph}[1]{\emph{#1}}
\newcommand{\gapinf}{\gap_{\infty}}

\newcommand{\vecsig}{\vec{\sigma}}

\newcommand{\Expsig}{\Exp^{\vecsig}}
\newcommand{\Delst}{\Delta^{\star}}
\newcommand{\Actst}{\actions^{\star}}
\newcommand{\boldsig}{\boldsymbol{\sigma}}

\newcommand{\regret}{\mathrm{Regret}}

\DeclarePairedDelimiter\ceil{\lceil}{\rceil}
\DeclarePairedDelimiter\floor{\lfloor}{\rfloor}

{
 \theoremstyle{plain}
      
}
\newtheorem{nono-theorem}{Theorem}[]
\theoremstyle{plain}
\newtheorem{thm}{Theorem}[section]
\newtheorem{claim}[thm]{Claim}
\newtheorem{lem}[thm]{Lemma}
\newtheorem{cor}[thm]{Corollary}
\newtheorem{fact}[thm]{Fact}

\newtheorem{prop}[thm]{Proposition}
\theoremstyle{definition}
\newtheorem{defn}{Definition}[section]

\newtheorem{rem}{Remark}[section]

\renewcommand{\Pr}{\mathbb{P}}
\newcommand{\Exp}{\mathbb{E}}
\newcommand{\Var}{\mathrm{Var}}

\newcommand{\unifsim}{\overset{\mathrm{unif}}{\sim}}

\newcommand{\KL}{\mathrm{KL}}
\newcommand{\kl}{\mathrm{kl}}

\newcommand{\Exppi}{\Exp^{\pi}}
\newcommand{\Prpi}{\Pr^{\pi}}
\newcommand{\cb}{\mathsf{CB}}
\newcommand{\cbre}{\cb^{\mathrm{R}}}
\newcommand{\cbpr}{\cb^{\mathrm{P}}}

\newcommand{\qpik}{\trueq^{\pik}}
\newcommand{\qpikh}{\trueq_h^{\pik}}

\newcommand{\geover}[1]{\overset{(#1)}{\ge}}
\newcommand{\eqover}[1]{\overset{(#1)}{=}}

\newcommand{\xah}{(x,a,h)}

\newcommand{\vecVarsub}{\vecVar^{\subopt}}
\newcommand{\vecVaropt}{\vecVar^{\opt}}

\newcommand{\nbar}{\overline{n}}

\newcommand{\N}{\mathbb{N}}

\newcommand{\R}{\mathbb{R}}
\newcommand{\I}{\mathbb{I}}

\newcommand{\qst}{\mathbf{Q}^{\star}}

\DeclareMathOperator{\BigOm}{\mathcal{O}}
\newcommand{\BigOh}[1]{\BigOm\left({#1}\right)}
\DeclareMathOperator{\BigOmtil}{\widetilde{\mathcal{O}}}
\newcommand{\BigOhTil}[1]{\BigOmtil\left({#1}\right)}

\newcommand{\threshop}[2]{\operatorname{clip}\left[#2\,|\,#1\right]}
\newcommand{\clip}[2]{\operatorname{clip}\left[#2|#1\right]}

\newcommand{\underbracer}[2]{\underbrace{#2}_{#1}}

\newcommand{\overeq}[1]{\overset{(#1)}{=}}

\newcommand{\pist}{\pi^{\star}}
\newcommand{\pisth}{\pist_h}
\newcommand{\pikh}{\pi_{k,h}}

\newcommand{\pikt}{\pi_{k,t}}
\newcommand{\pik}{\pi_{k}}

\newcommand{\vst}{\truev^{\star}}

\newcommand{\cond}{\kappa}

\newcommand{\gapeff}{\gap_{\mathrm{eff}}}

\newcommand{\states}{\mathcal{S}}
\newcommand{\actions}{\mathcal{A}}
\newcommand{\vecVar}{\overrightarrow{\Var}}

\newcommand{\stopt}{\boldsymbol{\tau}}

\newcommand{\truev}{\mathbf{V}}

\newcommand{\trueq}{\mathbf{Q}}

\newcommand{\vbar}{\overline{\mathbf{V}}}

\newcommand{\bonus}{\mathbf{b}}

\DeclareMathOperator*{\argmax}{arg\,max}

\newcommand{\Alg}{\mathsf{Alg}}

\newcommand{\rhat}{\widehat{r}}

\newcommand{\weight}{\boldsymbol{\omega}}
\newcommand{\weightkh}{\weight_{k,h}}

\newcommand{\wkh}{\weightkh}

\newcommand{\xtat}{(x_t,a_t)}

\newcommand{\samp}{\mathrm{samp}}
\newcommand{\errconst}{\mathbf{c}}

\newcommand{\Zopt}{\calZ_{\termfont{opt}}}
\newcommand{\Zsub}{\calZ_{\termfont{sub}}}
\newcommand{\constM}{C_{\calM}}
\newcommand{\constbarM}{\overline{C}_{\calM}}
\newcommand{\poly}{\mathrm{poly}}

\newcommand{\opt}{\termfont{opt}}
\newcommand{\subopt}{\termfont{sub}}

\newcommand{\Varoptxa}{\Var_{x,a}^{\opt}}
\newcommand{\Varsubxa}{\Var_{x,a}^{\subopt}}

\newcommand{\matZk}{\matZ_k}

\newcommand{\gapclip}{\fullclip{\gap}}

\newcommand{\gapdbar}{\underline{\overline{\gap}}}

\newcommand{\gapcliph}{\gapclip_h}
\newcommand{\gapclipmin}{\gapclip_{\min}}

\newcommand{\termmin}{\mathtt{min}}

\newcommand{\gapmin}{\gap_{\termmin}}
\newcommand{\gap}{\mathtt{gap}}
\newcommand{\gaph}{\gap_h}
\newcommand{\gaphxa}{\gaph(x,a)}

\newcommand{\HeffT}{\Heff_{T}}

\newcommand{\Varterm}{\mathtt{Var}}

\newcommand{\Varxa}{\Varterm^{\star}_{x,a}}
\newcommand{\Varxah}{\Varterm^{\star}_{h,x,a}}
\newcommand{\Varpikxah}{\Varterm^{\pik}_{h,x,a}}
\newcommand{\Varpixah}{\Varterm^{\pi}_{h,x,a}}
\newcommand{\Gterm}{\calG}
\newcommand{\Heff}{\overline{H}}
\newcommand{\Varkxah}{\Varterm^{(k)}_{h,x,a}}

\newcommand{\Varxat}{\Varterm^{\star}_{t,x,a}}
\newcommand{\termfont}[1]{\mathtt{#1}}

\newcommand{\eventsample}{\calE^{\samp}}
\newcommand{\sampletime}{\stopt}

\newcommand{\wk}{\weight_k}

\newcommand{\knot}{_{k_0}}
\newcommand{\nbarK}{\nbar_K}

\newcommand{\polyfactor}{\bm{\mathsf{poly}}}

\newcommand{\Lfactor}{\mathbf{L}}

\newcommand{\xa}{(x,a)}
\newcommand{\xhah}{(x_h,a_h)}

\newcommand{\vsthpl}{\truev^{\star}_{h+1}}
\newcommand{\vsth}{\truev^{\star}_{h}}
\newcommand{\qsth}{\qst_{h}}

\newcommand{\Ekh}{\surplus_{k,h}}
\newcommand{\Ekhpl}{\surplus_{k,h+1}}
\newcommand{\xast}{(x,a^{\star})}

\newcommand{\Ekt}{\mathbf{E}_{k,t}}

\newcommand{\Bern}{\mathrm{Bernoulli}}

\newcommand{\fullclip}[1]{\check{#1}}

\newcommand{\ahat}{\widehat{a}}

\newcommand{\nbark}{\nbar_k}

\newcommand{\nk}{n_k}

\newcommand{\Nmax}{N_{\max}}

\newcommand{\nK}{n_{K}}

\newcommand{\pluseq}{\mathrel{+}=}

\newcommand{\phat}{\widehat{p}}
\newcommand{\phatk}{\phat_k}

\newcommand{\p}{p}

\newcommand{\rbar}{\overline{r}}

\newcommand{\qup}{\overline{\matQ}}
\newcommand{\qupkh}{\qup_{k,h}}

\newcommand{\kt}{_{k,t}}
\newcommand{\kh}{_{k,h}}
\newcommand{\khpl}{_{k,h+1}}

\newcommand{\surplus}{\mathbf{E}}

\newcommand{\Eclip}{\fullclip{\mathbf{E}}}

\newcommand{\Ehclip}{\halfclip{\surplus}}

\newcommand{\Ebar}{\check{\mathbf{E}}}

\newcommand{\vup}{\overline{\matV}}
\newcommand{\vupkh}{\vup_{k,h}}

\newcommand{\Varmax}{\overline{\Varterm}\,}

\newcommand{\strongeuler}{\mathsf{StrongEuler}}
\newcommand{\Euler}{\mathsf{EULER}}

\newcommand{\const}{\mathsf{C}}

\newcommand{\rsum}{\mathsf{rsum}}
\newcommand{\rsumsq}{\mathsf{rsumsq}}

\newcommand{\vupkhpl}{\vup_{k,h+1}}

\newcommand{\vupktpl}{\vup_{k,t+1}}

\newcommand{\vlow}{\underline{\matV}}
\newcommand{\vlowkh}{\vlow_{k,h}}

\newcommand{\vlowkhpl}{\vlow_{k,h+1}}
\newcommand{\vlowktpl}{\vlow_{k,t+1}}

\newcommand{\bsfa}{\boldsymbol{\alpha}}
\newcommand{\bsfaxah}{\boldsymbol{\alpha}_{x,a,h}}

\newcommand{\bsfb}{\boldsymbol{\beta}}

\newcommand{\astar}{a^{\star}}

\newcommand{\rew}{\mathrm{rw}}

\newcommand{\transxah}{\alpha_{x,a,h}}

\newcommand{\epsclip}{\epsilon_{\mathrm{clip}}}

\newcommand{\vpik}{\truev^{\pik}}

\newcommand{\vpihpl}{\truev^{\pi}_{h+1}}

\newcommand{\vpikh}{\truev_h^{\pik}}
\newcommand{\vpikhpl}{\truev_{h+1}^{\pik}}

\newcommand{\calG}{\mathcal{G}}
\newcommand{\Qfunc}{\mathbf{Q}}
\newcommand{\Vfunc}{\mathbf{V}}

\newcommand{\Exppik}{\Exp^{\pik}}
\newcommand{\Ebarkh}{\Ebar_{k,h}}

\newcommand{\halfclip}[1]{\ddot{#1}}
\newcommand{\qhclip}{\halfclip{\Qfunc}}
\newcommand{\qhclippik}{\qhclip^{\pik}}

\newcommand{\vhclip}{\halfclip{\Vfunc}}

\newcommand{\vhclippik}{\vhclip^{\pik}}

\newcommand{\delvhclip}{\partial \vhclip}

\newcommand{\subk}{_k}
\newcommand{\euler}{\mathsf{EULER}}

\newcommand{\goodprob}{\calA^{\mathrm{prob}}}
\newcommand{\goodreward}{\calA^{\rew}}
\newcommand{\goodvarval}{\calA^{\mathrm{var},\mathrm{val}}}
\newcommand{\goodvarrew}{\calA^{\mathrm{var},\rew}}
\newcommand{\goodvarprob}{\calA^{\mathrm{var},\mathrm{prob}}}
\newcommand{\goodval}{\calA^{\mathrm{val}}}

\newcommand{\goodconcentration}{\calA^{\mathrm{conc}}}

\newcommand{\bonusprob}{\bonus^{\mathrm{prob}}}
\newcommand{\bonusrew}{\bonus^{\rew}}
\newcommand{\bonusstrong}{\bonus^{\mathrm{str}}}
\newcommand{\Varhat}{\widehat{\Var}}
\newcommand{\Varpxa}{\Var_{p\xa}}

\newcommand{\Varphatxa}{\Var_{\phat\xa}}

\togglefalse{nips}
\newcommand{\nipsvspace}[1]{\iftoggle{nips}{\vspace{1}}{}}

\title{Non-Asymptotic Gap-Dependent Regret Bounds for Tabular MDPs}
\author{Max Simchowitz\\
UC Berkeley \\
msimchow@berkeley.edu 
\and 
Kevin Jamieson\\
University of Washington \\
jamieson@cs.washington.edu}

\begin{document}
\maketitle

\begin{abstract}
This paper establishes that optimistic algorithms attain gap-dependent and non-asymptotic logarithmic regret for episodic MDPs. In contrast to prior work, our bounds do not suffer a dependence on diameter-like quantities or ergodicity, and smoothly interpolate between the gap dependent logarithmic-regret, and the $\widetilde{\mathcal{O}}(\sqrt{HSAT})$-minimax rate. The key technique in our analysis is a novel ``clipped'' regret decomposition which applies to a broad family of recent optimistic algorithms for episodic MDPs. 
\end{abstract}


\section{Introduction}

Reinforcement learning (RL) is a powerful paradigm for modeling a learning agent's interactions with an unknown environment, in an attempt to accumulate as much reward as possible. 
Because of its flexibility, RL can encode such a vast array of different problem settings - many of which are entirely intractable. Therefore, it is crucial to understand what conditions enable an RL agent to effectively learn about its environment, and to account for the success of RL methods in practice. 

In this paper, we consider tabular Markov decision processes (MDPs), a canonical RL setting where the agent seeks to learn a \emph{policy} mapping discrete states $x \in \states$ to one of finitely many actions $a \in \actions$, in an attempt to maximize cumulative reward over an episode horizon $H$. We shall study the \emph{regret} setting, where the learner plays a policy $\pi_k$ for a sequence of episodes $k= 1,\dots,K$, and suffers a regret proportional to the average sub-optimality of the policies $\pi_1,\dots,\pi_K$.


In recent years, the vast majority of literature has focused on obtaining \emph{minimax} regret bounds that match the worst-case dependence on the number states $|\states|$, actions $|\actions|$, and horizon length $H$; namely, a cumulative regret of $\sqrt{H |\states||\actions| T}$, where $T = KH$ denotes the total number of rounds of the game~\citep{azar2017minimax}. While these bounds are succinct and easy to interpret, they paint an overly pessimistic account of the complexity of these problems, and do not elucidate the favorable structural properties of which a learning agent can hope to take advantage.

The earlier literature, on the other hand, establishes a considerable more favorable regret of the form $ C\log T$, where $C$ is an instance-dependent constant given in terms of the \emph{sub-optimality gaps} associated with each action at a given state, defined as
\begin{align} \label{eqn:first_gap_def}
\gapinf \xa = \truev^{\pist}(x) - \trueq^{\pist}(x,a),
\end{align}
where $\truev^{\pist}$ and $\trueq^{\pist}$ denote the value and $Q$ functions for an optimal policy $\pist$, and the subscript-$\infty$ denotes these bounds hold for a non-episodic, infinite horizon setting. Depending on the constant $C$, the regret $C \log T$ can yield a major improvement over the $\sqrt{T}$ minimax scaling. Unfortunately, these analyses are asymptotic in nature, and only take effect after a large number of rounds, depending on other potentially-large, highly-conservative, or difficult-to-verify problem-dependent quantities such as hitting times or measures of uniform ergodicity \cite{jaksch2010near,tewari2008optimistic,ok2018exploration}. 

To fully account for the empirical performance of RL algorithms, we
seek regret bounds which take advantage of favorable problem instances, but apply in \emph{finite time} and for practically realistic numbers of rounds $T$.




\iftoggle{nips}
{

	\textbf{Contributions: }
}
{
	\subsection{Contributions} 

}
As a first step in this direction, \cite{zanette2019tighter} introduced a novel algorithm called $\Euler$, which enjoys reduced dependence on the episode horizon $H$ for favorable instances, while maintaining the same worst-case dependence for other parameters in their analysis as in \cite{azar2017minimax}.

In this paper, we take the next step by demonstrating that a common class of algorithms for solving MDPs, based on the \emph{optimism} principle, attains  gap-dependent, problem-specific bounds similar to those previously found only in the asymptotic regime. For concreteness, we specialize our analysis to a minor modification of the $\Euler$ algorithm we call $\strongeuler$; as we explain in Section~\ref{sec:optimism}, our analysis extends more broadly to other optimistic algorithms as well. We show that 
\iftoggle{nips}{\begin{itemize}[topsep=0pt, partopsep=0pt, labelindent=0pt, labelwidth=0pt, leftmargin=10pt]}{\begin{itemize}}
	\item For any episodic MDP $\calM$, $\strongeuler$ enjoys a high probability regret bound of $\constM \log (1/\delta)$ for all rounds $T \ge 1$, where the constant $\constM$ depends on the sub-optimality gaps between actions at different states, as well as the horizon length, and contains an additive almost-gap-independent term that scales as $AS^2\poly(H)$ (Corollary~\ref{cor:log_regret}). 
\end{itemize}
Unlike previous gap-dependent regret bounds, 
\iftoggle{nips}{\begin{itemize}[topsep=0pt, partopsep=0pt, labelindent=0pt, labelwidth=0pt, leftmargin=10pt]}{\begin{itemize}}
		\iftoggle{nips}{\vspace{-1mm}}{}
	\item The constant $\constM$ does not suffer worst-case dependencies on other problem dependent quantities such as mixing times, hitting times or measures of ergodicity. However, the constant $\constM$ \emph{does} take advantage of \emph{benign} problem instances (Definition~\ref{def:benign_instances}).
	%
	\item The regret bound of $\constM \log (1/\delta)$ is valid for any total number of rounds $T \ge 1$. Selecting $\delta = 1/T$, this implies a \emph{non-asymptotic} expected regret bound of $\constM \log T$\footnote{By this, we mean that for any \emph{fixed} $T \ge 1$, one can attain $\constM\log T$ regret. Extending the bound to anytime regret is left to future work}.
	%
	\item The regret of $\strongeuler$ interpolates between instance-dependent regret $\constM \log T$ and minimax regret $\widetilde{\calO}(\sqrt{H|\states||\actions|T})$, the latter of which may be sharper for smaller $T$ (Theorem~\ref{thm:main_regret_bound}). Following~\cite{zanette2019tighter}, this dependence on $H$ may also be refined for benign instances. 
\end{itemize}
\iftoggle{nips}
{}
{
	
}
Lastly, while the $\strongeuler$ algorithm affords sharper regret bounds than past algorithms, our analysis techniques extend more generally to other optimism based algorithms:
\iftoggle{nips}
{}
{
	
}
\iftoggle{nips}{\begin{itemize}[topsep=0pt, partopsep=0pt, labelindent=0pt, labelwidth=0pt, leftmargin=10pt]}{\begin{itemize}}
	\item We introduce a novel ``clipped'' regret decomposition (Proposition~\ref{prop:surplus_clipping_simple}) which applies to a broad family of optimistic algorithms, including the algorithms analyzed in \citep{ zanette2019tighter,dann2018policy,dann2017unifying,jin2018q,azar2017minimax}. 
	\item Following our analysis of $\strongeuler$, the  clipped regret decomposition can establish analogous gap-dependent $\log T$-regret bounds for many of the algorithms mentioned above.
\end{itemize}
\iftoggle{nips}
{}
{
	
}
\iftoggle{nips}{\textbf{What is $\constM$?}}{\paragraph{What is $\constM$?}} In many settings, we show that $\constM$ is dominated by an analogue to the sum over the reciprocals of the gaps defined in \eqref{eqn:first_gap_def}. This is known to be optimal for non-dynamic MDP settings like contextual bandits, and we prove a lower bound (Proposition~\ref{prop:info_th_lower_bound}) which shows that this is unimprovable for general MDPs as well. Furthermore, building on~\cite{zanette2019tighter}, we show this adapts to problems with additional structure, yielding\iftoggle{nips}{, e.g., }{, for example,} a horizon $H$-free bound for contextual bandits. 

However, our gap-dependent bound also suffers from a certain dependence on the \emph{smallest nonzero gap} $\gapmin$ (see Definition~\ref{def:gaps}), which may dominate in some settings. We prove a lower bound (Theorem~\ref{thm:lower_bound_informal}) which shows that optimistic algorithms in the recent literature -  including $\strongeuler$ - necessarily suffer a similar term in their regret. We believe this insight will motivate new algorithms for which this dependence can be removed, leading to new design principles and actionable insights for practitioners. Finally, our regret bound incurs an (almost) gap-independent burn-in term, which is standard for optimistic algorithms, and which we believe is an exciting direction of research to remove. 

Altogether, we believe that the results in our paper serve as a preliminary but significant step to attaining sharp, instance-dependent, and \emph{non-asymptotic} bounds for tabular MDPs, and hope that our analysis will guide the design of future algorithms that attain these bounds.

\subsection{Related Work\label{sec:related}}

Like the multi-armed bandit setting, regret bounds for MDP algorithms have been characterized both in \emph{gap-independent} forms that rely solely on $S:=|\states|,A:=|\actions|,H,T$, and in \emph{gap-dependent} forms which take into account the gaps~\eqref{eqn:first_gap_def}, as well as other instance-specific properties of the rewards and transition probabilities. 
  
\iftoggle{nips}{\textbf{Finite Sample Bounds, Gap-Independent Bounds: }}{\paragraph{Finite Sample Bounds, Gap-Independent Bounds: }} A number of notable recent works give undiscounted regret bounds for finite-horizon, tabular MDPs, nearly all of them relying on the principle of optimism which we describe in Section~\ref{sec:optimism} \citep{dann2015sample,azar2017minimax,dann2017unifying,jin2018q,zanette2019tighter}. Many of the more recent works \cite{azar2017minimax,zanette2019tighter,dann2018policy} attain a regret of $\sqrt{HSAT}$, matching the known lower bound of $\sqrt{HSAT}$ established in \cite{osband2016lower,jaksch2010near,dann2015sample}. As mentioned above, the $\Euler$ algorithm of \cite{zanette2019tighter} attains the minimax rates and simultaneously enjoys a reduced dependence on $H$ in benign problem instances, such as the contextual bandits setting where the transition probabilities do not depend on the current state or learners actions, or when the total cumulative rewards over any roll-out are bounded by $1$ in magnitude.  

\iftoggle{nips}{\textbf{Diameter Dependent Bounds:}}{\paragraph{Diameter Dependent Bounds:}} In the setting of infinite horizon MDPs with discounted regret, many previous works have established logarithmic regret bounds of the form $C(\calM) \log T$, where $C(\calM)$ is a constant depending on the underlying MDP. Notably, \cite{jaksch2010near} give an algorithm which attains a $\widetilde{\calO}(\sqrt{D^2 S^2 A T})$ gap-independent regret, and an $\widetilde{\calO}( \frac{D^2 S^2 A}{\gap_*} \log(T) )$ gap-dependent regret bound, where $\gap_*$ is the difference between the mean infinite-horizon reward of $\pi_*$ and the next-best stationary policy, and where $D$ denotes the maximum expected traversal time between any two states $x,x'$, under the policy which attains the minimal traversal time between those two states. We note that if $\gapinf \xa$ denotes the sub-optimality of any action $a$ at state $x$ as in~\eqref{eqn:first_gap_def}, then $\gap_* \leq \min_{x,a} \gapinf\xa$. The bounds in this work, on the other hand, depend on an average over inverse gaps, rather than a worst case. Moreover, the diameter $D$ can be quite large when there exist difficult-to-access states. We stress that the bound due to \cite{jaksch2010near} is non-asympotic, but the bound in terms of $\gap_*$ dependences other worst-case quantities measuring ergodicity.

\iftoggle{nips}{\textbf{Asymptotic Bounds:} }{\paragraph{Asymptotic Bounds:} } Prior to \cite{jaksch2010near}, and building on the bounds of \cite{burnetas1997optimal}, \cite{tewari2008optimistic} presented bounds in terms of a diameter-related quantity $\bar{D} \geq D$, which captures the minimal hitting time between states when restricted to optimal policies.  \cite{tewari2008optimistic}  prove that their algorithm enjoys a regret\footnote{\cite{tewari2008optimistic} actually presents a bound of the form $\frac{\bar{D}^2 SA}{\min_{(s,a) \in \text{\sf CRIT}} \gapinf \xa} \log(T)$ but it is straightforward to extract the claimed form from the proof.} 
of $\sum_{(s,a) \in \text{\sf CRIT}} \frac{\bar{D}^2}{\gapinf\xa} \log(T)$ asymptotically in $T$ where $\text{\sf CRIT}$ contains those sub-optimal state-action pairs $\xa$ such that $a$ can be made to the the unique, optimal action at $x$ by replacing $p(s'|s,a)$ with some other vector on the $\states$-simplex. 
Recently, \cite{ok2018exploration} present per-instance lower bounds for both structured and unstructured MDPs, which apply to any algorithm which enjoys sub-linear regret on any problem instance, and an algorithm which matches these bounds asymptotically. This bound replaces $\bar{D}^2$ with $\bar{H}^2$, where $\bar{H}$ denotes the range of the bias functions, an analogue of $H$ for the non-episodic setting \cite{bartlett2009regal}. 
We further stress that whereas the logarithmic regret bounds of \cite{jaksch2010near} hold for finite time with polynomial dependence on the problem parameters, the number of episodes needed for the bounds of \cite{burnetas1997optimal,tewari2008optimistic,ok2018exploration} to hold may be exponentially large, and depend on additional, pessimistic problem-dependent quantities (e.g. a uniform hitting time in \iftoggle{nips}{Proposition 29 in~\cite{tewari2007reinforcement}}{\citet[Proposition 29]{tewari2007reinforcement}}).


\iftoggle{nips}{\textbf{Novelty of this work:}}{\paragraph{Novelty of this work:}} The major contribution of our work is showing problem-dependent $\log(T)$ regret bounds which i) attain a refined dependence on the gaps, as in \cite{tewari2008optimistic}, ii) apply in finite time after a burn-in time only polynomial in $S$, $A$, $H$ and the gaps, iii) depend only on $H$ and not on the diameter $D$ (and thus, are not adversely affected by difficult to access states), and iv) smoothly interpolate between $\log T$ regret and the minimax $\sqrt{HSAT}$ rate attained by \cite{azar2017minimax} et seq.


\iftoggle{nips}{\vspace{-1.5mm}}{}

\subsection{Problem Setting, Notation, and Organization}
\iftoggle{nips}{\vspace{-1.5mm}}{}
\iftoggle{nips}{\secemph{Episodic MDP:}}{\paragraph{Episodic MDP:}} A \emph{stationary}, episodic MDP is a tuple $\calM := (\states,\actions, H,r,p,p_0,R)$, where for each $x \in \states, a \in \actions$ we have that $R(x,a) \in [0,1]$ is a random reward with expectation $r(x,a)$, $p: \states \times \actions \to \Delta^{\states}$ denotes transition probabilities, $p_0 \in \Delta^{\states}$ is an initial distribution over states, and $H$ is the horizon, or length of the episode.
A policy $\pi $ is a sequence of mappings $\pi_h : \states \to \actions$. 
For our given MDP $\calM$, we let $\Exppi$ and $\Prpi$ denote the expectation and probability operator with respect to the law of sequence $(x_1,a_1),\dots,(x_H,a_H)$, where $x_1 \sim p_0$, $a_h = \pi_h(x_h)$, $x_{h+1} \sim p(x_h,a_h)$. 
We define the \emph{value} of $\pi$ as \iftoggle{nips}
{
	$\truev_0^{\pi} := \Exppi\left[\sum_{h=1}^H r(x_h,a_h)\right]$
}
{
	\begin{align*}
	\truev_0^{\pi} := \Exppi\left[\sum_{h=1}^H r(x_h,a_h)\right],
	\end{align*}
}
and for $h \in [H]$ and $x \in \states$,\iftoggle{nips}
{
	$\truev_h^{\pi}(x) := \Exppi\left[\sum_{h' \ge h}^H r(x_{h'},a_{a'}) \mid x_h = x\right]$, }
{
	\begin{align*}
	\truev_h^{\pi}(x) := \Exppi\left[\sum_{h' \ge h}^H r(x_{h'},a_{a'}) \mid x_h = x\right],
	\end{align*}
}
which we identify with a vector in $\R^\states$.
We define the associated Q-function $\trueq^{\pi}: \states \times \actions \to \R$,
\iftoggle{nips}
{
	$\trueq_h^{\pi}(x,a) := r(x,a) + p(x,a)^\top \truev_{h+1}^{\pi}$,
}
{
	\begin{align*}
	\trueq_h^{\pi}(x,a) := r(x,a) + p(x,a)^\top \truev_{h+1}^{\pi},
	\end{align*}
}
so that $\trueq_h^{\pi}(x,\pi_h(x)) = \truev_h^{\pi}(x)$. We denote the \emph{set} of optimal policies
\iftoggle{nips}
{
	$\pist := \argmax_{\pi}\truev^{\pi}_0$,
}
{
	\begin{align*}
	\pist := \argmax_{\pi}\truev^{\pi}_0,
	\end{align*}
}
and let $\pist_h(x) := \{a: \pi_h(x) = a, \pi \in \pist\}$ denote the set of optimal actions. Lastly, given any optimal $\pi \in \pist$, we introduce the shorthand $\vst_h = \truev_h^\pi$ and $\qst_h = \trueq_h^\pi$, where we note that  even when $\pi$ is not unique, $\vst_h$ and $\qst_h$ do not depend on the choice of optimal policy. 

\iftoggle{nips}{\secemph{Episodic Regret: }}{\paragraph{Episodic Regret:}} We consider a game that proceeds in rounds $k = 1,\dots,K$, where at each state an algorithm $\Alg$ selects a policy $\pik$, and observes a roll out $(x_1,a_1),\dots,(x_H,a_H) \sim \Pr^{\pik}$. The goal is to minimize the cumulative simple regret, defined as 
\iftoggle{nips}
{
	$\regret_K := \sum_{k=1}^K \vst_0 - \vpik_0$. 
}
{
	\begin{align*}
\regret_K := \sum_{k=1}^K \vst_0 - \vpik_0.
\end{align*}
}

\secemph{Notation and Organization:} For $n \in \mathbb{N}$, we define $[n] = \{1,\dots,n\}$.
For two expressions $f,g$ that are functions of any problem-dependent variables of $\calM$, we say $f \lesssim g$ ($f \gtrsim g$, respectively) if there exists a universal constant $c >0$ independent of $\calM$ such that $f \leq c g$ ($f \geq c g$, respectively). $\lessapprox$ will denote an informal, approximate inequality. Section~\ref{sec:main_results} presents our main results, and Section~\ref{sec:optimism} sketches the proof and highlights the novelty of our techniques. \iftoggle{nips}{All references to the appendix refer to the appendix of the supplement. }{} All formal proofs, and many rigorous statement of results, are deferred to the appendix, whose organization and notation are described at length in Appendix~\ref{sec:appendix_notation}. 

\iftoggle{nips}{\vspace{-1.5mm}}{}
\subsection{Optimistic Algorithms\label{sec:optimistic_algs}} 
\iftoggle{nips}{\vspace{-1.5mm}}{}
Lastly, we introduce \emph{optimistic algorithms} which select a policy which is optimal for an over-estimated, or \emph{optimistic}, estimate of the true $Q$-function, $\qst$.

\begin{defn}[Optimistic Algorithm]\label{def:optimistic_alg} We say that an algorithm $\Alg$ satisifes \emph{optimism} if, for each round $k \in [K]$ and stage $h \in [H]$, it constructs an \emph{optimistic} $Q$-function $\qupkh(x,a)$ and policy $\pik = (\pikh)$ satisfying 
\iftoggle{nips}
{
	$\forall x,a:\, \qup_{k,H+1}\xa = 0, \,\qupkh(x,a)  \ge \qsth(x,a),\, \, \text{ and }\, \pikh(x) \in \argmax_{a}\qupkh(x,a)$.
}
{
	\begin{align*}
	\forall x,a:\, \qup_{k,H+1}\xa = 0, \,\qupkh(x,a)  \ge \qsth(x,a),\, \, \text{ and }\, \pikh(x) \in \argmax_{a}\qupkh(x,a).
	\end{align*}
}
The associated \emph{optimistic value function} is $\vupkh(x) :=\qupkh(x,\pikh(x))$. 
\end{defn} 
We shall colloquially refer to an algorithm as \emph{optimistic} if it satsifies optimism with high probability. Optimism has become the dominant approach for learning finite-horizon MDPs, and all recent low-regret algorithms are optimistic \citep{dann2017unifying,dann2018policy,azar2017minimax,zanette2019tighter,jin2018q}. 
In \emph{model-based} algorithms, the overestimates $\qupkh$ are constructed recursively as $\qupkh\xa = \rhatk\xa + \phatk\xa^\top \vupkhpl + \bonus\kh\xa$, where $\rhatk\xa$ and $\phatk\xa$ are empirical estimates of the mean rewards and transition probabilities, and $\bonus\kh\xa \ge 0$ is a confidence \emph{bonus} designed to ensure that $\qupkh\xa \ge \qst\xa$. Letting $\nk\xa$ denote the total number of times a given state-action pair is visited, a simple bonus $\bonus\kh\xa \eqsim \sqrt{\frac{H\log(SAHK/\delta)}{\nk\xa}}$ suffices to induce optimism, yielding the UCBVI-CH algorithm of \citep{azar2017minimax}. This leads to an episodic regret bound of $\sqrt{H^2SAT}$, a factor of $\sqrt{H}$ greater than the minimax rate. More refined bonuses based on the ``Bernstein trick'' achieve the optimal $H$-dependence \citep{azar2017minimax}, and the $\Euler$ algorithm of \cite{zanette2019tighter} adopts further refinements to replace worst-case $H$ dependence with more adaptive quantities. \iftoggle{nips}{}{

}The $\strongeuler$ algorithm considered in this work applies similarly adaptive bonuses, but our analysis extends to all aforementioned bonus configurations. We remark that there are also \emph{model-free} optimistic algorithms based on Q-learning (see, e.g.~\cite{jin2018q}) that construct overestimates in a slightly different fashion. While our main technical contribution, the clipped regret decomposition (Proposition~\ref{prop:surplus_clipping_simple}), applies to \emph{all} optimistic algorithms, our subsequent analysis is tailored to model-based approaches, and may not extend straightforwardly to Q-learning methods. 

\iftoggle{nips}{\vspace{-2mm}}{}
\section{Main Results\label{sec:main_results}}
\iftoggle{nips}{\vspace{-1.5mm}}{}


\iftoggle{nips}{\textbf{Logarithmic Regret for Optimistic Algorithms:}}{\paragraph{Logarithmic Regret for Optimistic Algorithms:}} We now state regret bounds that describe the performance of $\strongeuler$, an instance of the model-based, optimistic algorithms described above.  $\strongeuler$ is based on carefully selected bonuses from~\cite{zanette2019tighter}, and formally instantiated in Algorithm~\ref{alg:strong_euler} in Appendix~\ref{sec:alg:strong_euler}. We emphasize that other optimistic algorithms enjoy similar regret bounds, but we restrict our analysis to $\strongeuler$ to attain the sharpest $H$-dependence.
The key quantities at play are the \emph{suboptimality-gaps} between the Q-functions: 
\begin{defn}[Suboptimality Gaps]\label{def:gaps} For $h \in [H]$, define the stage-dependent suboptimality gap 
\iftoggle{nips}
{
	$\gaph\xa := \vsth(x) - \qsth(x,a)$, 
}
{
	\begin{align*}
\gaph\xa := \vsth(x) - \qsth(x,a)\,,
\end{align*} 
}
as well as the minimal stage-independent gap $\gap\xa := \min_h \gaph(x,a)$, and the minimal gap $\gapmin := \min_{x,a,h} \{\gaph\xa: \gaph\xa > 0\}.$
\end{defn}
Note that any optimal $\astar\in \pisth(x)$ satisfies the Bellman equation $ \qsth(x,\astar) = \max_{a}\qsth(x,a)$$ = \vsth(x)$, and thus $\gaph(x,\astar) = 0$ \iftoggle{nips}{iff}{if and only if}  $\astar \in \pisth(x)$. Following \cite{zanette2019tighter}, we consider two illustrative benign problem settings \iftoggle{nips}{which afford}{under which we obtain} an improved dependence on the horizon $H$:
\begin{defn}[Benign Settings]\label{def:benign_instances} We say that an MDP $\calM$ is a \emph{contextual bandit instance} if $p(x'|x,a)$ does not depend on $x$ or $a$. An MDP $\calM$ has $\Gterm$-\emph{bounded rewards} if, for any policy $\pi$, $\sum_{h=1}^H R(x_h,a_h) \le \Gterm$ holds with probability 1 over trajectories $((x_h,a_h)) \sim \Pr^{\pi}$.
\end{defn}

\iftoggle{nips}
{	
	Lastly, we define $\Zopt$ as the set of pairs $\xa$ for which $a$ is optimal at $x$ for some stage $h \in [H]$: $\Zopt := \{(x,a) : \exists h \in [H] \text{ with } a \in \pisth(x)\}$ and its complement $\Zsub:= \states \times \actions - \Zopt$.
}
{
	Lastly, we define $\Zopt$ as the set of pairs $\xa$ for which $a$ is optimal at $x$ for some stage $h \in [H]$, and its complement $\Zsub$:
	\begin{align*}
	\Zopt := \{(x,a) : \exists h \in [H] \text{ with } a \in \pisth(x)\} \text{ and } \Zsub := \states \times \actions - \Zopt.
	\end{align*}
}
Note that typically $|\Zopt| \lesssim H|\states|$ or even $|\Zopt| \lesssim |\states|$ (see Remark~\ref{rem:Z_opt} in the appendix). We now state our first result, which gives a gap-dependent regret bound that scales as $\log(1/\delta)$ with probability at least $1 - \delta$. 
The result is a consequence of a more general result stated as Theorem~\ref{thm:main_regret_bound}, itself a simplified version of more precise bounds stated in Appendix~\ref{sec:main_reg_ht}.
\begin{cor}\label{cor:log_regret} Fix $\delta \in (0,1/2)$, and let $A = |\actions|$, $S= |\states|$, $M = (SAH)^2$. Then with probability at least $1-\delta$, $\strongeuler$ run with confidence parameter $\delta$ enjoys the following regret bound for all $K \ge 1$:
\begin{align*}
\regret_K &\lesssim \left(\sum_{(x,a)\in \Zsub} \frac{H^3 }{\gap\xa} \log\frac{M T}{\delta}\right)+ \frac{H^3 |\Zopt| }{\gapmin} \log\frac{M T}{\delta}   \\
&\quad+  H^4 S A (S \vee H) \log\frac{MH}{\gapmin} \log \frac{MT}{\delta}\numberthis\label{eq:log_regret_bound}.
\end{align*}
Moreover, if $\calM$ is either a contextual bandits instance, or has $\Gterm$-bounded rewards for $\Gterm \lesssim 1$,  then the factors of $H^3$on the first line can be sharped to $H$. In addition, if $\calM$ is a contextual bandits instance, the factor of $H^3$ in the first term (summing over $\xa \in \Zsub$) can be sharped to $1$. 
\end{cor}
\iftoggle{nips}{\vspace{-1.5mm}}{}
Setting $\delta = 1/T$ and noting that $\sum_{k=1}^K \vst_0 -  \vpik_0 \le KH = T$ with probability $1$ (recall $R(x,a) \in [0,1]$), we see that the expected regret $\Exp[\sum_{k=1}^K \vst_0 -  \vpik_0]$ can be bounded by replacing $1/\delta$ with $T$ in right hand side of the inequality~\eqref{eq:log_regret_bound}; this yields an expected regret that scales as $\log T$. 

\iftoggle{nips}{\secemph{Three regret terms: }}{\paragraph{Three regret terms: }} The first term in Corollary~\ref{cor:log_regret} reflects the sum over sub-optimal state-action pairs, which a lower bound (Proposition~\ref{prop:info_th_lower_bound}) shows is unimprovable in general. In the infinite horizon setting, \cite{ok2018exploration} gives an algorithm whose regret is asymptotically bounded by an analogue of this term.
The third term characterizes the burn-in time suffered by nearly all model-based finite-time analyses and is the number of rounds necessary before standard concentration of measure arguments kick in. 
The second term is less familiar and is addressed in Section~\ref{sec:whygapmin} below.

\iftoggle{nips}{\secemph{$H$ dependence:}}{\paragraph{$H$ dependence:}}Comparing to known results from the infinite-horizon setting, one expects the optimal dependence of the first term on the horizon to be $H^2$. However, we cannot rule out that the optimal dependence is $H^3$ for the following three reasons: (i) the infinite-horizon analogues $D,\bar{D},\bar{H}$ (Section~\ref{sec:related}) are not directly comparable to the horizon $H$; (ii) in the episodic setting, we have a potentially different value function $\vsth$ for each $h \in [H]$, whereas the value functions of the infinite horizon setting are constant across time; (iii) the $H^3$ may be unavoidable for non-asymptotic (in $T$) bounds, even if $H^2$ is the optimal asymptotic dependence after sufficient burn-in (possibly depending on diameter-like quantities). Resolving the optimal $H$ dependence is left as future work. We also note that for contextual bandits, we incur \emph{no} $H$ dependence on the first term;  and thus the first term coincides with the known asymptotically optimal (in $T$), instance-specific regret \citep{garivier2018explore}. 

\iftoggle{nips}{\secemph{Guarantees for other optimistic algorithms: }}{\paragraph{Guarantees for other optimistic algorithms: }} To make the exposition concrete, we only provide regret bounds for the $\strongeuler$ algorithm. However, the ``gap-clipping'' trick (Proposition~\ref{prop:surplus_clipping_simple}) and subsequent analysis template described in Section~\ref{sec:clip_trick} can be applied to obtain similar bounds for other recent optimistic algorithms, as in~\citep{azar2017minimax,dann2017unifying,zanette2019tighter,dann2018policy}.\footnote{To achieve logarithmic regret, some of these algorithms require a minor modification to their confidence intervals; otherwise, the gap-dependent regret scales as $\log^2 T$. See Appendix~\ref{sec:alg:strong_euler} for details. }

\iftoggle{nips}{\vspace{-3mm}}{}
\subsection{Sub-optimality Gap Lower Bound}
\iftoggle{nips}{\vspace{-1.5mm}}{}

Our first lower bound shows that when the total number of rounds $T=KH$ is large, the first term of Corollary~\ref{cor:log_regret} is unavoidable in terms of regret. Specifically, for every possible choice of gaps, there exists an instance whose regret scales on the order of the first term in~\eqref{eq:log_regret_bound}. 

Following standard convention in the literature, the lower bound is stated for algorithms which have sublinear worst case regret. Namely,  we say than an algorithm $\Alg$ is $\alpha$-\emph{uniformly good} if, for any MDP instance $\calM$, there exists a constant $C_{\calM} > 0$ such that $\Exp^{\calM}[\regret_K] \le C_{\calM} K^{\alpha}$ for all $K$.\footnote{We may assume as well that $\Alg$ is allowed to take the number of episodes $K$ as a parameter.} 
\begin{prop}[Regret Lower Bound]\label{prop:info_th_lower_bound} Let $S \ge 2$, and $A \ge 2$, and let $\{\Delta_{x,a}\}_{x,a \in [S] \times [A]} \subset (0,H/8)$ denote a set of gaps.  Then, for any $H \ge 1$, there exists an MDP $\calM$ with states $\calS = [S+2]$, actions $\calA = [A]$, and $H$ stages, such that, 
\begin{align*}
\gap_1(x,a) &= \Delta_{x,a}, \qquad \forall x \in [S], a\in \calA
\iftoggle{nips}{}{\footnote{We can also show a regret lower bound of $\gtrsim \sum_{x,a: \gap_h\xa > 0}\frac{(H-h)^2}{\gap_h\xa} \log K$ (see Equation~\eqref{eq:general_stage_gap} in the Appendix)}}\\
\gap_h\xa &\ge 1/2, \qquad \forall x \in \{S+1,S+2\}, a \in \calA - \{1\}, 
\end{align*}
 and any $\alpha$-uniformly good algorithm satisfies
\begin{align*}
	\lim_{K \to \infty} \frac{\Exp^\calM[\regret_K]}{\log T} \gtrsim (1-\alpha)\sum_{x,a:\gap_1(x,a) > 0}\frac{H^2}{\gap_1(x,a)} 
\end{align*}
\end{prop}
The above proposition is proven in Appendix~\ref{sec:inf_th_lb}, using a construction based on~\cite{dann2015sample}.
For simplicity, we stated an asymptotic lower bound. We remark that if the constant $C_{\calM}$ is $\poly(|\states|,|\actions|,H)$, then one can show that the above asymptotic bound holds as soon as $K \ge (|\states||\actions| H/\gap_*)^{\BigOh{1/(1-\alpha)}}$, where $\gap_* := \{\min \gap_1(x,a) : \gap_1(x,a) > 0\}$.  More refined non-asymptotic regret bounds can be obtained by following~\cite{garivier2018explore}.

\iftoggle{nips}{\vspace{-3mm}}{}
\subsection{Why the dependence on $\gapmin$?\label{sec:whygapmin}}
\iftoggle{nips}{\vspace{-1.5mm}}
Without the second term, Corollary~\ref{cor:log_regret} would only suffer one factor of $1/\gapmin$ due to the sum over state-actions pairs $\xa \in \Zsub$ (when the minimum is achieved by a single pair). 
However, as remarked above, $|\Zopt|$ typically scales like $|\states|$ and therefore the second term scales like $|\states|/\gapmin$, with a dependence on $1/\gapmin$ that is at least a factor of $|\states|$ more than we would expect.
Here, we show that $|\states|/\gapmin$ is unavoidable for the sorts of optimistic algorithms that we typically see in the literature; a rigorous proof is deferred to Appendix~\ref{sec:min-gap-lower-bound}.
\begin{thm}[Informal Lower Bound]\label{thm:lower_bound_informal} 
Fix $\delta \in (0,1/8)$. 
For universal constants $c_1,c_2,c_3,c_4$, if $\epsilon \in (0,c_1)$, and $S$ satisfies $c_2 \log(\epsilon^{-1}/\delta) \leq S \leq c_3 \epsilon^{-1} / \log(\epsilon^{-1}/\delta)$, there exists an MDP with $|\states| = S$, $|\actions| = 2$ and horizon $H = 2$, such that exactly one state has a sub-optimality gap of $\gapmin = \epsilon$ and all other states have a minimum sub-optimality gap $\gap_h\xa \ge 1/2$. 
For this MDP, $\sum_{h,x,a:\gaph\xa > 0}\frac{1}{\gaphxa} \lesssim S + \frac{1}{\gapmin}$ but all existing optimistic algorithms for finite-horizon MDPs which are $\delta$-correct suffer a regret of at least $\frac{S}{\gapmin}\log (1/\delta) \gtrsim \sum_{h,x,a:\gaph\xa > 0}\frac{\log(1/\delta)}{\gaphxa} + \frac{S \log(1/\delta)}{\gapmin}$ with probability at least $1- c_4 \delta$.
\end{thm}
The particular instance described in Appendix~\ref{sec:min-gap-lower-bound} that witnesses this lower bound is instructive because it demonstrates a case where optimism results in \emph{over}-exploration.




\iftoggle{nips}{\vspace{-3mm}}{}
\subsection{Interpolating with Minimax Regret for Small $T$} 
\iftoggle{nips}{\vspace{-1.5mm}}

We remark that while the logarithmic regret in Corollary~\ref{cor:log_regret} is non-asymptotic, the expression can be loose for a number of rounds $T$ that is small relative to the sum of the inverse gaps. 
Our more general result interpolates between the $\log T$ gap-dependent and $\sqrt{T}$ gap-independent regret regimes.

\begin{thm}[Main Regret Bound for $\strongeuler$]\label{thm:main_regret_bound} Fix $\delta \in (0,1/2)$, and let $A = |\actions|$, $S= |\states|$, $M = (SAH)^2$. Futher, define for all $\epsilon > 0$ the set  $\Zsub(\epsilon):= \{(x,a) \in \Zsub : \gap\xa < \epsilon\}$. Then with probability at least $1-\delta$, $\strongeuler$ run with confidence parameter $\delta$ enjoys the following regret bound for all $K \ge 2$:
	\begin{align*}
	\regret_K &\lesssim \min_{\epsilon > 0} \Big\{ \sqrt{ |\Zsub(\epsilon)| H\,  \, T (\log T)\log \tfrac{M T}{\delta}}\} + \sum_{(x,a)\in \Zsub  \setminus \Zsub(\epsilon)} \frac{H^3 }{\gap\xa} \log\left(\tfrac{MT }{\delta}\right)\Big\}\\
	&\qquad+  \min\left\{\sqrt{|\Zopt|\,H \,T (\log T) \log  \tfrac{M T}{\delta}},\,\, |\Zopt| \, \frac{H^3}{\gapmin} \log\left(\tfrac{M T}{\delta}\right)\right\}\\
	&\qquad+  H^4 S A (S \vee H) \min \log \tfrac{MT}{\delta} \left\{  \log \tfrac{MT}{\delta}, \, \log\tfrac{MH}{\gapmin} \right\}\\
	&\lesssim \sqrt{HSAT \log(T) \log(\tfrac{MT}{\delta})} + H^4 S A (S \vee H) \log^2 \tfrac{TM}{\delta}\,,
	\end{align*}
where the second inequality follows from the first with $\max\{\max_\epsilon |\Zsub(\epsilon)|,|\Zopt|\} \leq SA$. 
Moreover, if $\calM$ is an instance of contextual bandits, then the factors of $H$ under the square roots can be refined to a $1$, and if $\calM$ has $\lesssim 1$-bounded rewards, then these same factors of $H$ can be replaced by a $1/H$. In both settings, logarithmic terms can be refined as in Corollary~\ref{cor:log_regret}.

\end{thm}

By the same argument as above, Theorem~\ref{thm:main_regret_bound} with $\delta=1/T$ implies an expected regret scaling like gap-dependent $\log T$ or worst-case $\sqrt{HSAT}$. In Appendix~\ref{sec:main_reg_ht}, we state a more refined bound given in terms of the reward bound $\calG$, and the maximal variance of any state-action pair (Theorem~\ref{thm:main_regret_bound_ht}).

\iftoggle{nips}{\vspace{-.1in}}{}
\section{Gap-Dependent bounds via `clipping' \label{sec:optimism} }
\iftoggle{nips}{\vspace{-.1in}}{}
In this section, we (i) introduce the key properties of optimistic algorithms, (ii) explain existing approaches to the analysis of such algorithms, and (iii) introduce the ``clipping trick'', and sketch how this technique yields gap-dependent, non-asymptotic bounds. 
\begin{defn}[Optimistic Surplus] Given an optimistic algorithm $\Alg$, we define the (optimistic) \emph{surplus}
\iftoggle{nips}
{
	$\Ekh\xa := \qupkh(x,a) -  r(x,a) - p(x,a)^\top \vupkhpl$.
}
{
	\begin{align*}
	\Ekh\xa := \qupkh(x,a) - \left( r(x,a) - p(x,a)^\top \vupkhpl \right).
	\end{align*}
}
\iftoggle{nips}{}{We further say that }$\Alg$ is \emph{strongly optimistic} if $\Ekh(x,a) \ge 0$ for all $k \ge 1$, and $(x,a,h) \in \states \times \actions \times [H]$, which implies that $\Alg$ is also optimistic.
\end{defn}
While the nomenclature ``suplus'' is unique to our work, surplus-like terms arise in many prior regret analyses \cite{dann2017unifying,zanette2019tighter}. The notion of \emph{strong optimism} is novel to this work, and facilitates a sharper $H$-dependence in  contextual bandit setting of Definition~\ref{def:benign_instances}; intuitively, strong optimism means that the Q-function $\qupkh$ at stage $h$ over-estimates $\qsth$ more than $\qup_{k,h+1}$ does $\qst_{k,h+1}$.


\iftoggle{nips}{\secemph{The Regret Decomposition for Optimistic Algorithms:}}{\paragraph{The Regret Decomposition for Optimistic Algorithms:}}
%
Under optimism alone, we can see that for any $h$ and any $\astar \in \pist(x)$,
\begin{align*}
\vupkh(x) = \max_{a} \qupkh(x,a) \ge \qupkh(x,\astar) \ge \qsth(x,\astar)  = \vsth(x),
\end{align*} 
and therefore, we can bound the sub-optimality of $\pik$ as $\vst_0 - \vpik_0 \le \vup_{k,0}  - \vpik_0 $. 

We can decompose the regret further by introducing the following notation: we let $\wkh\xa := \Pr^{\pik}[\xhah = \xa]$ denote the probability of visiting $x$ and playing $a$ at time $h$ in episode $k$. 
\iftoggle{nips}{}
{We note that since $\pik(x)$ is a deterministic function, $\wkh\xa$ is supported on only one action $a$ for each state $x$ and stage $h$.}
A standard regret decomposition (see e.g. \iftoggle{nips}{Lemma E.15~\cite{dann2017unifying}}{\citet[Lemma E.15]{dann2017unifying}}) then shows that for a trajectory $(x_h,a_h)_{h=1}^H$,
\iftoggle{nips}
{
	$\vup_{k,0}  - \vpik_0 = \Exp^{\pik}[\sum_{h=1}^H \Ekh(x_h,a_h) ] = \sum_{h=1}^H \sum_{x,a} \wkh \xa \Ekh(x,a)$,
}
{
	\begin{align*}
	\vup_{k,0}  - \vpik_0 = \Exp^{\pik}[\sum_{h=1}^H \Ekh(x_h,a_h) ] = \sum_{h=1}^H \sum_{x,a} \wkh \xa \Ekh(x,a),
	\end{align*}
}
yielding a regret bound of
\begin{align*}
\sum_{k=1}^K\vst_0 - \vpik_0 \le \sum_{k=1}^K\vup_{k,0}  - \vpik_0 \le  \sum_{k=1}^K\sum_{h=1}^H \sum_{x,a} \wkh \xa \Ekh(x,a). 
\end{align*}

\newcommand{\wj}{\weight_j}
\newcommand{\wjh}{\weight_{j,h}}

\iftoggle{nips}{\secemph{Existing Analysis of MDPs: }}{\paragraph{Existing Analysis of MDPs: }} We begin by sketching the flavor of minimax analyses. Introducing the notation 
\iftoggle{nips}
{
	$\nk\xa := \{\# \text{times } (x,a) \text{ is visited before episode } k \}$,
}
{
	\begin{align}
\nk\xa := \{\# \text{times } (x,a) \text{ is visited before episode } k \}, 
\end{align}
}
existing analyses carefully manipulate the surpluses $\Ekh\xa$ to show that \iftoggle{nips}{$\sum_{h=1}^H \sum_{x,a} \wkh \xa \Ekh(x,a) \lesssim \sum_{h=1}^H \sum_{x,a} \wkh\xa \frac{\constbarM}{\sqrt{\nk\xa}} + \text{ lower order terms}$,}{\begin{align*}\sum_{h=1}^H \sum_{x,a} \wkh \xa \Ekh(x,a) \lesssim \sum_{h=1}^H \sum_{x,a} \wkh\xa \frac{\constbarM}{\sqrt{\nk\xa}} + \text{ lower order terms}m\end{align*}} where typically $\constbarM = \poly(H,\log(T/\delta)$. 
Finally, they replace $\nk\xa$ with an ``idealized analogue'', $\nbark\xa := \sum_{j=1}^{k}\sum_{h=1}^H \wjh\xa := \sum_{j=1}^{k}\wj\xa$, where we introduce $\wj\xa := \sum_{h=1}^H\wjh\xa$ denote the expected number of visits of $\xa$ at episode $j$. Letting $\{\calF_k\}$ denote the filtration capturing all events up to the end episode $k$, we see that $\Exp[\nbar_{k}\xa - \nbar_{k-1} | \calF_{k-1}] = \wk\xa$, and thus by standard concentration arguments (see Lemma~\ref{lem:sample}, or \iftoggle{nips}{Lemma 6 in~\cite{dann2018policy}}{\citet[Lemma 6]{dann2018policy}}), $\nbar_{k}\xa$ and $\nk\xa$ are within a constant factor of each other for all $k$ such that $\nbark\xa$ is sufficiently large. Hence, by replacing $\nk\xa$ with $\nbark\xa$, we have (up to lower order terms)
\begin{align}
\sum_{k=1}^K\vst_0 - \vpik_0~\lesssim~ \sum_{x,a} \sum_{k=1}^K \wk\xa \frac{\constbarM}{\sqrt{\nbark\xa}} + \text{ lower order terms}  \label{eqn:minimax_regret_intuition}.
\end{align}
A $\sqrt{SAK\,\mathrm{poly}(H)}$ bound is typically concluded using a careful application of Cauchy-Schwartz, and an integration-type lemma (e.g., Lemma~\ref{lem:integration}). An analysis of this flavor is used in Appendix~\ref{sec:proof_of_interpolated}.

On the other hand, one can \emph{exactly} establish the identity \iftoggle{nips}{$\vst_0 - \vpik_0 = \sum_{x,a}\sum_{h=1}^H \wkh\xa\gaph\xa$. }{\begin{align*}\vst_0 - \vpik_0 = \sum_{x,a}\sum_{h=1}^H \wkh\xa\gaph\xa. \end{align*}} Then one can achieve a gap dependent bound as soon as one can show that the algorithm ceases to select suboptimal actions $a$ at $(x,h)$ after sufficiently large $T$. Crucially, determining if action $a$ is (sub)optimal at $(x,h)$ requires precise knowledge about the value function at other states in the MDP at future stages $h' > h$. This difficulty is why previous gap-dependent analyses appeal to diameter or ergodicity assumptions, which ensure sufficient uniform exploration of the MDP to reason about the value function at subsequent stages.

\subsection{The Clipping Trick  \label{sec:clip_trick}}
We now introduce the ``clipping trick'', a technique which merges both the minimax analysis in terms of the surpluses $\Ekh\xa$, and the gap-dependent strategy, which attempts to control how many times a given suboptimal action is selected. 
Core to our analysis, define the \emph{clipping} operator
\begin{align*}
\threshop{\epsilon}{x} = x \I\{ x \geq \epsilon\},
\end{align*}
for all $x,\epsilon >0$. We can now state our first main technical result, which states that the sub-optimality $\vst_0 - \vpik_0$ can be controlled by a sum over surpluses which have been \emph{clipped} to zero whenever they are sufficiently small.
\begin{prop}\label{prop:surplus_clipping_simple} Let $\gapcliph(x,a) :=  \frac{\gapmin}{2H} \vee \frac{\gaph\xa}{4H}$. Then, if $\pik$ is induced by an optimistic algorithm with surpluses $\Ekh\xa$,
\begin{align*}
\vst_0 - \vpik_0 &\le 2e \sum_{h=1}^H\sum_{x,a} \wkh(x,a)\threshop{\gapcliph\xa}{\Ekh\xa}.
\end{align*}
If the algorithm is \emph{strongly optimistic}, and $\calM$ is a contextual bandits instance, we can replace $\gapcliph(x,a)$ with $\gapcliph(x,a) :=  \frac{\gapmin}{2H} \vee \frac{\gaph\xa}{4}$.
\end{prop}
The above proposition is a consequence of a more general bound, Theorem~\ref{thm:surplus_clipping}, given in \iftoggle{nips}{Appendix}{Section}~\ref{sec:proof_of_main}. 
Unlike gap-dependent bounds that appeal to hitting-time arguments, we \emph{do not} reason about when a suboptimal action $a$ will cease to be taken. Indeed,  an algorithm may still choose a suboptimal action $a$ \emph{even if} the surplus $\Ekh\xa$ is small, because \emph{future} surpluses may be large. Instead, we argue in two parts:
\begin{enumerate}
	\item A sub-optimal action $a \notin \pisth(x)$ is taken only if 
	$\qupkh\xa \geq \qsth\xast$ for some $a^\star \in \pisth(x)$, or equivalently in terms of the surplus, only if $\Ekh\xa + p(x,a)^\top(\vupkhpl - \vst_{k,h+1}) > \gaph\xa$.
	Thus if $\Alg$ selects a suboptimal action, then this is because either the \emph{current} surplus $\Ekh\xa$ is larger than $\Omega(\frac{\gaph\xa}{H})$, or the expectation over \emph{future} surpluses, captured by $p(x,a)^\top(\vupkhpl - \vst_{k,h+1})$ is larger than  $(1-\BigOh{\frac{1}{H}})\gaph\xa$. Intuitively, the first case occurs when $\xa$ has not been visited enough times, and the second when the future state/action pairs have not experienced sufficient visitation. In the first case, we can clip the surplus at $\Omega(\frac{\gaph\xa}{H})$; in the second, $\Ekh\xa + p(x,a)^\top(\vupkhpl - \vst_{k,h+1}) \le (1 + \BigOh{\frac{1}{H}})p(x,a)^\top(\vupkhpl - \vst_{k,h+1})$, and push the the contribution of $\Ekh\xa$ into the contribution of future surpluses. This incurs a factor of at most $(1 + \BigOh{\frac{1}{H}})^H \lesssim 1$, avoiding an exponential dependence on $H$. 
	\item Clipping surpluses for pairs $\xa$ for optimal $a \in \pisth(x)$ requires more care. We introduce ``half-clipped'' surpluses $\Ehclip\kh(x,a) := \threshop{\frac{\gapmin}{2H}}{\Ekh(x,a)}$ where \emph{all} actions are clipped at $\gapmin/2H$, and recursively define value functions $\vhclippik_h(\cdot)$ corresponding to these clipped surpluses (see Definition~\ref{defn:halfclip}). We then show that, for $\vhclippik_0 := \Exp_{x \sim p_0}\left[\vhclip_1(x)\right]$, we have (Lemma~\ref{lem:half_clip})
	\begin{align*}
	\vst_0 - \vpik_0 \le 2 (\vhclippik_0  - \vpik_0).
	\end{align*}
	This argument is based on carefully analyzing when $\pikh$ first recommends a suboptimal action $\pikh(x) \notin \pist(x)$, and showing that when this occurs, $\vst_0 - \vpik_0$ is roughly lower bounded by $\frac{\gapmin}{H}$ times the probability of visiting a state $x$ where $\pikh(x)$ plays suboptimally. We can then subtract off $\frac{\gapmin}{2H}$ from all the surplus terms at the expense of at most halving the suboptimality, and using the fact $\Ekh - \frac{\gapmin}{2H} \le \clip{\frac{\gapmin}{2H}}{\Ekh}$ concludes the bound. This step is crucial, because it allows us to clip the surpluses even at pairs $(x,a)$ where $a \in \pisth(x)$ is the optimal action. We note that in the formal proof of Proposition~\ref{prop:surplus_clipping_simple}, this half-clipping precedes the clipping of suboptimal actions described above.

\end{enumerate}

Unfortunately, the first step involving the half-clipping is rather coarse, and leads to  $S/\gapmin$ term in the final regret bound. As argued in~\Cref{thm:lower_bound_informal}, this is unavoidable for existing optimistic algorithms, and suggests that~\Cref{prop:surplus_clipping_simple} cannot be significantly improved in general.

\subsection{Analysis of $\strongeuler$}  

Recall that $\strongeuler$ is precisely described by Definition~\ref{def:optimistic_alg} up to our particular choice of confidence intervals defined  (see~\Cref{alg:strong_euler} in Appendix~\ref{sec:alg:strong_euler}). We now state a surplus bound (proved in Appendix~\ref{sec:prop_surplus_proof}) that holds for these particular choice of confidence intervals, and which ensures that the strong optimism criterion of Definition~\ref{def:optimistic_alg} is satisfied:
\begin{prop}[Surplus Bound for Strong Euler (Informal)]\label{prop:surplus_informal} Let $M = SAH$, and define the variances $\Varxah := \Var[R\xa] + \Var_{x'\sim p\xa}[\vsthpl(x')]$. Then,  with probability at least $1 - \delta/2$, the following holds for all $(x,a) \in \states \times \actions$, $h \in [H]$ and $k \ge 1$, 
	\begin{align*}
	0 \le \Ekh(x,a) &\lesssim \underbracer{\boundlead\kh\xa}{\sqrt{\frac{\Varxah \log(M\nk\xa/\delta)}{\nk\xa}}} + \text{ lower order terms}.
	\end{align*}
\end{prop}
We emphasize that \Cref{prop:surplus_informal}, and its formal analogue Proposition~\ref{prop:surplus} in Appendix~\ref{sec:proof_prelim}, are the \emph{only} part of the analysis that relies upon the particular form of the $\strongeuler$ confidence intervals; to analyze other model-based optimistic algorithms, one would simply establish an analogue of this proposition, and continue the analysis in much the same fashion. While Q-learning \cite{jin2018q} also satisfies optimism, it induces a more intricate surplus structure, which may require a different analysis.

Recalling the clipping from Proposition~\ref{prop:surplus_clipping_simple}, we begin the gap-dependent bound with $\sum_{k=1}^K \vst_0 - \vpik_0 \lesssim \sum_{x,a,k,h} \wkh\xa \threshop{\gapcliph\xa}{\Ekh(x,a)}$.  Neglecting lower order terms, \Cref{prop:surplus_informal} ensures that this is approximately less than $\sum_{x,a,k,h} \wkh\xa \threshop{\gapcliph\xa}{\boundlead\kh(x,a)}$. Introduce the minimal (over $h$) clipping-gaps $\gapclip\xa := \min_h \gapclip\xa \ge \tfrac{\gap\xa \vee \gapmin}{4H}$ and maximal variances $\Varxa := \max_h \Varxah$. We can then render $\boundlead\kh\xa \le f(\nk\xa)$, where $f(u)\lesssim \threshop{\gapclip\xa }{\sqrt{\frac{1}{u}\Varxa \log(Mu/\delta)}}$.  Recalling the approximation $\nk\xa \approx \nbark\xa$ described above, we have, to first order,
\begin{align*}
\sum_{k=1}^K \vst_0 - \vpik_0 &\lesssim \sum_{x,a,k,h} \wkh\xa \threshop{\gapcliph\xa}{\boundlead\kh(x,a)} \\
&\lesssim \sum_{x,a,k} \wk\xa f(n_k\xa) \lesssim \sum_{x,a,k} \wk\xa f(\nbark\xa),
\end{align*}
where we recall the expected visitations $\wk\xa := \sum_{h=1}^H\wkh\xa$. Since $\nbark\xa := \sum_{j = 1}^{k} \wj\xa$, we can regard the above as an \emph{integral} of the function $f(u)$ (see Lemma~\ref{lem:integration}), with respect to the \emph{density} $\wk\xa$. Evaluating this integral (Lemma~\ref{lem:integral_computations}) yields (up to lower order terms)
\begin{align*}
\sum_{k=1}^K \vst_0 - \vpik_0 \lessapprox \sum_{x,a} \frac{H\Varxa \log \tfrac{MT}{\delta}}{\min_h \gapcliph\xa } \lessapprox \sum_{x,a} \frac{H\Varxa \log\tfrac{M T }{\delta}}{\gap\xa \vee \gapmin}.
\end{align*}
Finally, bounding $\Varxa \le H^2$ and splitting the bound into the states $\Zsub:= \{\xa :\gap\xa > 0\}$ and $\Zopt := \{\xa: \gap\xa = 0\}$ recovers the first two terms in Corollary~\ref{cor:log_regret}. In benign instances (Definition~\ref{def:benign_instances}) , we can bound $ \Varxah \lesssim 1$, improving the $H$-dependence. In contextual bandits, we save an addition $H$ factor via $\gapcliph\xa \gtrsim (\gapmin/H) \vee \gap\xa$. The interpolation with the minimax rate in Theorem~\ref{thm:main_regret_bound} is decribed in greater detail in Appendix~\ref{sec:proof_of_interpolated}.

\section{Conclusion}

In this paper, we proposed a new approach for providing logarithmic, gap dependent bounds for tabular MDPs in the episodic, non-generative setting. Our approach extends to any of the \emph{model-based}, optimistic algorithms in the present literature. Extending these bounds to model-free approaches based on Q-learning (e.g., \cite{jin2018q}), and resolving the optimal horizon dependence are left for future work.

While we found that our models \emph{nearly} matched information-theoretic lower bounds analogous, we also demonstrated that existing optimistic algorithms (both model-based and model-free) necessarily incur an additional dependence on $S/\gapmin$ for worst-case instances. It would be interesting to understand if one can circumvent this limitation by augmenting the principle of optimistism with new algorithmic ideas. 

Lastly, it would be exciting to extend logarithmic bounds to settings where taking advantage of the suboptimality gaps is \emph{indispensable} for attaining non-trivial regret guarantees. For example, we hope our techniques might enable sublinear regret in ``infinite arm'' settings, where $\calA$ is a countably infinite set, and actions $a \in \calA$ are drawn from an (unknown) reservoir distribution. It would also be interesting to extend our tools to adaptive discretization when the state-action pairs embed into a metric space \citep{song2019efficient}, and to the function approximation settings \citep{jin2018q}.

\clearpage
\bibliographystyle{plainnat}
\bibliography{main}
\clearpage

\clearpage
\appendix
\tableofcontents

\newpage
\section{Notation and Organization\label{sec:appendix_notation}}
\textbf{Organization: } This section describes the organization of the appendix, and clarifying our notation. The remainder of the appendix is divided into three parts: 

Part~\ref{part:main_bound} presents more detailed statements of the regret upper bounds obtained by $\strongeuler$, and their complete proofs. Section~\ref{sec:proof_prelim} introduces Corollary~\ref{cor:precise_log_bound} and Theorem~\ref{thm:main_regret_bound_ht}, refining Corollary~\ref{cor:log_regret} and Theorem~\ref{thm:main_regret_bound}, from the main text. The section continues to prove both results. In addition, we introduce  Theorem~\ref{thm:surplus_clipping}, which refines the clipped regret decomposition Proposition~\ref{prop:surplus_clipping_simple}. The proofs in this section rely on numerous technical lemmas, whose proofs are defered to Section~\ref{sec:general_technical}. Finally, this section states Proposition~\ref{prop:surplus}, which ensures that $\strongeuler$ is optimistic and provides a precise bound on the surpluses $\Ekh\xa$, described informally in Proposition~\ref{prop:surplus_informal}.

In Part~\ref{part:strongeuler}, we present the $\strongeuler$ algorithm  and its guarantees. Section~\ref{sec:alg:strong_euler} describes how $\strongeuler$ instantiates the \emph{model-based} examples of optimistic algorithms described in \Cref{sec:optimistic_algs}; the algorithm and choice of confidence bonuses are specified in pseudocode. In Section~\ref{sec:prop_surplus_proof}, we prove the surplus bound Proposition~\ref{prop:surplus}, and verify that $\strongeuler$ is strongly optimistic 

Lastly, Part~\ref{part:lb} contains the proofs of our lower bounds. Section~\ref{sec:min-gap-lower-bound} proves the $\Omega(S/\gapmin)$ lower bound described in Theorem~\ref{thm:lower_bound_informal}, and rigorously describes the class of algorithms to which it applies. Finally, Section~\ref{sec:inf_th_lb} proves the information theoretic lower bound, Proposition~\ref{prop:info_th_lower_bound}.

\textbf{Notational Rationale:} Unfortunately, the regret analysis of tabular MDPs requires significant notational overhead. Here we take a moment to highlight some notational conventions that we shall use throughout. The superscript $(\cdot)^{\star}$ denotes ``optimal'' quantities, i.e. the optimal policy $\pist$, the optimal value $\vst$, and variances of the optimal policy $\Varxah$. The accents $\overline{(\cdot)}$ will be used to denote upper bounds on quantities, e.q. an optimistic Q-function $\qup\kh$ is an upper bound on $\qst\kh$, and $\Varmax$ is an upper bound on the variance, and so on. $\underline{(\cdot)}$ will denote lower bounds on quantities. For example, $\strongeuler$ will maintain lower bounds on the values $\vlow\kh \le \vst$. The accent $\check{(\cdot)}$ will pertain to clipped quantities; e.g. $\gapclip$ is the gap-value at which surpluses are clipped. Many quantities, like $\gaph\xa$ (gaps) and $\Varxah$ (variances) depend on the triples $(x,a,h)$. The quantities $\gap\xa$ and $\Varxa$ with $h$ suppresed to denote worse-case bounds on these term over $h \in [H]$; e.g. $\gap\xa := \min_{h\in [H]}\gaph\xa$ and $\Varxa := \max_{h\in [H]}\Varxah$. 

\addcontentsline{toc}{section}{Notation Table}
\begin{table}[]
\flushleft
\begin{tabular}{| l | l |}
\hline
\textbf{General Notation} \\
\hline
$\lesssim$ denotes inequality up to a universal constant.\\
$f \eqsim g$ denotes $f \lesssim g \lesssim f$. 
$\log_+(x) := \log \max\{x,1\}$.\\
$\I$ denotes an indicator function\\
$\calM = (\states,\actions,[H],p_0,p,R)$ denotes an MDP\\
$H$ denote the horizon,\\
$\actions$  and $\states$ denotes the space of actions and states \\
$A := |\actions|$ and $S := |\states|$\\
$h \in [H]$, $a \in \actions$, $x \in \states$ are used for stages, actions, and states \\
$R(x,a) \in [0,1]$ denotes the R.V. with reward distribution at $\xa$.\\
$r(x,a) := \Exp[R(x,a)]$ denotes expected reward\\
$p_0(x)$ denotes initial distribution of $x_1$\\
$p(x'|x,a)$ denotes transition probability.\\
$M = SAH$\\ 
$K$ denotes number of episodes, indexed with $k \in [K]$ \\
$T = KH$ denotes total length of game.
\\
\hline
\end{tabular}
\end{table}

\begin{table}[]
\flushleft
\begin{tabular}{| l | l |}
\hline
\textbf{Policies, Value Functions, Q-functions} \\
\hline
$\pi = (\pi_h)_{h=1}^H$ denotes a policy with $\pi_h:\states \to \actions$\\
$\truev_0^{\pi}$ denotes the value of $\pi$\\
$\truev_h^{\pi}(x)$ denotes the value of $\pi$ at $h \in [H]$ and $x \in \states$ \\
$\trueq_h^{\pi}(x,a)$ denotes Q-function of $\pi$\\
$\vst_0,\vst_h(x),\qst(x,a)$ denote optimal value, value function, Q-function\\
$\pist_h(x)$ denotes the set of optimal actions at $h \in [H]$, $x \in \states$. \\
$\vupkh(x)$/$\qupkh(x,a)$ denotes optimistic value/Q function \\
$\pikh(x) = \argmax_{a} \qupkh(x,a)$ denotes optimistic policy \\
\hline
\end{tabular}
\end{table}

\begin{table}[]
\flushleft
\begin{tabular}{| l | l |}
\hline
\textbf{Problem Dependent Quantities} \\
\hline
$\gaph(x,a) := \vsth(x) - \qsth(x,a)$. 
\\
$\gap(x,a) := \min_h \gaph\xa$\\
$\gapmin := \min_{x,a}\{\gaph\xa : \gaph\xa > 0\}$\\
$\bsfaxah \in [0,1]$ denotes transition suboptimality (Definition~\ref{def:transition_suboptimality})\\ 
\hline
$\Varxah := \Var[R(x,a)] + \Var_{x' \sim p(x,a)}[\vsthpl(x)]$
\\
$\Varxa := \max_h\Varxa$
\\
$\Varmax := \max_{x,a,h}\Varxah$.\\
\hline
$\Gterm \le H$: upper bound on $\sum_{h=1}^H R(x,\pi_h(x))$ (Definition~\ref{def:benign_instances})
\\
$\HeffT := \min\{\Varmax, \Gterm^2/H\}$\\
\hline
\end{tabular}
\end{table}

\begin{table}[]
\flushleft
\begin{tabular}{| l | l |}
\hline
\textbf{Quantities for Analysis} \\
\hline
$\Lfactor(u) := \sqrt{ 2\log(10 M^2\max\{u,1\}^2/\delta)}$\\
$\wkh\xa := \Pr^{\pik}[(x_h,a_h) = \xa]$ denotes the surplus \\
$\wk\xa :=\sum_{h=1}^H \wkh\xa$ denotes the surplus \\
$\nk\xa $ denotes the number of times $\xa$ is observed up to time $k-1$\\
$\nbark\xa := \sum_{t=1}^k \weight_t\xa$. \\
$\sampletime(x,a)$ denotes time after which $\nbark\xa$ is sufficiently large\\
$\goodconcentration$ (good concentration event) \\
$\eventsample$ (good sampling event, Lemma~\ref{lem:sample}) \\
$\Hsample \lesssim H\log \frac{M}{\delta}$ number of sampes for $\eventsample$ to apply \\
\hline
$\gapcliph(x,a) :=  \frac{\gapmin}{2H} \vee \frac{\gaph\xa}{4(H\bsfa_{x,a,h} \vee 1)}$ (clipped gap)\\
$\gapclipmin := \min_{x,a,h}\gapcliph\xa$ (clipped gap)\\
\hline
$\Varpixah :=  \Var[R\xa] + \Var_{x'\sim p\xa}[\vpihpl(x')]$\\ 
$\Varkxah = \min\left\{\Varxah, \Varpikxah\right\}$\\
\hline
$\Ekh\xa := \qupkh(x,a) - r(x,a) - p(x,a)^\top\vupkhpl$ denotes the surplus \\
$\Ekh\xa \lesssim \boundlead\kh\xa  + \Exppik\left[\sum_{t=h}^H \boundfut_k(x_t,a_t) \mid \xhah = \xa\right]$ \quad (surplus bound, Proposition~\ref{prop:surplus})\\
$\boundlead\kh\xa := H \wedge \sqrt{\frac{\Varkxah \log\left(\frac{M\nk\xa}{\delta}\right)}{\nk\xa}}$ (lead bound on surplus)\\
$\boundfut_k\xa = H^3 \wedge H^3\left(\sqrt{\frac{S\log\left(\frac{M\nk\xa}{\delta}\right)}{\nk\xa}} + \frac{S\log\left(\frac{M\nk\xa}{\delta}\right)}{\nk\xa} \right)^2$ (bound on future surpluses)\\
\hline
\end{tabular}
\end{table}

\newpage

\newpage

\newpage
\part{Precise Results and Analysis\label{part:main_bound}}

\section{Precise Statement and Rigorous Proof Sketch of Main Regret Bounds \label{sec:proof_of_main}}
In this section, we present a precise statements and formal proofs of the upper bounds, Corollary~\ref{cor:log_regret} and Theorem~\ref{thm:main_regret_bound}, from the main text. These bounds both makes the improvements in the benign instances of Definition~\ref{def:benign_instances}, and takes advantage of other possibly-favorable instance-specific quantities.  The remainder of the section is organized as follows. In Section~\ref{sec:main_reg_ht}, we introduce the relevant problem-dependent quantitites in terms of which we state our more refined bounds. We then state Corollary~\ref{cor:precise_log_bound}, a more precise analogue of the $\log$-regret bounds in Corollary~\ref{cor:log_regret}, followed by Theorem~\ref{thm:main_regret_bound_ht}, which refines the regret bound Theorem~\ref{thm:main_regret_bound} in interpolating between the $\log T$ and $\sqrt{T}$ regimes. 

Next, in Section~\ref{sec:proof_prelim}, we set up the preliminaries for the proof of our upper bound, including (a) Theorem~\ref{thm:surplus_clipping}, the granular clipping bound strengthening Proposition~\ref{prop:surplus_clipping_simple}, (b) Proposition~\ref{prop:surplus}, which upper bounds the surpluses $\Ekh\xa$ for $\strongeuler$, and (c) Lemma~\ref{lem:clipped_regret_with_future} which combines the two into a useful form.

Then, in Section~\ref{sec:proof_prelim}, we present a rigorous proof of Corollary~\ref{cor:precise_log_bound} based on integration tools developed in Section~\ref{sec:general_technical}. Finally, we modify the arguments slightly to obtain the interpolation in Theorem~\ref{thm:main_regret_bound_ht}.  The proof of Proposition~\ref{prop:surplus_clipping_simple} is given in Section~\ref{sec:prop_surplus_proof}, Theorem~\ref{thm:surplus_clipping} is given in Section~\ref{sec:clipping_proof}, and the remainder of technical results in the present section are established in Section~\ref{sec:general_technical}.

We emphasize that the tools in this section provide a general recipe for establishing similar regret bounds for the existing model-based optimistic algorithms in the literature. We have attempted to present our tools in a modular fashion in hope that they can be borrowed to automate the proofs of similar guarantees in related settings.

\subsection{More Precise Statement of Regret Bound Theorem~\ref{thm:main_regret_bound}\label{sec:main_reg_ht}}
	We shall begin by stating a more precise version of Theorem~\ref{thm:main_regret_bound},  Following \cite{zanette2019tighter}, we begin by defnining the variances of the value optimal functions::
	\begin{defn}[Variance Terms] We define the variance of a triple $\xah$ as 
	\begin{align*}
	\Varxah := \Var[R(x,a)] + \Var_{x'\sim p(x,a)}[\vsthpl(x')],
	\end{align*}
	and the statewise maximal variances as $\Varxa := \max_h \Varxah$, and the maximal variance as $\Varmax := \max_{x,a,h} \Varxah$.
	\end{defn}
	\begin{rem}[Typical Bounds on the variance ]
	While $\Varmax \le H^2$ for general MDPs (see e.g.~\citep{azar2017minimax}), we have  $\Varmax$ is smaller for the benign instances in Definition~\ref{def:benign_instances}. We briefly summarize this discussion from \cite{zanette2019tighter}: If $\calM$ has $\calG$ bounded rewards, then $\vsthpl(x) \le \calG$ for any $x$, and thus $\Varxah \le 1 + \calG^2$, which is $\lesssim 1$ if $\calG \lesssim 1$. For contextual bandits, $p = p(x,a)$ does not depend on $x,a$, and $\vsthpl(x) = (\max_{a} R(x,a)) + (\Exp_{x' \sim p}\vst_{h+2}(x'))$, where the second term does not dependent  Hence, $\Var_{x' \sim p}[\vsthpl(x')] \le \Var[(\max_{a} R(x,a))] \le 1$, and thus $\Varxah \le 2$. 
	\end{rem}


	We can then define an associated ``effective horizon'', which replaces $H$ with a possibly smaller problem dependent quantity: 

	\begin{defn}[Effective Horizon]\label{def:eff_horizon} Suppose that $\calM$ has $\calG$-bounded rewards, as in Definition~\ref{def:benign_instances}) We define the \emph{effective horizon} as
		\begin{align*}
	\HeffT  := \min\left\{\Varmax, \frac{\Gterm^2}{H}\log  T \right\}\,.
	\end{align*}
	Since any horizon-$H$ MDP has $H$-bounded rewards, $\HeffT$ always satisfies $\HeffT \le \min\left\{H^2, H\log T \right\}$. 
	\end{defn}
	We note that the bound $\Varmax \lesssim 1$ for contextual bandits implies $\HeffT   \lesssim 1$, whereas if $\calM$ has $\Gterm$-bounded rewards with $\Gterm \lesssim 1$, $\HeffT  \lesssim 1 \wedge \frac{1}{H}\log T$. 

	Lastly, we shall introduce one more condition we call \emph{transition suboptimality}, which is a notion of distributional closeness that enables the improved clipping and sharper regret bounds for the special case of contextual bandits (Definition~\ref{def:benign_instances}):

	\begin{defn}[Transition Sub-optimality]\label{def:transition_suboptimality} Given $\bsfa \in [0,1]$, we say that a tuple $(x,a,h)$ is $\bsfa$-transition suboptimal if there exists an $\astar \in \pisth(x)$ such that
	\begin{align*}
	p(x'|x,a) - p(x'|x,\astar) &\le~ \bsfa p(x'|x,a) \quad \forall x' \in \states.
	\end{align*}
	\end{defn}
	Intuitively, the condition states that the transition distributions $p(x,a)$ and $p(x,\astar)$ are close in a pointwise, multiplicative sense. This is motivated by the contextual bandit setting of Definition~\ref{def:benign_instances}, where each $(x,a,h)$ is exactly $0$-transition suboptimal. For arbitrary MDPs, the bound $p(x'|x,\astar) \ge 0$ implies that every triple $(x,a,h)$ is at most $1$-transition suboptimal.\footnote{The condition can be relaxed somewhat to only needing to hold for a set $\calS$ for which $p(x' \in \calS \mid x,a)$ is close to $1$; for simplicity, consider the unrelaxed notion as defined as above.} 

	With these definitions in place, we can state the more precise analogue of Corollary~\ref{cor:log_regret} as follows:
	\begin{cor}[Logarithmic Regret Bound for $\strongeuler$]\label{cor:precise_log_bound} Fix $\delta \in (0,1/2)$, and let $A = |\actions|$, $S= |\states|$, $M = (SAH)^2$. Then with probability at least $1-\delta$, $\strongeuler$ run with confidence parameter $\delta$ enjoys the following regret bound for all $K \ge 2$:
	\begin{align*}
		\regret_K &\lesssim  \sum_{(x,a) \in \Zsub} \frac{\Varxa (1 \vee \bsfa H)}{\gap\xa} \log(\tfrac{M}{\delta}(\tfrac{H}{\gap\xa} \wedge T))+  |\Zopt|\frac{\Varmax H}{\gapmin\xa} \log(\tfrac{M}{\delta}(\tfrac{H}{\gap\xa} \wedge T)) \\
		&\qquad+  H^4 S A (S \vee H)\log \tfrac{TM}{\delta} \min \left\{ \log\tfrac{HM}{\gapmin},  \log  \tfrac{TM}{\delta}\right\}\\
		\end{align*}
		 In particular, if $\calM$ is an instance of contextual bandits, then $\Varmax$ can be replaced by $1$, $\HeffT$ can be replaced by $1$ and $\max\{\bsfa H , 1\} = 1$. If $\calM$ has $\calG \lesssim 1$ bounded rewards, then $\Varmax$ can be replaced by $1$ in the above bound,
	\end{cor}
	Moreover, our more precise analogue of Theorem~\ref{thm:main_regret_bound}, which interpolates between the $\log T$ and $\sqrt{T}$ regimes, is as follows:
	\begin{thm}[Main Regret Bound for $\strongeuler$]\label{thm:main_regret_bound_ht} Fix $\delta \in (0,1/2)$, and let $A = |\actions|$, $S= |\states|$, $M = (SAH)^2$. Let $\HeffT$ be as in Definition~\ref{def:eff_horizon}, and suppose that each tuple $(x,a,h)$ is $\bsfa$-transition suboptimal. Futher, define  $\Zsub(\epsilon):= \{(x,a) \in \Zsub : \gap\xa < \epsilon\}$. Then with probability at least $1-\delta$, $\strongeuler$ run with confidence parameter $\delta$ enjoys the following regret bound for all $K \ge 2$:
		\begin{align*}
		\regret_K &\lesssim \min_{\epsilon > 0} \left\{ \sqrt{ \HeffT\,|\Zsub(\epsilon)|  \, T \log \tfrac{M T}{\delta}} + \sum_{(x,a)\in \Zsub  \setminus \Zsub(\epsilon)} \frac{\max\{\bsfa H,1 \}\, \Varxa }{\gap\xa} (\tfrac{M}{\delta}(\tfrac{H}{\gap\xa} \wedge T))\right\}\\
		&\qquad+  \min\left\{\sqrt{\HeffT\,|\Zopt| \,T  \log  \tfrac{M T}{\delta}},\,\, |\Zopt| \, \frac{H\Varmax}{\gapmin} \log(\tfrac{M}{\delta}(\tfrac{H}{\gap\xa} \wedge T))\right\}\\
		&\qquad+  H^4 S A (S \vee H) \log\tfrac{MT}{\delta}\min \left\{  \log \frac{MH}{\gapmin},  \log \tfrac{MT}{\delta}\right\}\\
		&\lesssim \sqrt{HSAT \log(T) \log(\tfrac{MT}{\delta})} + H^4 S A (S \vee H) \log^2 \tfrac{TM}{\delta},
		\end{align*}
	where the second inequality follows from the first with $\max\{\max_\epsilon |\Zsub(\epsilon)|,|\Zopt|\} \leq SA$. In particular, if $\calM$ is an instance of contextual bandits, then $\Varmax$ can be replaced by $1$, $\HeffT$ can be replaced by $1$ and $\max\{\bsfa H , 1\} = 1$. If $\calM$ has $\calG \lesssim 1$ bounded rewards, then $\Varmax$ can be replaced by $1$ in the above bound, and $\HeffT$ replaced by $\min\{1,\frac{\log T}{H}\}$. 
	\end{thm}
	We observe that Theorem~\ref{thm:main_regret_bound}, and Corollary~\ref{cor:log_regret} are direct consequences of the above theorem. 
	\begin{rem}[Bounds on $|\Zopt|$]\label{rem:Z_opt} Note that $|\Zopt|\le \sum_{x,h}|\pisth(x)|$; in particular if for each $(x,h)$ there is exactly one optimal action, then $|\Zopt|\le H|\states|$. If in addition the same action is optimal at $x$ for each $h \in [H]$, then $|\Zopt| = |\states|$. For many environments $|\Zopt|\lesssim |\states|$; for instance, a race car doing many laps around a track may have $h$-dependent optimal actions in the first and last laps, but for the steady-state laps the optimal action will depend just on the current state. 
	\end{rem}
	\begin{rem}[Coupling Variances and Gaps] For state action pairs $\xa \in \Zsub$, Corollary~\ref{cor:precise_log_bound} Theorem~\ref{thm:main_regret_bound_ht} suffer for the term $\frac{(1\vee \bsfa H)\Varxa}{\gap\xa}$, where $\Varxa := \max_{h}\Varxah$ is the maximal variance over stages, and $\gap\xa = \min_{h}\gap_h\xa$ is the minimal gap. This quantity can be refined to defend on roughly $\max_h \frac{\Varxah}{\gaph\xa}$, coupling the variance and gap terms. To do so, one needs to bin the gaps into intervals of $[2^{j-1}H,2^j]$ or integers $j \in \N$, and apply numerous careful manipulations. In the interest of brevity, we defer the details to a later work. 
	\end{rem}

\subsection{Rigorous proof of upper bounds: Preliminaries \label{sec:proof_prelim}}

We now turn to a rigorous proof of the regret bounds for  $\strongeuler$: Corollary~\ref{cor:precise_log_bound} and Theorem~\ref{thm:main_regret_bound_ht} (and consequently Theorem~\ref{thm:main_regret_bound} and Corollary~\ref{cor:log_regret}).  

We first state our generalized surplus clipping bound in terms of the transition-suboptimality condition, which generalizes Proposition~\ref{prop:surplus_clipping_simple}: 
\begin{thm}\label{thm:surplus_clipping} Suppose that each tuple $(x,a,h)$ is $\bsfa_{x,a,h}$ transition-suboptimal, and set $\gapcliph(x,a) :=  \frac{\gapmin}{2H} \vee \frac{\gaph\xa}{4(H\bsfa_{x,a,h} \vee 1)}$. Then, if $\pik$ is induced by a strongly optimistic algorithm with surpluses $\Ekh\xa$,
\begin{align*}
\vst_0 - \vpik_0 &\le 2e \sum_{h=1}^H\sum_{x,a} \wkh(x,a)\threshop{\gapcliph\xa}{\Ekh\xa}.
\end{align*}
If the algorithm is optimistic but not strongly optimistic, then the above holds by replacing $\bsfa_{x,a,h}$ with $1$ in the definition of $\gapcliph\xa$.
\end{thm}
The proof of the above theorem is given in Section~\ref{sec:clipping_proof}. We remark that the above theorem specializes to Proposition~\ref{prop:surplus_clipping_simple} by noting that each tuple $(x,a,h)$ is $0$-transition suboptimal for contextual bandits. For simplicitiy, we shall assume in the proof of Theorem~\ref{thm:main_regret_bound_ht} that each state is $\bsfa$-suboptimal for a common $\bsfa$; the bound can be straightforwardly refined to allow $\bsfa$ to vary across $\xah$.

Next, in order to ensure optimal $H$-dependence when interpolating with the $\BigOh{\sqrt{T}}$ regret bounds, we introduce policy-dependent variance quantities:
\begin{defn}  Define the variances 
\begin{align*}
\Varpixah :=  \Var[R\xa] + \Var_{x'\sim p\xa}[\vpihpl(x')].
\end{align*}
where we recall that $\Varxah := \Var^{\pi_*}_{h,x,a}$. Further, define  $\Varkxah = \min\left\{\Varxah, \Varpikxah\right\}$. 
\end{defn}

We are now ready to state the formal version of Proposition~\ref{prop:surplus_informal}, which upper bounds the surpluses of $\strongeuler$, and verifies that the algorithm satisfies strong optimism:
\begin{restatable}[Surplus Bound for $\strongeuler$]{prop}{surplusbound}\label{prop:surplus} There exists a universal constant $\errconst \ge 1$ and event $\goodconcentration$, with $\Pr[\goodconcentration] \ge 1 - \delta/2$, such that on $\goodconcentration$, for all $x \in \states$, $a \in \actions$, $h \in [H]$ and $k \ge 1$, 
\begin{align*}
	0 \le \frac{1}{\errconst}\Ekh(x,a) &\le \boundlead\kh\xa +\Exppik\left[\sum_{t=h}^H \boundfut_k(x_t,a_t) \mid (x_h,a_h) = \xa\right].
\end{align*}
where have defined the terms
\begin{align*}
&\boundlead\kh\xa := H \wedge \sqrt{\frac{\Varkxah \log\left(\frac{M\nk\xa}{\delta}\right)}{\nk\xa}}, \quad \text{ and }\\
&\qquad\boundfut_k\xa = H^3 \wedge H^3\left(\sqrt{\frac{S\log\left(\frac{M\nk\xa}{\delta}\right)}{\nk\xa}} + \frac{S\log\left(\frac{M\nk\xa}{\delta}\right)}{\nk\xa} \right)^2.
\end{align*}
\end{restatable}
The above proposition is proven in Appendix~\ref{sec:prop_surplus_proof}. 
Here $\boundlead$ denotes a ``lead term'' in the analysis, which contributes to the dominate factors in our regret bounds. $\boundfut$ notates ``future'' bound terms under a rollout of $\pik$ starting at a given triple $(x,a,h)$; these terms are responsible for the lower $\BigOhTil{SAH^4(S \vee H)}$-term in the regret. 
\begin{rem}[Remarks on Proposition~\ref{prop:surplus}] First, the dominant term in the upper bound on $\Ekh$ is $\boundlead\kh\xa$, which decays as $\BigOhTil{\nk\xa^{-1/2}}$. The terms $\boundfut_k\xa$ decay more rapidly $\BigOhTil{\nk\xa^{-1}}$, and will thus be responsible for the (nearly gap-free) portion of the regret. Second, in order to analyze similar optimistic algorithms in the same vein (e.g.~\citep{azar2017minimax,dann2015sample,dann2017unifying}), one would instead prove the appropriate analogue to Proposition~\ref{prop:surplus} and follow the remaining steps of the present proof. Little would change, except one would be forced to replace $\Varxah$ with a more pessimistic, less problem-dependent quantity. Lastly, note that the lead term $\boundlead\kh\xa$ depends on the \emph{minimum} of the variance of the optimal value function, $\Varxah$ and of the variance of the value function for $\pik$, $\Varpikxah$. As in the aforementioned works, this dependence on $\Varpikxah$ is crucial for obtaining the correct minimax $\widetilde{O}(\sqrt{HSAT})$ regret.
\end{rem}

Next, let us combine Proposition~\ref{prop:surplus} with our main clipping theorem, Theorem~\ref{thm:surplus_clipping}. Since $\Ekh\xa \lesssim \boundlead\kh\xa + \Exppik\left[\sum (\dots) \mid (\dots)\right]$, combining the two results into a convenient form requires that we reason about how to distribute clipping operations across sums of terms. To this end, we invoke the following technical lemma:
\begin{lem}[Distributing the clipping operator]\label{lem:clipping_lemma} Let $m \ge 2$,  $a_1,\dots,a_m \ge 0$, and $\epsilon \ge 0$. $\threshop{\epsilon}{\sum_{i=1}^m a_i} \le 2\sum_{i=1}^m \threshop{\frac{\epsilon}{2m}}{a_i} $. 
\end{lem}
\begin{proof}
Let us assume  without loss of generality $0 \le a_1 \le \dots \le a_m$, and that $\sum_{i=1}^m a_i \ge \epsilon$. Defining the index $i^* := \min\{i: a_i \ge \frac{\epsilon}{2m}\}$, we observe that $a_{i^*} \ge \frac{\epsilon}{2m}$, and since $(a_i)$ are non-decreasing by assumption, $\sum_{i=i^*}^m a_i = \sum_{i=i^*}^m \threshop{\frac{\epsilon}{2m}}{a_i} \le \sum_{i=1}^m \threshop{\frac{\epsilon}{2m}}{a_i} $. It therefore suffices to show that $\sum_{i=1}^{m}a_i \le 2\sum_{i=i^*}^m a_i$. To this end, we see that, since $a_i \le \frac{\epsilon}{2m} $ for $i < i^*$,  $\sum_{i=1}^{i^*-1}a_i \le \sum_{i=1}^{i^*-1}\frac{\epsilon}{2m} \le \frac{(i^*-1)\epsilon}{2m} \le \epsilon/2$. On the other hand, since $\sum_{i=1}^m a_i \ge \epsilon$, we must have that $\sum_{i=i^*}^m a_i \ge \frac{\epsilon}{2}$, and thus $\sum_{i=1}^{m}a_i \le 2\sum_{i=i^*}^m a_i$, as needed.
\end{proof}
Applying the above lemma careful, we arrive at the following useful regret decomposition:
\begin{lem}[Clipped Regret Decomposition Lead and Future Bounds]\label{lem:clipped_regret_with_future} Let $\gapclipmin = \min_{x,a,h} \gapcliph\xa$. Then on the event $\goodconcentration$ the regret of $\strongeuler$ is bounded by
 \begin{align*}
	\regret_K &\le 4e\ \sum_{k=1}^K\sum_{h=1}^H\sum_{x,a} \wkh(x,a)\threshop{\frac{\gapcliph\xa}{4}}{\errconst\boundlead\kh\xa}.\\
	 &\qquad +8 He \sum_{k=1}^K\sum_{h=1}^H\sum_{x,a}\wkh(x,a)\clip{\frac{\gapclipmin}{8SAH}}{\errconst\boundfut_k(x',a')},
	\end{align*} 
where $\errconst$ is a universal constant.
\end{lem}

\subsection{Proof of Corollary~\ref{cor:precise_log_bound}: A proof via integration\label{sec:cor_proof}}
Note that Lemma~\ref{lem:clipped_regret_with_future} bounds $\regret_K$ by a sum of \emph{local} bounnds terms $\boundlead\kh\xa$ and $\boundfut_k$, which depend only on the number of samples $\nk\xa$ obtained from state action pair $\xa$. More precisesly, we can represent the bound terms by defining the functions
\newcommand{\glead}{g^{\mathrm{lead}}}
\newcommand{\gfut}{g^{\mathrm{fut}}}
\newcommand{\flead}{f^{\mathrm{lead}}}
\newcommand{\ffut}{f^{\mathrm{fut}}}
\newcommand{\epslead}{\epsilon^{\mathrm{lead}}}
\newcommand{\epsfut}{\epsilon^{\mathrm{fut}}}

\begin{align*}
\glead_{x,a}(u) :=\sqrt{\tfrac{\Varxa \log\left(\frac{Mu}{\delta}\right)}{u}}, \quad \text{ and } \quad\gfut(u) =  H^3\left(\sqrt{\tfrac{S\log\left(\tfrac{Mu}{\delta}\right)}{u}} + \tfrac{S\log\left(\tfrac{Mu}{\delta}\right)}{u} \right)^2.
\end{align*}
Further, define $\epslead_{x,a} := \min_h\frac{\gapcliph\xa}{4}$, $\epsfut := \frac{\gapclipmin}{8SAH}$, and lastly set
\begin{align*}
\flead_{x,a}(u) := H \wedge \clip{\epslead_{x,a}}{\errconst\glead_{x,a}(u)}, \quad \ffut(u) := H^3 \wedge \clip{\epsfut}{\errconst\gfut(u)}.
\end{align*}
Then, recalling the definitions of $\boundlead,\boundfut$, and the fact that $\boundlead\kh\xa \le H$ and $\boundfutk\xa \le H^3$, we can write
\begin{align}
\regret_K &\lesssim \sum_{k=1}^K\sum_{h=1}^H\sum_{x,a} \wkh(x,a)\flead_{x,a}(\nk\xa) + H\sum_{k=1}^K\sum_{h=1}^H\sum_{x,a}\wkh(x,a)\ffut(\nk\xa). \label{eq:function_reg_bound}
\end{align}
As described in Section~\ref{sec:optimism}, the crucial step now is to relate the empirical conunts $\nk\xa$ to the visitation probabililties. Precisely, let us aggregate
\begin{align*}
\wk\xa := \sum_{h=1}^H\wkh\xa, \quad \nbark\xa := \sum_{j=1}^k\weight_j\xa.
\end{align*}
Note that if $\{\calF_k\}$ denotes the filtration corresponding to the episodes $k$, then, $\Exp[\nk\xa \mid \calF_{k-1}] = n_{k-1}\xa + \weight_{k-1}\xa$. In other words, $\nbar_{k-1}\xa$ is precise the sum of the increments $\Exp[n_j\xa - n_{j-1}\xa\mid \calF_{j-1}]$ for $j = 1,\dots,k-1$.\footnote{Note that we induce $\nbark\xa$ to include a sum up to index $k$; this makes the following arguments more convenient, and will only accrue constant factors in the analysis}. Hence, by a now-standard martingale concentration argument, we find that  $\nk\xa$ will be be lower bounded by $\nbark\xa$, provided that the latter quantity is sufficiently large. More precisely:
\begin{lem}[Sampling Event]\label{lem:sample} Define the event
\begin{align*}
\eventsample(\Hsample)&:= \left\{\forall \xa, \,\forall k \ge \sampletime_{\Hsample}(x,a), \quad \nk\xa \ge \frac{1}{4}\nbark\xa \right\},\\
&\quad \text{where}\quad\sampletime_{n}\xa := \inf\left\{k: \nbark\xa  \ge n\right\}
\end{align*}
Then, for some $\Hsample \lesssim H \log \frac{M}{\delta}$, $\eventsample(\Hsample)$ holds with probability at least $1 - \delta/2$. 
\end{lem}
Lemma~\ref{lem:sample} is proved in Appendix~\ref{sec:lem_sample_proof} as a consequence of \iftoggle{nips}{Lemma 6 in~\cite{dann2018policy}}{\citet[Lemma 6]{dann2018policy}}. Together, the events $\goodconcentration$ and $\eventsample$ account for $1-\delta$ probability with which our regret bounds hold. For short, we will let $\eventsample$ denote $\eventsample(\Hsample)$ when clear from context, and $\sampletime := \sampletime_{\Hsample}$. 

After neglecting the first $\sampletime\xa$ samples in the sum~\eqref{eq:function_reg_bound}, we can approximately bound
\begin{align*}
\regret_K &\lessapprox \sum_{x,a}\sum_{k\ge \sampletime\xa}^K\sum_{h=1}^H \wkh(x,a)\flead_{x,a}(\nbark\xa/4) \\
&\quad+ \sum_{x,a} H\sum_{k=\sampletime\xa}^K\sum_{h=1}^H\wkh(x,a)\ffut(\nbark\xa/4),
\end{align*}
where $\lessapprox$ denotes an informal inequality. Now, $\wkh\xa$ and $\nbark\xa/4$ are directly related via $\nbark\xa/4 := \sum_{j=1}^k\sum_{h=1}^H\weight_{j,h}\xa$. Hence, we can view the above regret bounds as discrete integrals of the functions $\flead_{x,a}$ and $\ffut(\nbark\xa/4)$. This argument is made precise by the following lemma, which comprises the workhorse of out argument:
\begin{restatable}[Integral Conversion]{lem}{integralconversion}\label{lem:integral_conversion} Suppose that the event $\eventsample(\Hsample)$ holds. Then, for any collection of functions $f_{x,a}(\cdot)$ non-increasing functions from $\N \to \R$ bounded aboved by $f_{\max}$ and any $\epsilon_{x,a,h} \ge 0$ , we have that
\begin{align*}
\sum_{k=1}^K \sum_{h=1}^H\sum_{x,a} \wkh\xa f_{x,a}(\nk\xa) \le 2AS \Hsample f_{\max} + \sum_{x,a}\I(\nbarK \xa \ge H)\int_{H}^{\nbarK\xa}f(u/4)du
\end{align*}
\end{restatable}
Since $\Hsample \lesssim H\log (M/\delta)$, and the functions $\ffut_{x,a}$, $\flead_{x,a}$ are bounded by $H^3$, we see that, on $\eventsample \cap \goodconcentration$, it holds that
\begin{align}
\regret_K &\lesssim SAH^5\log (M/\delta) +  \sum_{x,a}\I(\nbarK \xa \ge H)\int_{H}^{\nbarK\xa}\flead_{x,a}(u/4)du \nonumber\\
&\qquad+ \sum_{x,a}\I(\nbarK \xa \ge H)\int_{H}^{T}\ffut(u/4)du, \label{eq:function_reg_bound}.
\end{align}
where for the term on the second line, we have bounded $\nbarK\xa \le T$ and used that $\ffut(\cdot) \ge 0$.


All that remains is to evaluate the above integrals. This is directly adressed by the following technical lemma, proved in Section~\ref{sec:integral_comp}:

\begin{restatable}[General Integration Computations]{lem}{integralcomputations}\label{lem:integral_computations}
Let $f(u) \le \min\left\{\fmax,\clip{\epsilon}{ g(u)}\right\}$ where $\epsilon \in [0,H]$ and $g(u)$ is a non-increasing function is specified in each of two cases that follow. Further, let $M \ge 1$, and $\delta \in (0,1/2)$ be problem dependent constants. Finally, let $\lesssim$ denote inequality up to a problem independent constant. Then, the following integral computations hold:
\begin{enumerate}
\item[(a)] Suppose that $C > 0$ is a problem depedendent constant satisfying $\log C \lesssim \log (2M)$, and that $g(u) \lesssim \sqrt{\frac{C \log(  M u/\delta )}{u}}$. Then,
\begin{align*}
\int_{H}^{N}f(u/4)du \lesssim \min\left\{\sqrt{CN \, \log\tfrac{M N}{\delta}},\, \frac{C}{\epsilon}\log\left( \frac{M}{\delta} \cdot \min\{T, \tfrac{H}{\epsilon}\}  \right)\right\} .
\end{align*}
\item[(b)]  Suppose that $C,C' > 0$ are a problem depedendent constant satisfying $\log (CC') \lesssim \log 2M$, and that $g(u) \lesssim C\left(\sqrt{\frac{C' \log(M u/\delta)}{u}} +\frac{C' \log(M u/\delta)}{u} \right)^2$. Then, 
\begin{align*}
\int_{H}^{N}f(u/4)du &\lesssim \left(1+C'\right)\fmax \log (\tfrac{M}{\delta}) \\
&\qquad+ CC'\log\left(\tfrac{MN}{\delta}\right)\min\left\{\log \tfrac{MN}{\delta},\log \left(\tfrac{MH}{\epsilon}\right) \right\}
\end{align*}
Note that the special case $g(u) \lesssim \frac{C}{\log (M u/\delta)}{u}$ can be obtained by setting $C' = 1$ in the above inequality. 
\end{enumerate}
Lastly, the above computations hold if $f(u/4)$ is replaced by $f(u/c)$ for any universal constant $c > 0$. Moreover, the above computations hold if $f(u) \lesssim \min\left\{\fmax, g(u)\right\}$ by taking $\epsilon = 0$ and setting $\frac{1}{\epsilon} = \infty$.
\end{restatable}
\begin{rem}[Integration without anytime bounds]\label{rem:anytime_log} If instead we consider functions $g(u)$ satisfying the looser bounds (a) $g(u) \lesssim \sqrt{\frac{C \log(  M T/\delta )}{u}}$ and (b) $g(u) \lesssim C\left(\sqrt{\frac{C' \log(M T/\delta)}{u}} +\frac{C' \log(M T/\delta)}{u} \right)^2$ for $T \ge N$, then we can recover the bounds 
\begin{align*}
\int_{H}^{N}f(u/4)du \lesssim \begin{cases}\min\left\{\sqrt{CN \, \log\tfrac{M T}{\delta}},\, \frac{C}{\epsilon}\log\tfrac{MT}{\delta} \right\} & \text{ case (a)}\\
\left(1+C'\right)f_{\max} \log \tfrac{MT}{\delta} + CC'\log\left(\tfrac{MT}{\delta}\right)\min\{\log\tfrac{MT}{\delta},\log \frac{\log(MT/\delta)}{\epsilon}\} & \text{ case (b)}
\end{cases}
\end{align*}
These sorts of bounds arise when the confidence intervals are derived via union bounds over all time $T$, rather than via anytime estimates. In particular, we see that using a naive union bounded over all time $T$ incurs a dependence on $\log T \cdot (\log \log T)$, and thus does not imply a strictly $\BigOh{\log T}$ regret.
\end{rem}
Let us conclude by applying the above lemma to the terms at hand. First, applying the Part (a) with $f = \flead_{x,a}$, $g = \glead_{x,a}$, $C = \Varxa$, and $H \ge \epsilon =\epslead_{x,a} := \min_h\frac{\gapcliph\xa}{4} \gtrsim \frac{\gap\xa}{(1 \vee \bsfa H)}$ for $(x,a) \in \Zsub$, and $H \ge \epsilon{x,a} \gtrsim \frac{\gapmin}{H}$ for $\xa \in \Zopt,$ we have that
\begin{align*}
&\sum_{x,a}\I(\nbarK \xa \ge H)\int_{H}^{\nbarK\xa}\flead_{x,a}(u/4)du \\
&\quad\lesssim \sum_{x,a} \frac{\Varxa}{\min_h\gapcliph\xa} \log( \tfrac{MH}{\delta \min_h\gapcliph\xa})\\
&\quad\lesssim \sum_{(x,a) \in \Zsub} \frac{\Varxa (1 \vee \bsfa H)}{\gap\xa} \log\tfrac{M H}{\delta \gap\xa} + \sum_{(x,a) \in \Zopt} \frac{\Varxa H}{\gapmin\xa} \log \tfrac{M H}{\delta \gapmin}\\
&\quad\le \sum_{(x,a) \in \Zsub} \frac{\Varxa (1 \vee \bsfa H)}{\gap\xa} \log\tfrac{M H}{\delta \gap\xa} +  |\Zopt|\frac{\Varmax H}{\gapmin\xa} \log \tfrac{M H}{\delta \gapmin}.
\end{align*}
Similarly, applying the Part (b) with $f = \ffut$, $g = \gfut$, $C' =S$ and $C = H^3$, and $\epsilon =\epsfut := \frac{\gapclip_\min}{8SAH} \gtrsim \frac{\gapmin}{H^2}$ (and also satisfying $\epsfut \le H$), we can bound
\begin{align*}
&\lesssim S^2 A H^3\log\left(\tfrac{MT}{\delta}\right)\min\left\{\log \tfrac{MT}{\delta},\log \left(\tfrac{MH}{\gapmin}\right) \right\},
\end{align*}
Plugging the above two displays into \eqref{eq:function_reg_bound} concludes the proof of Corollary~\ref{cor:precise_log_bound}. 

\subsection{Proof of Theorem~\ref{thm:main_regret_bound_ht}\label{sec:proof_of_interpolated}}
We conclude the section by proving the regret bound of Theorem~\ref{thm:main_regret_bound_ht}, which interpolates between the $\sqrt{T}$ and $\log T$ regimes. Let us recall the subset $\Zsub(\epsilon) := \{\xa : \gap(x,a) < \epsilon\}$, as well as $\HeffT  := \min\left\{\Varmax, \frac{\Gterm^2}{H}\log  T \right\}\,.
$. Retracing the proof of Corollary~\ref{cor:precise_log_bound}, it suffices to establish only two points:
\begin{align*}
&\sum_{(x,a) \in \Zsub(\epsilon)} \sum_{k=1}^K\sum_{h=1}^H\wkh(x,a)\threshop{\frac{\gapcliph\xa}{4}}{\errconst\boundlead\kh\xa} \\
&\qquad\lesssim \sqrt{ \HeffT\,|\Zsub(\epsilon)|  \, T \log \tfrac{M T}{\delta}} + SAH^2\log(M\delta/T)\\
&\sum_{(x,a) \in \Zopt} \sum_{k=1}^K\sum_{h=1}^H\wkh(x,a)\threshop{\frac{\gapcliph\xa}{4}}{\errconst\boundlead\kh\xa} \\
&\qquad\lesssim \sqrt{ \HeffT\,|\Zopt|  \, T \log \tfrac{M T}{\delta}} + SAH^2\log(M\delta/T)
\end{align*}
For both of these inequalities, we will discard the clipping, and thus the two bounds will be syntatically the same. Hence, let us simply prove the following bound:
\begin{align*}
\sum_{k=1}^K\sum_{h=1}^H\sum_{(x,a) \in \Zopt} \boundlead\kh\xa \lesssim \sqrt{ \HeffT\,|\Zopt|  \, T \log \tfrac{M T}{\delta}}.
\end{align*}
Since $\HeffT  := \min\left\{\Varmax, \frac{\Gterm^2}{H}\log  T \right\}$, it suffices to prove the above bound first with $\HeffT$ replaced by $\Varmax$, and then replaced by $\frac{\Gterm^2}{H}\log  T $. 

\paragraph{Bound with $\Varmax$: }
To obtain a bound involving $\Varmax$, we use the fact that $\boundlead\kh\xa \lesssim \glead(\nk\xa)$, for the function $\glead(u) = \sqrt{\Varmax\log(Mu/\delta)/u}$. Hence, following the integration arguments in the proof of Corollary~\ref{cor:precise_log_bound}, clipped at $\epsilon = 0$, we can bound
\begin{align*}
\sum_{k=1}^K\sum_{h=1}^H \wkh\xa\boundlead\kh\xa &\le H \log(M/\delta) + \sqrt{\Varmax\nbarK\xa \log(\nbarK\xa M/\delta)} \\
&\le H^2 \log(M/\delta) + \sqrt{\Varmax\nbarK\xa \log(T M/\delta)}.
\end{align*}
Hence, by Cauchy Schwartz, and the bound $\sum_{\xa\in \Zopt}\nbarK\xa \le \sum_{\xa}\nbarK\xa = T$, 
\begin{align*}
\sum_{\xa \in \Zopt}\sum_{k=1}^K\sum_{h=1}^H \wkh\xa\boundlead\kh\xa &\lesssim SAH\log(M/\delta) + \sum_{x,a \in \Zopt} \sqrt{\Varmax\nbarK\xa \log(T M/\delta)}\\
&\le SAH\log(M/\delta) +  \sqrt{\Varmax |\Zopt|\sum_{x,a \in \Zopt}\nbarK\xa \log(T M/\delta)}\\
&\le SAH\log(M/\delta) +  \sqrt{\Varmax |\Zopt|T \log(T M/\delta)},
\end{align*}
as needed.

\paragraph{Bound with $\Heff_T$: }
This bound requires a little more subtely. Define the function $f(u) = (1/\max\{u,1\})$. Then, using the definition of $\boundlead\kh\xa$ from Proposition~\ref{prop:surplus}, we have
\begin{align*}
\sum_{(x,a) \in \Zopt}\sum_{h=1}^H \boundlead\kh\xa &\lesssim \sum_{k=1}^K\sum_{h=1}^H\sum_{(x,a) \in \Zopt} \wkh\xa(H\wedge \sqrt{\log (M \nk \xa/\delta) f(\nk\xa)\Varxah})\\
&\le \sum_{k=1}^K\sum_{h=1}^H\sum_{(x,a) \in \Zopt} \wkh\xa(H\wedge\sqrt{\log (MT/\delta) f(\nk\xa)\Varxah}).
\end{align*}
Applying the recipe we used for Corollary~\ref{cor:precise_log_bound} will not quite carry over in this setting. Instead, we apply an argument based on Cauchy-Schwartz, defered to Section~\ref{sec:causch_integral}:
\begin{lem}[Cauchy-Schwartz Integration Lemma for $\Gterm$-bounds]\label{lem:Cauchy_Schwartz_Integral} Let $\{V_{x,a,k,h}\}$ be a sequence of numbers, and let $f(u)$ be a nonnegative, non-decreasing function, $\fmax > 0$, $\Lfactor$ a problem dependent parameter, and let $\calZ_0 \subset \states \times \actions$. Then, on $\eventsample$, 
\begin{multline*}
\sum_{(x,a) \in \calZ_0}\sum_{k=1}^K\sum_{h=1}^H\wkh\xa \fmax \wedge \sqrt{\Lfactor f(\nk\xa)V_{x,a,k,h}} \\
\le |\calZ_0|\Hsample \fmax + \sqrt{\Lfactor\sum_{k=1}^K \Exppik\left[\sum_{h=1}^H V_{x,a,k,h}\right] \cdot |\calZ| (Hf(H) + \int_{1}^T f(u)du)}.
\end{multline*}
\end{lem}
We apply the above lemma with $V_{x,a,k,h} = \Var^{\pik}_{x,a,h}$, $\fmax = H$ and $f(u) = 1/\max\{u,1\}$, and $\Lfactor = \log(MT/\delta)$. It is easy to see that the term $|\calZ_0|\Hsample \fmax \lesssim SAH^2\log(M/\delta)$ will already absorbed into terms already present in the final bound. On the other hand, by a now-standard law of total variance argument, 
\begin{align*}
\sum_{k=1}^K \Exp^{\pik}\left[\sum_{h=1}^H\Var^{\pik}_{x_h,a_h,h}\right] &\le K \max_{\pi}\Exp^{\pi}\left[\sum_{h=1}^H\Var^{\pi}_{x_h,a_h,h}\right] \le T \Gterm^2/H, \numberthis\label{eq:G_bound}
\end{align*}
where the last inequality is from the proof of \iftoggle{nips}{\cite{zanette2019tighter}, Proposition 6}{\citet[Proposition 6]{zanette2019tighter}}. On the other hand, we can bound$ (Hf(H) + \int_{1}^T f(u)du) \le 1 + \log T$. This finally yields
\begin{align*}
\sum_{k=1}^K \Exp^{\pik}\left[\sum_{h=1}^H\Var^{\pik}_{x_h,a_h,h}\right] \lesssim SAH^2\log(M/\delta) + \sqrt{\log(MT/\delta) T \cdot \Gterm^2/H \log (T)},
\end{align*}
as needed.


\section{Proof of Technical Lemmas\label{sec:general_technical}}

\subsection{Proof of clipping with future bounds, Lemma~\ref{lem:clipped_regret_with_future}\label{sec:clipped_regret_with_future}}

	Since strong optimistm holds on $\goodconcentration$, Theorem~\ref{thm:surplus_clipping} yields
	\begin{align*}
	\vst_0 - \vpik_0 &\le 2e \sum_{h=1}^H\sum_{x,a} \wkh(x,a)\threshop{\gapcliph\xa}{\Ekh\xa}.
	\end{align*} 
	Applying Lemma~\ref{lem:clipping_lemma} with $m = 2$, $a_1 = \errconst\boundlead\kh\xa$, and 
	\begin{align*}
	a_2 = \Exppik\left[\sum_{t=h}^H \errconst\boundfut_k(x,a) \mid (x_h,a_h) = \xa\right],
	\end{align*}
	\begin{align*}
	\vst_0 - \vpik_0 &\le 4e\sum_{h=1}^H\sum_{x,a} \wkh(x,a)\threshop{\frac{\gapcliph\xa}{4}}{\errconst\boundlead\kh\xa}.\\
	 &\qquad 4e\, \sum_{h=1}^H\sum_{x,a} \wkh(x,a)\threshop{\frac{\gapcliph\xa}{4}}{\Exppik\left[\sum_{t=h}^H \errconst\boundfut_k(x_t,a_t) \mid (x_h,a_h) = \xa\right]}.
	\end{align*} 
	The term on the right hand side of the first line of the above display is exactly as needed. Let us turn our attention to the term on the second line. We have that
	\begin{align*}
	\Exppik\left[\sum_{t=h}^H \errconst\boundfut_k(x_t,a_t) \mid (x_h,a_h) = \xa\right] = \sum_{x',a'}\sum_{t=h}^H \errconst\boundfut_k(x',a')\Pr[(x_t,a_t) = (x',a') \mid (x_h,a_h) = \xa].
	\end{align*}
	Hence, applying Lemma~\ref{lem:clipping_lemma} with the terms $a_i$-terms corresponding to $\boundfut_k(x',a')\Pr[(x_t,a_t) = (x',a') \mid (x_h,a_h) = \xa]$ and the number of such terms $m$ bounded by $SAH$, we have
	\begin{align*}
	&\threshop{\frac{\gapcliph\xa}{4}}{\Exppik\left[\sum_{t=h}^H \boundfut_k(x_t,a_t) \mid (x_h,a_h) = \xa\right]} \\
	&\qquad \le 2\sum_{x',a'}\sum_{t=h}^H \clip{\frac{\gapcliph\xa}{8SAH}}{\errconst\boundfut_k(x',a')\Pr[(x_t,a_t) = (x',a') \mid (x_h,a_h) = \xa]}.
	\end{align*}
	Since $\clip{\epsilon}{\alpha x} \le \alpha \clip{\epsilon}{x}$ for $\alpha \le 1$, and since the probabilities $\Pr[(x_t,a_t) = (x',a') \mid (x_h,a_h) = \xa]$ are bounded by $1$, we can bound the above by
	\begin{align*}
	2\sum_{x',a'}\sum_{t=h}^H \Pr[(x_t,a_t) = (x',a') \mid (x_h,a_h) = \xa]\clip{\frac{\gapcliph\xa}{8SAH}}{\errconst\boundfut_k(x',a')}.
	\end{align*}
	Hence,
	\begin{align*}
	&4e\ \sum_{h=1}^H\sum_{x,a} \wkh(x,a)\threshop{\frac{\gapcliph\xa}{4}}{\Exppik\left[\sum_{t=h}^H \errconst\boundfut_k(x_t,a_t) \mid (x_h,a_h) = \xa\right]}\\
	&\qquad \le 8e \sum_{h=1}^H\sum_{x,a}\sum_{x',a'}\wkh(x,a)\sum_{t=h}^H \Pr[(x_t,a_t) = (x',a') \mid (x_h,a_h) = \xa]\clip{\frac{\gapcliph\xa}{8SAH}}{\errconst\boundfut_k(x',a')}\\
	&\qquad \le 8e \sum_{h=1}^H\sum_{x,a}\sum_{x',a'}\wkh(x,a)\sum_{t=h}^H \Pr[(x_t,a_t) = (x',a') \mid (x_h,a_h) = \xa]\clip{\frac{\gapclipmin}{8SAH}}{\errconst\boundfut_k(x',a')}\\
	&\qquad= 8e \sum_{h=1}^H\sum_{t=h}^H\sum_{x',a'}\wkh(x',a')\clip{\frac{\gapclipmin}{8SAH}}{\errconst\boundfut_k(x',a')}\\
	&\qquad\le 8 He \sum_{h=1}^H\sum_{x,a}\wkh(x,a)\clip{\frac{\gapclipmin}{8SAH}}{\errconst\boundfut_k(x',a')}.
	\end{align*}
	Altogether,
	\begin{align*}
	\vst_0 - \vpik_0 &\le 4e\ \sum_{h=1}^H\sum_{x,a} \wkh(x,a)\threshop{\frac{\gapcliph\xa}{4}}{\errconst\boundlead\kh\xa}.\\
	 &+\qquad 8 He \sum_{h=1}^H\sum_{x,a}\wkh(x,a)\clip{\frac{\gapclipmin}{8SAH}}{\errconst\boundfut_k(x',a')}.
	\end{align*} 
	Summing over $k = 1,\dots, K$ proves the inequality.

\subsection{Proof of sampling lemma (Lemma~\ref{lem:sample}) \label{sec:lem_sample_proof} }
	Recall $(\eventsample)' := \{\forall k,s,a : \nk\xa  \ge \frac{1}{2}\nbar_{k-1}\xa - H\log \frac{2HSA}{\delta}\}$;  Lemma 6 in~\cite{dann2018policy} in shows that this event occurs with probability at least $1-\delta/2$. We show that $ (\eventsample)' \subseteq \eventsample(\Hsample)$, for $\Hsample = 4H\log \frac{2HSA}{\delta} \lesssim H \log \frac{M}{\delta}$. 

	Noting that $\nbark \le \nbar_{k-1}+H$,  $(\eventsample)'$ implies that $\nk \ge \frac{1}{2}\nbar_{k}\xa - H\log \frac{2HSA}{\delta} -H = \frac{1}{2}\nbar_{k}\xa - H\log \frac{2eHSA}{\delta}$. Hence, for any $k \ge \sampletime\xa$, we have  $\nbark \ge 4H\log \frac{2HSA}{e\delta}$ and thus $\nk\xa \ge  \frac{\nbark}{4} + \frac{\nbark}{4} -  H\log \frac{2eHSA}{\delta} \ge \frac{\nbark}{4}$. Bounding $\log \frac{2HSA}{e\delta} \le \Lfactortil(1)$ concludes the proof.

\subsection{Proof of integral conversion, Lemma~\ref{lem:integral_conversion}}

Recall that $\sampletime\xa$ denote $\inf\{k:\nbark\xa \ge \Hsample\}$. Then,
\begin{align*}
\sum_{k=1}^K \sum_{x,a} \wk\xa f_{x,a}(\nk\xa) &= \sum_{x,a}\sum_{k = 1}^{\sampletime\xa - 1}\wk\xa f_{x,a}(\nk\xa) + \sum_{x,a}\sum_{k = \sampletime\xa}f_{x,a}(\nk)\\
&\le \sum_{x,a}\sum_{k = 1}^{\sampletime\xa - 1}\wk\xa f_{\max} + \sum_{x,a}\sum_{k = \sampletime\xa}f_{x,a}(\nk)\\
&\le SA \Hsample f_{\max} + \sum_{x,a}\sum_{k = \sampletime\xa}f_{x,a}(\nk\xa)\\
&\le SA \Hsample f_{\max} + \sum_{x,a}\sum_{k = \sampletime\xa}f_{x,a}(\nbark\xa/4),
\end{align*}
since $\sum_{k=1}^{\sampletime\xa-1}\wk\xa = \nbar_{\sampletime\xa - 1}\xa \le \Hsample$madn $f(\cdot) \le \fmax$.  We now appeal to the following integration lemma, which we prove momentarily.

\begin{lem}[Integration over $\wk\xa$]\label{lem:integration} Let $f: [H,\infty) \to \R_{> 0}$ be a non-increasing function. Then,
\begin{align}\label{eq:integral_expression}
&\sum_{k = \sampletime\xa}^{K} \wk\xa f(\nbark\xa) \le  H f(H) + \int_{H}^{\nbarK\xa} f(u)du.
\end{align}
\end{lem}
To conclude the proof of Lemma~\ref{lem:integral_conversion}, we apply the above for each $\xa$ with the functions $f(u) \leftarrow f_{x,a}(u/4)$, and note $H f_{x,a}(H/4) \le H \fmax \le \Hsample \fmax$. \qed

	\begin{proof}[Proof of Lemma~\ref{lem:integration}] The proof generalizes Lemma E.5 in~\cite{dann2017unifying}. For ease of notation, define $k_0 = \tau\xa$. We can define the step function $g: [ k_0, K] \to \R$ via  $g(t) = \sum_{k=k_0}^{K - 1} \weight_{k+1}\xa \I(t \in [k,k+1)]$. Then, letting $G(t) :=  \nbar_{k_0}\xa + \int_{0}^t g(u)du$, we see that  $G'(t) = g(t)$ almost everywhere, $G$ is non-decreasing, and $G(k) = \nbark\xa$ for all $k \in [k_0,K]$. We can therefore express
	\begin{align*}
	\sum_{k > \tau\xa}^K \wk\xa f(\nbark\xa) &=  \sum_{k = k_0+1}^{K} \wk\xa f(\nbark\xa ) ~= \sum_{k = k_0+1}^{K}\left(\int_{k-1}^{k} g(t) dt\right) f(G(k)) \cdot  \\
	&\overset{(i)}{\le} \sum_{k = k_0+1}^{K}\left(\int_{k-1}^{k} g(t) f(G(t)) dt\right)   ~= \int_{k = k_0}^{K}g(t) f(G(t))dt  \\
	&\overset{(ii)}{=} \int_{G(k_0)}^{G(K)}  f(u)du ~\overset{(iii)}{=} \int_{\nbar_{k_0}\xa}^{\nbar_{K}\xa}f(u)du,
	\end{align*}
	where $(i)$ uses the fact that $f\circ G$ is non-increasing,  $(ii)$ is the Fundamental Theorem of Calculus, with $G'(t) = g(t)$, and $(iii)$ is $G(k) = \nbark\xa$ for $k \in [k_0,K]$. Hence, we have the bound
	\begin{align*}
	\sum_{k \ge k_0}^K \wk\xa f(\nbark\xa) &\le \weight_{k_0}\xa f(\nbar\knot\xa) + \int_{\nbar_{k_0}\xa}^{\nbar_{K}\xa}f(u)du\\
	&\overset{(i)}{\le} H f(\nbar\knot\xa) + \int_{\nbar_{k_0}\xa}^{\nbarK\xa}f(u)du  \overset{(ii)}{\le} H f(H) + \int_{H}^{\nbarK\xa}f(u)du,
	\end{align*}
	where $(i)$ uses $\weight_{k_0} \le H$, and that $f(u) \ge 0$, and $(ii)$ uses the fact that $f$ is nonincreasing, and $\nbar\knot\xa \ge\Hsample \ge H$.
	\end{proof}

\subsection{Proof of interal conversion for $\Gterm$-bounds, Lemma~\ref{lem:Cauchy_Schwartz_Integral}\label{sec:causch_integral}}
Let $\sampletime\xa = \sampletime_{\Hsample}\xa$. Then, as in the proof of Lemma~\ref{lem:integral_conversion},
\begin{align*}
&\sum_{(x,a) \in \calZ}\sum_{k=1}^K\sum_{h=1}^H\wkh\xa (\fmax \wedge \sqrt{\Lfactortil f(\nk\xa)V_{x,a,k,h}}) \\
&\qquad\lesssim \sum_{(x,a) \in \calZ}\sum_{k=\sampletime\xa}^K\sum_{h=1}^H\wkh\xa \sqrt{\Lfactortil f(\nk\xa)V_{x,a,k,h}} + |\Zopt|\Hsample\fmax.
\end{align*}
By Cauchy-Schwartz
\begin{align*}
&\sum_{(x,a) \in \calZ}\sum_{k=\sampletime\xa}^K\sum_{h=1}^H\wkh\xa  &\lesssim \\
&\quad\sqrt{\sum_{(x,a) \in \calZ}\sum_{k=\sampletime\xa}^K\sum_{h=1}^H\wkh\xa V_{x,a,k,h} }
&\qquad \times\Lfactor^{1/2}\sqrt{\sum_{(x,a) \in \calZ}\sum_{k=\sampletime\xa}^K\sum_{h=1}^H\wkh\xa f(\nk\xa)}.
\end{align*}
The first term in the above product can be bounded as
\begin{align*}
\sum_{(x,a) \in \calZ}\sum_{k=\sampletime\xa}^K\sum_{h=1}^H\wkh\xa V_{x,a,k,h} \le \sum_{k=1}^K\sum_{h=1}^H\sum_{x,a}\wkh\xa V_{x,a,k,h} = \sum_{k=1}^K \Exppik[\sum_{h=1}^H V_{x,a,k,h}].
\end{align*}
Using Lemma~\ref{lem:integration}, the second term can be bounded as 
\begin{align*}
\sqrt{\sum_{(x,a) \in \calZ}\sum_{k=\sampletime\xa}^K\sum_{h=1}^H\wkh\xa f(\nk\xa)} \le |\calZ| (Hf(H) + \int_{1}^T f(u)du).
\end{align*}

\subsection{General Integral computations  (Lemma~\ref{lem:integral_computations})\label{sec:integral_comp}}
For convenience, let us restate the lemma we are about to prove.
\integralcomputations*

\begin{proof} \newcommand{\nend}{n_{\mathrm{end}}}

By inflating $C$ by a problem-independent constant if necessary, we may assume without loss of generality that $g(u) = \sqrt{C\log(M u/\delta)/u}$ in part (a) and $g(u) = C(\sqrt{C'\log(M u/\delta)/u} + \sqrt{C'\log(M u/\delta)/u})^2$, with equality rather than approximate inequality $\lesssim $.

Next, define
	\begin{align*}
	\nend := \begin{cases}\max\{u: g(u/4)  \ge \epsilon\} \epsilon > 0\\
	N & \epsilon = 0.
	\end{cases}
	\end{align*}
	Throughout, we shall assume the case $\epsilon > 0$, as the $\epsilon = 0$ can be derived by just taking $\nend = N$. Note then that $f(u/4)  = \clip{\epsilon}{g(u/4)} = 0$ for all $u > \nend$. Hence, it suffices to upper bound
	\begin{align*}
	\I(\nend \ge H)\int_{H}^{N \wedge \nend}\fmax \wedge g(u/4)du.
	\end{align*}
	Lastly, let us define $\Lfactortil(u) := \log(Mu/\delta)$ for $u \ge H$. We shall rquire the following inversion lemma, which is standard in the multi-arm bandits literature.
	\begin{lem}[Inversion Lemma]\label{lem:lfactor_inversion} There exists a universal constant $c > 0$ such that for all $b \ge 0$, $\Lfactortil(u)/u \le b$ as long as $u \ge   \Lfactortil(1+b^{-1})/ c b$. Moreover, for $u \lesssim \Lfactortil( b^{-1})/ c b $, it holds that  $	\Lfactortil(u) \lesssim  \Lfactortil(1 + b^{-1})$. 
	\end{lem}
	\begin{proof} Let $u = \Lfactortil( 1/b)/cb$ for a constant $c$ to be chosen shortly. Then,
	\begin{align*}
	\Lfactortil(u)/u  &= cb \frac{\Lfactortil(\frac{1}{cb} \Lfactortil( b^{-1})}{\Lfactortil( 1 + b^{-1})} = cb \frac{\log \frac{1}{c} + \log(M/b\delta ) + \log( \log \frac{M}{b\delta})}{\log(M/b\delta)}\\
	&\le \frac{cb \log \frac{1}{c}}{\log 2} + 2 cb,
	\end{align*}
	where we use $\log \log(x) \le x$ and $\Lfactortil( 1 + b^{-1}) \ge \Lfactortil(1) \ge \log 2$. It is easy to see that this quantity is less than $b$ for a constant $c$ sufficiently small that does not depend on $M,\delta,b$. The second statement follows from an analogous computation.  
	\end{proof}

	\paragraph{Proof of Part (a):} Suppose $g(u) = \sqrt{\frac{C \Lfactortil(u)}{u}}$. It is straightforward to bound
	\begin{align*}
	\I(\nend \ge H)\int_{H}^{N \wedge \nend}\fmax \wedge  g(u/4)du &\lesssim \I(\nend \ge H)\sqrt{C\Lfactortil(N \wedge \nend)}\int_{1}^{N \wedge \nend}\frac{du}{\sqrt{u}}\\
	&\lesssim \sqrt{C\Lfactortil(N \wedge \nend)}\cdot \sqrt{(N \wedge \nend)}\\
	&\lesssim \min\left\{\sqrt{CN\Lfactortil(N)},\, \sqrt{C\nend\Lfactortil(\nend) }\right\}\\
	&= \min\left\{\sqrt{CN\log \frac{MN}{\delta}},\, \sqrt{C\nend\Lfactortil(\nend) }\right\}.
	\end{align*}
	To conclude, let us find $\nend\xa$. By our inversion Lemma~\ref{lem:lfactor_inversion}, we can see that
	\begin{align*}
	&\nend \lesssim \frac{C}{\epsilon^2}\Lfactortil(1 + \frac{C}{\epsilon^2}) \lesssim \frac{C}{\epsilon^2}\Lfactortil(1 + \frac{C}{\epsilon})\\
	&\Lfactortil(\nend) \lesssim \Lfactortil(1 + \frac{C}{\epsilon}).
	\end{align*}
	Therefore,
	\begin{align*}
	 \sqrt{C\nend\Lfactortil(\nend) } \lesssim \frac{C}{\epsilon}\Lfactortil\left(1 + \frac{C}{\epsilon}\right).
	\end{align*}
	Moreover, for if $\log C \lesssim \log M$ and $\epsilon \le H$, we can bound $\Lfactortil\left(1 + \frac{C}{\epsilon}\right) \lesssim \log \frac{MH}{\epsilon}$. Hence, we haveshow
	\begin{align*}
	\int_{H}^{N}f(u/4)du \lesssim \min\left\{\sqrt{CN \, \log\tfrac{M N}{\delta}},\, \frac{C}{\epsilon}\log\left( \frac{M}{\delta} \cdot \frac{H}{\epsilon} \right)\right\} .
	\end{align*}
	To conclude, it remains to show that we can replace $\frac{H}{\epsilon}$ with $N$. For this, we use a simpler argument:
	\begin{align*}
	\int_{H}^{N}f(u/4)du &\lesssim \int_{H}^{N}\clip{\epsilon}{\frac{\sqrt{C\Lfactortil(u)}}{u}} \\
	&\le \int_{H}^{N}\clip{\epsilon}{\frac{\sqrt{C\Lfactortil(N)}}{u}} \\
	&= \sqrt{\Lfactortil(T)}\int_{H}^{N}\clip{\epsilon'}{\frac{1}{\sqrt{u}}} , \quad \text{where } \epsilon' = \sqrt{\Lfactortil(T)}.
	\end{align*}
	Using similar arguments to above, we can bound $\int_{H}^{N}\clip{\epsilon'}{\frac{1}{\sqrt{u}}}  \lesssim \frac{1}{\epsilon'}$, yielding the bound $\int_{H}^{N}f(u/4)du \lesssim \frac{\Lfactortil^{1/2}(T)}{\epsilon'} = \frac{\Lfactortil(T)}{\epsilon}$.


	\paragraph{Proof of Part (b): A first step}  This proof will require slightly more care than part (b). We shall first require the following lemma:
	\begin{claim}\label{lem:integral_special_case}In the setting of Lemma~\ref{lem:integral_computations}, if  $g(u) = \frac{C\log \tfrac{Mu}{\delta}}{u} = \frac{C \Lfactortil(u)}{u}$, then
	\begin{align*}
	\int_{H}^{N}f(u/4)du \lesssim \fmax \log M  + C\log\left(M T/\delta\right)\min\left\{\log(\frac{MT}{\delta}),\log \left(\frac{MH}{\epsilon}\right) \right\}
	\end{align*}
	\end{claim}
	\begin{proof}[Proof of Claim~\ref{lem:integral_special_case}]
	   	Define $n_0 = 2 + \log(M/\delta)$. Then, we have
		\begin{align*}
		\I(\nend \ge H)\int_{H}^{N \wedge \nend}\fmax \wedge g(u/4)du &\le \fmax n_0 + \I(N \wedge \nend \ge n_0) \cdot \int_{n_0}^{N \wedge \nend}g(u/4)du.\\
		&\le \fmax n_0 +  \int_{n_0}^{n_0 + N \wedge \nend}g(u/4)du\\
		&\lesssim \fmax \log(M/\delta) +  \int_{n_0}^{n_0 + N \wedge \nend}g(u/4)du.
		\end{align*}

		Now take $g(u) = C\Lfactortil(u)/u$. Since $\Lfactortil(u) \lesssim \log(M/\delta) + \log (u)$ for $u \ge n_0 \ge 2$, it it is straightforward to bound
		\begin{align*}
		\int_{n_0}^{n_0 + N \wedge \nend }g(u/4)du &\lesssim C\log(M/\delta)\log(1 + \frac{N \wedge \nend }{n_0}) + C\log^2 (1 + \frac{N \wedge \nend}{n_0})\\
		&\lesssim C\log(M T/\delta)\log(1 + \frac{N \wedge \nend }{n_0}),
		\end{align*}
		where in the final inequality, we use $N \le T$, $M/\delta \ge 2$, and $n_0 \ge 1$. By the same token, we can crudely bound the above by $C\lesssim \log^2(M T/\delta)$. 

		Let us now develop a more refined bound by taking advantage of $\nend$. By our inversion lemma, we have
		\begin{align*}
		\nend \lesssim \frac{C}{\epsilon}\Lfactortil\left(1 + \frac{C}{\epsilon}\right) = \frac{C}{\epsilon}\left(\log 1 + \frac{C}{\epsilon}) + \log \frac{M}{\delta}\right).
		\end{align*}
		Since $n_0 = 1 + \log \frac{M}{\delta}$,
		\begin{align*}
		\frac{\nend}{n_0} \lesssim \frac{C}{\epsilon}\log (1+\frac{C}{\epsilon}) + \frac{C}{\epsilon}\frac{\log \frac{M}{\delta}}{\log \frac{M}{\delta}} \lesssim \frac{C}{\epsilon}\log \left(1+\frac{C}{\epsilon}\right).
		\end{align*}
		Hence, with some algebra we can bound
		\begin{align*}
		\log(1 + \frac{\nend}{n_0}) \lesssim \log(1 + \frac{C}{\epsilon}) 
		\end{align*}
		This leads to the more refined bound $\int_{n_0}^{n_0 + N \wedge \nend }g(u/4)du \lesssim C\log(1 + \frac{C}{\epsilon})  \log (\Mtil T)$. Again, since $\log C \lesssim \log M$ and $\epsilon \le H$, we bound again bound $\log(1 + \frac{C}{\epsilon}) \lesssim \log \frac{MH}{\epsilon}$. 
	\end{proof}

	\paragraph{Concluding the proof of Part (b)} Define
	\begin{align*}
	n_0 := \{\inf u: \sqrt{\frac{C' \Lfactortil(u)}{u}} \le 1\}.
	\end{align*} 
	Then, we have
	\begin{align*}
	\int_H^{N}f(u/4)du \le \fmax n_0 + \int_{n_0}^{N}\fmax \wedge \clip{\epsilon}{g(u/4)}du.
	\end{align*}
	Note that for $u \ge n_0$, $g(u/4) \lesssim h(u/4)$, where $h(u) \le \frac{CC'}{u}\Lfactortil(u)$. Hence, applying the bound from Lemma~\ref{lem:integral_special_case} with $C \leftarrow CC'$, we have
	\begin{align*}
	\int_H^{N}f(u/4)du \lesssim \fmax \log \Mtil  + CC'\log(\Mtil T)\min\left\{\log(\frac{MT}{\delta} ),\log (\frac{MH}{\epsilon}) \right\}
	\end{align*}
	On the otherhand, by our inversion lemma and using $C' \le \Mtil^{\BigOh{1}}$, we can bound
	\begin{align*}
	n_0 \le C'\Lfactortil(C') \lesssim C'\log(\frac{M}{\delta}).
	\end{align*}
	Combining these two pieces yields the bound.
	\end{proof}

\newpage


\section{Proof of `clipping' bound: Proposition~\ref {prop:surplus_clipping_simple} / Theorem~\ref{thm:surplus_clipping}\label{sec:clipping_proof}}

In this section, we prove Theorem~\ref{thm:surplus_clipping} (of which Proposition~\ref{prop:surplus_clipping_simple} \iftoggle{nips}{in the body}{} is a direct consequence), which allows us to clip the surpluses when they are below a certain value. The center of our analysis is the following lemma, which tells us that if $\gaph\xa > 0$ for a pair $(x,a,h)$, then either the surplus $\Ekh(x,a)$ is large, or expected difference in value functions at the next stage, $p(x,a)^\top (\vupkhpl - \vpikhpl)$, is large:
\begin{lem}[Fundamental Gap Bound] \label{lem:fundamental_gap_bound}
 Then suppose that $\Alg$ is strongly optimistic, and consider a pair $(x,a,h)$ with $\gaph\xa > 0$ which is is $\bsfa$-transition optimal. Then
\begin{align*}
 \gaph\xa \le \Ekh(x,a) + \bsfa \cdot p(x,a)^\top (\vupkhpl - \vpikhpl).
\end{align*}
If $\Alg$ is possibly not strongly optimistic, then the above holds still holds  $\bsfa = 1$.
\end{lem}
Lemma~\ref{lem:fundamental_gap_bound} is established in Section~\ref{sec:lem:fundamental_gap_bound}. Notice that as $\bsfa$ gets close to zero, the above bound implies that when $\Ekh(x,a)$ is much smaller than the $\gaph\xa$, the difference in value functions at the next stage, $p(x,a)^\top (\vupkhpl - \vpikhpl)$, must become even larger to compensate. The extreme case is $\bsfa = 0$, e.g. in contextual bandits, where the gap always lower bounds the surplus. 

Continuing with the proof of Theorem~\ref{thm:surplus_clipping}, we begin with the ``half-clipping'' which clips the surpluses at at most $\gapmin$:
\begin{defn}[Half Clipped Value Function]\label{defn:halfclip}  We define the half-clipped surplus $\Ehclip\kh(x,a) := \threshop{\epsclip}{\Ekh(x,a)}$, where $\epsclip := \gapmin/(2H)$. We  set $\vhclippik_{k,H+1}(x) = 0$ for all $x \in \states$, and recursively define 
\begin{align*}
\qhclippik_h\xa = r(x,a) + \Ehclip\kh\xa + p(x,a)^\top \vhclippik\khpl, \quad \vhclippik\kh(x) := \qhclippik_h(x,\pikh(x)), 
\end{align*}
denote the value and Q-functions of under $\pik$ associated with MDP whose transitions are transitions $p(\cdot,\cdot)$ and non-stationary rewards $r(x,a) + \Ehclip\kh(x,a)$ at stage $h$. 
\end{defn}
After the half-clipping has been introduced, it is no longer the case that $\pik$ is optimal for this half clipped MDP. As a result, it is not certain that the half-clipped Q-function for $\pik$ is \emph{optimistic} in the sense that $\qhclippik\kh(x,a) \ge \qsth(x,a)$. We shall instead show that if $\vhclip^{\pikh}$ is approximately optimistic, in the sense that its excess relative to $\vpik$, $\vhclip^{\pikh}_0 - \vpik_0$ is at least a constant factor of the regret $\vst_0 - \vpik_0$:

\begin{lem}[Lower Bound on Half-Clipped Surplus]\label{lem:half_clip} For $\epsclip = \gapmin/2H$, it holds that
\begin{align*}
\vhclippik_0 - \vpik_0 ~=~ \Exppik\left[\sum_{h=1}^H\Ehclip\kh(x_h,a_h)\right] ~\ge~ \frac{1}{2}(\vst_0 - \vpik_0),
\end{align*}
\end{lem}
The above bound is established in Section~\ref{sec:lem:half_clip}. Hence, to establish the bound of Theorem~\ref{thm:surplus_clipping}, it suffices to bound the gap $\vhclip^{\pikh}_0 - \vpik_0$. For a given $h$, and an $x: \pikh(x) \notin \pisth(x)$, let us consider the difference
\begin{align*}
\vhclippik_h(x) - \vpik_h(x) = \Ehclip\kh(x,\pikh(x)) + p(x,\pikh(x))^\top \left(\vhclippik_{h+1} - \vpik_{h+1} \right).
\end{align*}
We now introduce the following lemma, proven Section~\ref{sec:lem:gap_clipping}, which allows us to further clip the bonus for suboptimal actions $a \notin \pisth(x)$, i.e. , actions with $\gaph(x,a) > 0$:
\begin{lem}[Gap Clipping]\label{lem:gap_clipping} Suppose either $\Alg$ is strongly optimistic and each tuple is $\bsfa_{x,a,h}$-transition suboptimal. Then the fully-clipped surpluses
\begin{align*}
\Ebarkh(x,a) := \threshop{ \epsclip \vee \frac{\gaph\xa)}{4(\bsfa_{x,a,h}H \vee 1)}}{\Ekh(x,a)}
\end{align*}
satisfy the bound
\begin{align*}
\vhclippik_h(x) - \vpik_h(x) &\le \Ebarkh\left(x,\pikh(x)\right) + \left(1 + \frac{1}{H}\right) p(x,\pikh(x))^\top \left(\vhclippik_{h+1} - \vpik_{h+1} \right)
\end{align*}
If $\Alg$ is just optimistic, then the above bound holds with $\bsfa_{x,a,h}  = 1$. 
\end{lem}
Unfolding the above lemma, and noting that even when $\Alg$ is not strongly optimistic, the clipping ensures that $\Ebarkh\xa \ge 0$, so that we can bound
\begin{align*}
\vhclip^{\pik}_{k,0} - \vpik_0 &= \Exp^{\pik}[\vhclip^{\pik}_1(x_1) - \vpik_1(x_1)] \\
&\le \Exp^{\pik}[\Ebarkh\left(x_1,a_1\right) + \left(1 + \frac{1}{H}\right) p(x,\pikh(x))^\top \left(\vhclip^{\pik}_{2} - \vpik_{2} \right)] \\
&= \Exp^{\pik}[\Ebarkh\left(x_1,a_1\right) + \left(1 + \frac{1}{H}\right) \left(\vhclip^{\pikh}_{2}(x_2) - \vpik_{2}(x_2) \right)] \\
&\le \Exp^{\pik}\left[\sum_{h=1}^H \left( \prod_{h' =2}^h \left(1 + \frac{1}{H}\right)\right)  \Ebarkh(x_h,a_h)\right] ~\le \left(1 + \frac{1}{H}\right)^H\Exp^{\pik}\left[\sum_{h=1}^H  \Ebarkh(x_h,a_h)\right]\\
&\le e\Exp^{\pik}\left[\sum_{h=1}^H  \Ebarkh(x_h,a_h)\right] = e \sum_{x,a}\sum_{h=1}^H \wkh\xa\Ebarkh(x,a),
\end{align*}
where we recall $\wkh\xa = \Pr^{\pik}[(x_h,a_h) = (x,a)]$. Combining with our earlier bound $\vst_0 - \vpik_0 \le  2(\vhclip^{\pik}_{k,0}(x) - \vpik_0(x))$ from Lemma~\ref{lem:half_clip}, we find that $\vst_0 - \vpik_0 \le 2e\sum_{x,a}\sum_{h=1}^H \wkh\xa\Ebarkh(x,a)$, thereby demonstrating Theorem~\ref{thm:surplus_clipping}. 

\subsection{Proof of Lemma~\ref{lem:half_clip}\label{sec:lem:half_clip}}

We can with a crude comparison between the clipped and optimistic value functions.
\begin{lem}\label{lem:til_bounds} We have that $ \Ehclip\kh(x,\pikh(x)) \ge  \Ekh(x,\pikh(x)) - \epsclip$, which implies
\begin{align}\label{eq:bounds}
\vhclippik\kh(x) + (H-h+1) \epsclip \ge \vupkh(x)  \ge \vpikh(x).
\end{align}
\end{lem}
\begin{proof}
The bound $\Ehclip\kh(x,\pikh(x)) \ge  \Ekh(x,\pikh(x)) - \epsclip$ follows directly from 
\begin{align*}
\Ehclip\kh(x,a) = \Ekh(x,a)\I(\Ekh(x,a) \ge \epsclip) \geq \Ekh(x,a) - \epsclip.
\end{align*} 
Hence, 
\begin{align*}
\vhclippik\kh(x) - \vpikh(x) &\overeq{i.a} \Exp^{\pik}\left[\sum_{t = h}^H \Ehclip_{k,t}(x_t,\pikh(x_t)) \mid x_h = x\right]\\
&\ge \Exp^{\pik}\left[\sum_{t = h}^H  \Ekt(x_t,\pikh(x_t))) - \epsclip \mid x_h = x \right] \\
&=\Exp^{\pik}\left[\sum_{t = h}^H  \Ekt(x_t,\pikh(x_t)))  \right] - (H - h + 1)\epsclip \\
&\overeq{i.b} \vupkh(x) - \vpikh(x)-   (H - h + 1)\epsclip,
\end{align*}
where $(i.a)$ and $(i.b)$  follow by recursively unfolding the identities $\vhclippik\kh(x) - \vpikh(x) = \Ehclip\kh\xa $ $+ p\xa^\top (\vhclippik\khpl(x) - \vpikh(x))$ and $\vup\kh(x) - \vpikh(x) = \Ekh\xa + p\xa^\top (\vup\khpl(x) - \vpikh(x)).$

\end{proof}
We now turn to proving Lemma~\ref{lem:half_clip}. 
\begin{proof}
The strategy is as follows. We shall introduce the events over $\Pr^{\pik}$, $\calE_h := \{\pikh(x_h) \notin \pisth(x_h)\}$, which is the event that the policy $\pikh$ does not prescribe an optimal action  $x_h$. We further define the events
\begin{align*}
\calA_h = \calE_h \cap \bigcap_{h' < h} \calE_{h'}^c,
\end{align*}
which is the event that the policy $\pik$ agrees with an optimal action on $x_{1},\dots,x_{h-1}$, and disagrees on $x_h$. Below, our goal will be to establish the following two formulae for the suboptimality gap $\vst_0 - \vpik_0$ and $\vhclip^{\pik}_0 - \vpik_0$:
\begin{align}\label{eq:vtil_decomp}
\vhclip^{\pik}_0 - \vpik_0
&\ge \sum_{h=1}^H \Exp^{\pik}[\I(\calA_h) \left\{\gap(x_h,\pikh(x_h)) - H \epsclip + \qst_h(x_h,\pikh(x_h)) - \vpikh(x_h)\right\}  ]
\end{align}
and 
\begin{align}\label{eq:vst_decomp}
\vst_0 - \vpik_0 &= \sum_{h=1}^H \Exp^{\pik}[\I(\calA_h) \ \left\{\gap(x_h,\pikh(x_h)) + \qst_h(x_h,\pikh(x_h)) - \vpikh(x_h)\right\}  ]
\end{align}
Note that on $\calA_h$, $\calE_h = \{\pikh(x_h) \notin \pisth(x_h)\}$ also occurs, and therefore $\gap(x_h,\pikh(x_h)) \ge \gapmin$.  In particular, displays~\eqref{eq:vtil_decomp} and~\eqref{eq:vst_decomp} both imply
\begin{align*}
\vhclip^{\pikh}_0 - \vpik_0 &\overset{(i)}{\ge} \sum_{h=1}^H \Exp^{\pik}[\I(\calA_h) \left\{\frac{1}{2}\gap(x_h,\pikh(x_h))  + \qst_h(x_h,\pikh(x_1)) - \vpikh(x)\right\}  ]\\
&\overset{(ii)}{\ge} \frac{1}{2}\sum_{h=1}^H \Exp^{\pik}[\I(\calA_h) \left\{\gap(x_h,\pikh(x_j)) + \qst_h(x_h,\pikh(x_h)) - \vpikh(x)\right\}  ] \\
&\overset{(iii)}{\ge} \frac{1}{2}( \vst_0 - \vpik_0),
\end{align*}
where $(i)$ uses $\epsclip = \frac{\gapmin}{2H}$ and display~\eqref{eq:vtil_decomp}, $(ii)$ uses that  $\qst_h(x_h,\pikh(x_h)) - \vpikh(x) \ge 0$, and $(iii)$ uses display~\eqref{eq:vst_decomp}.

Let us start with proving~\eqref{eq:vtil_decomp}. First, consider a stage $h$, state $x$, and suppose that  $\pikh(x) \notin \pisth(x)$.  Observe that by Lemma~\ref{lem:til_bounds}, optimism, and the definition of $\gaph\xa$, we have that for any $\astar \in \pisth(x)$,
\begin{align*}
H \epsclip + \vhclippik\kh(x) &\ge \vupkh(x) = \qupkh(x,\pikh(x)) \ge \qupkh(x,\astar)\\
&\ge \qst_h(x,\astar) = \gap(x,\pikh(x)) + \qst_h(x,\pikh(x)).
\end{align*}
Subtracting, we find that for $\pikh(x) \notin \pisth(x)$,
\begin{align*}
 \vhclippik\kh(x) - \vpikh(x) &\ge  \gap(x,\pikh(x)) - H \epsclip  + \qst_h(x,\pikh(x)) - \vpikh(x) .\numberthis\label{eq:non_optimal_case}
\end{align*}
Now, on the other hand, if $\pikh(x) \in \pisth(x)$, then, 
\begin{align*}
\vhclippik\kh(x) - \vpikh(x) &= \Ehclip\kh(x,\pikh(x)) + r(x,\pikh(x)) + p(x,\pikh(x))^\top \vhclip_{k,h+1}^\pikh  \\
&\qquad- r(x,\pikh(x)) -p(x,\pikh(x))^\top \vpikh\\
&= \Ehclip\kh(x,\pikh(x)) +  p(x,\pikh(x))^\top ( \vhclip_{k,h+1}^\pikh  - \truev^{\pik}_{h+1}) \numberthis\label{eq:vtil_identity}\\
&\overset{(i)}{=} \Ehclip\kh(x,\pikh(x)) +  p(x,\pikh(x))^\top \delvhclip_{h+1} \\
&\overset{(ii)}{\ge}  p(x,\pikh(x))^\top \delvhclip_{h+1},\numberthis\label{eq:optimal_case}
\end{align*}
where in $(i)$ we have defined the increment $\delvhclip_h :=\vhclip_{k,h}^\pikh  - \truev^{\pik}_{h}$ with $\delvhclip_{H+1} = 0$, and $(ii)$ holds since $\Ehclip\kh(x,\pikh(x)) = \Ekh(x,\pikh(x))\I(\Ekh(x,\pikh(x)) \ge \epsclip) \ge 0$.

Now, recalling that $\calE_h$ denotes the event that $\pikh(x) \notin \pisth(x)$, we have
\begin{align*}
\Exp^{\pik}[\delvhclip_1] &\ge \Exp^{\pik}\left[\I(\calE_1) \left\{\gap(x_1,\pi_{k,1}(x_1)) - H \epsclip + \qst_h(x_1,\pi_{k,1}(x_1)) - \vpikh(x)\right\}   \right] \tag*{(by Eq.~\eqref{eq:non_optimal_case})}\\
&+ \Exp^{\pik}\left[\I(\calE_1^c) p(x_1,\pi_{k,1}(x_1))^\top \delvhclip_2  \right].\tag*{(by Eq.~\eqref{eq:optimal_case})}
\end{align*}
We continue with
\begin{align*}
\hspace{.5in}&\hspace{-.5in}\Exp^{\pik}\left[\I(\calE_1^c) p(x_1,\pi_{k,1}(x_1))^\top \delvhclip_2 \right]\\
\ge& \Exp^{\pik}\left[\I(\calE_1^c)\I(\calE_2) \left\{\gap(x_2,\pi_{k,2}(x_2)) - H \epsclip + \qst_2(x_2,\pi_{k,2}(x_2)) - \vpik_2(x)\right\} \right]\\
&+ \Exp^{\pik}\left[\I(\calE_1^c)\I(\calE_2^c) p(x_2,\pikh(x_2))^\top \delvhclip_3 \right].
\end{align*}
Recalling the event $\calA_h = \calE_h \cap \bigcap_{h' < h} \calE_{h'}^c$, we can continue the above induction to find that,
\begin{align*}
\Exp^{\pik}[\delvhclip_1] &\ge \sum_{h=1}^H \Exp^{\pik}[\I(\calA_h) \left\{\gap(x_h,\pikh(x_h)) - H \epsclip + \qst_h(x_h,\pikh(x_h)) - \vpik_h(x)\right\}  ]\\
&\qquad+ \underbrace{\Exp^{\pik}[\I(\bigcap_{h=1}^H \calE_{h}^c) p(x_h,\pikh(x_h))^\top \delvhclip_{H+1}  ]}_{= 0}, 
\end{align*}
as needed.  Now let's prove~\eqref{eq:vst_decomp}.  We can always write
\begin{align*}
\vsth(x) - \vpikh(x) = \gaph(x,a) + \qst_h(x,\pikh(x)) - \vpikh(x),
\end{align*} 
where $\gaph(x,a) = 0$ when $\pisth(x) \in \pikh(x)$, that is, on $\calE^c$. Hence, the same line of reasoning used to prove Eq.~\eqref{eq:vtil_decomp} (omitting the subtracted $\epsclip H$), verifies Eq.~\eqref{eq:vst_decomp}.
\end{proof}

\subsection{Proof of Lemma~\ref{lem:fundamental_gap_bound}\label{sec:lem:fundamental_gap_bound}}
	\begin{proof}
	For simplicity, set $a = \pikh(x)$, and let $\astar \in \pisth(x)$ be an action which witnesses the $\bsfa$ transition-suboptimality condition. We then have
	\begin{align*}
	\vupkh(x) &\overset{(i)}{=} \qupkh(x,a)~\overset{(ii)}{\geq} \qupkh(x,\astar)\\
	&= \qsth(x,\astar) + \left(\qupkh(x,\astar) -  \qsth(x,\astar)\right)  \\
	&~\overset{(iii)}{=} \gaph(x,a) + \qsth(x,a) + \left(\qupkh(x,\astar) -  \qsth(x,\astar\right),
	\end{align*} 
	where $(i)$ is by definition of $\vupkh(x)$, $(ii)$ is since $a = \pikh(x) = \argmax_{a'} \qupkh(x,a')$, and $(iii)$ is the definition of $\gaph(x,a)$. Rearranging, we have 
	\begin{align}
	\gaph \le \vupkh(x)  - \qsth(x,a)- \left(\qupkh(x,\astar) -  \qst(x,\astar\right) \label{eq:gaph_intermediate}
	\end{align}
	If $\Alg$ is not necessarily strongly optimistic then we bound $\qupkh(x,\astar) -  \qst(x,\astar) \ge 0$ and $\qsth(x,a) \ge \vpikh(x)$, yielding 
	\begin{align*}
	\gaph(x,a) &\le  \vupkh(x) -\vpikh(x) \\
	&= \qupkh(x,a) -\vpikh(x) \\
	&= \Ekh(x,a) + r(x,a) + p(x,a)^\top \vupkhpl -\vpikh(x) \\
	&= \Ekh(x,a) + p(x,a)^\top(\vupkhpl -\vpikhpl )
	\end{align*}
	which corresponds to the desired bound for $\bsfa = 1$.

	When $\Alg$ is strongly optimistic, we handle~\eqref{eq:gaph_intermediate} more carefully. Specifically, we compute
	\begin{align*}
	\vupkh(x) - \qsth(x,a) &= \Ekh(x,a) + r(x,a) + p(x,a)^\top \vupkhpl - \left(r(x,a) + p(x,a)^\top \vsthpl\right) \\
	&= \Ekh(x,a) + p(x,a)^\top ( \vupkhpl - \vsthpl).
	\end{align*}
	Moreover, recalling that $a^* \in \pisth(x)$, we have
	\begin{align*}
	\qupkh(x,a^*) -  \qst(x,a^*) &= r(x,a^*) + \Ekh(x,a^*) +  p(x,a^*)^\top \vupkhpl - r(x,a^*) -  p(x,a^*)^\top \vsthpl \\
	&=  \Ekh(x,a^*) + p(x,a^*)^\top (\vupkhpl - \vsthpl). 
	\end{align*}
	where the last inequality uses strong optimism of $\Alg$. Hence, 
	\begin{align*}
	\gaph\xa &\le \vupkh(x) -  \qsth(x,a) - \left(\qupkh(x,\astar) -  \qst(x,\astar)\right) \\
	&= \Ekh(x,a) + p(x,a)^\top ( \vupkhpl - \vsthpl) - \left(\Ekh(x,a^*) + p(x,a^*)^\top (\vupkhpl - \vsthpl)\right) \\
	&= \Ekh(x,a) - \Ekh(x,a^*) + (p(x,a)-p(x,a^*))^\top (\vupkh - \vsthpl)\\
	&\le \Ekh(x,a) + (p(x,a)-p(x,a^*))^\top (\vupkh - \vsthpl) \tag*{(Strong Optimism)}\\
	&\le \Ekh(x,a) + \bsfa p(x,a)^\top (\vupkh - \vsthpl),
	\end{align*}
	where the lastar line uses the component-wise inequalityes $p(x,a)-p(x,\astar) \le \bsfa p(x,a)$ due to the fact that $\astar$ witnesses the $\bsfa$ transition-suboptimality, and $\vupkh - \vsthpl \ge 0$ due to optimism. 
	\end{proof}

\subsection{Proof of Lemma~\ref{lem:gap_clipping}\label{sec:lem:gap_clipping}}

	\begin{proof} For ease, we suppress the dependence of $\bsfa$ on $\xah$. By our fundamental gap bound (Lemma~\ref{lem:fundamental_gap_bound}) and then Lemma~\ref{lem:til_bounds}, we have that 
	\begin{align*}
		\gaph\xa &\le \Ekh(x,a) + \bsfa \cdot p(x,a)^\top (\vupkhpl - \vpikhpl) \\
		&\le \Ehclip\kh(x,a) + \bsfa \cdot p(x,a)^\top (\vhclippik\khpl - \vpikhpl)  + (H - h + 1) \bsfa \epsclip\\
		&\le \Ehclip\kh(x,a) + \bsfa \cdot p(x,a)^\top (\vhclippik\khpl - \vpikhpl)  + \gaph\xa/2,
	\end{align*}
	where the inequality bounds $\bsfa (H - h + 1)\epsclip \le \bsfa \gapmin/2 \le \bsfa \cdot \gaphxa/2 \le \gaphxa/2$. This yields
	\begin{align*}
	\tfrac{1}{2}\gaphxa \le \Ehclip\kh(x,a) + \bsfa \cdot p(x,a)^\top (\vhclippik\khpl - \vpikhpl). 
	\end{align*}
	Now, fix a constant $c \in (0,1]$ to be chosen later. Either we have that $\Ehclip\kh(x,a) \ge \tfrac{c}{2} \gaphxa$, or otherwise,
	\begin{align*}
	\bsfa \cdot p(x,a)^\top (\vhclippik\khpl - \vpikhpl) \ge (1-c)\frac{1}{2}\gaphxa \ge \frac{1 - c}{c}\Ehclip\kh(x,a), 
	\end{align*}
	which can be rearranged into 
	\begin{align*}
	\Ehclip\kh(x,a) \le \frac{c\bsfa}{1-c}p(x,a)^\top (\vhclippik\khpl - \vpikhpl).
	\end{align*}
	Hence, we have
	\begin{align*}
	\Ehclip\kh(x,a) &\le \Ehclip\kh(x,a) \I\left\{ \Ehclip\kh(x,a) \ge \tfrac{c}{2}\gaphxa\right\} \\
	&\qquad+ \frac{c\bsfa}{1-c} p(x,a)^\top (\vhclippik\khpl - \vpikhpl) \I\left\{ \Ehclip\kh(x,a) < \tfrac{c}{2}\gaphxa\right\}\\
	&\le \Ehclip\kh(x,a) \I\left\{ \Ehclip\kh(x,a) \ge \tfrac{c}{2} \gaphxa\right\} + \frac{c\bsfa}{1-c} p(x,a)^\top (\vhclippik\khpl - \vpikhpl), 
	\end{align*}
	and thus, 
	\begin{align*}
	\Ehclip\kh(x,a) + p(x,a)^\top (\vhclippik\khpl - \vpikhpl)  &\le \Ehclip\kh(x,a) \I\left\{ \Ehclip\kh(x,a) \ge \tfrac{c}{2} \gaphxa\right\}   \\
	&+ (1 + \frac{c\bsfa}{1-c}) p(x,a)^\top (\vhclippik\khpl - \vpikhpl).
	\end{align*}
	In particular, choosing $c = \frac{1}{2}\min\{1,(\bsfa H)^{-1})\}$, we have $(1 + \frac{c\bsfa}{1-c}) \le 1 + \frac{1}{H}$, and 
	\begin{align*}
	\tfrac{1}{2} = \frac{(1 \wedge (\bsfa H)^{-1})}{4} = \frac{1}{4(\bsfa H \vee 1)},
	\end{align*}
	so that $\Ehclip\kh(x,a) \I\left\{ \Ehclip\kh(x,a) \ge \frac{c}{2} \gaphxa\right\} = \Ebarkh(x,a)$. This concludes the proof. 
	\end{proof}

\newpage
\part{$\strongeuler$ and its surpluses\label{part:strongeuler}}


\section{The $\strongeuler$ Algorithm\label{sec:alg:strong_euler}}
Before continuing, let us define a logarithmic factor we shall use throughout:
\begin{align}
\Lfactor(u) := \sqrt{ 2\log(10 M^2\max\{u,1\}/\delta)},\label{eq:Lfactor}
\end{align}
where we recall that $M=SAH \ge 2$. This section formally presents $\strongeuler$, which makes two subtle modification of the $\euler$ algorithm of~\cite{zanette2019tighter}.

First, similar to \citep{dann2017unifying,dann2018policy}, $\strongeuler$ refines the log factors in the bonuses to depend on the number of samples $\nk\xa$ via $\Lfactor(\nk\xa) \propto \log(M \nk\xa/\delta)$, rather than the overall time $T = KH$ via $\Lfactor(\nk\xa) \propto \log(MT/\delta)$, which is necessary to ensure the optimal $\log T$ regret. Following~\citep{dann2017unifying,dann2018policy}, our confidence bounds can be slightly refined using law-of-iterated logarithm bounds, but for simplicity we do not pursue this direction here.

Second, $\strongeuler$ satisfies \emph{strong optimism}. We remind the reader that strong optimism is not necessary to achieve gap dependent bounds, but can achieve sharper bounds for settings with simple transition dynamics like contextual bandits. The $\euler$ algorithm, or its predecessors (e.g.~\cite{azar2017minimax}), would also achieve-gap dependent bounds due to our analysis. Moreover, running these algorithms with the refined $\log(M \nk\xa/\delta)$ log factors would also yield $\log T$- asymptotic regret, whereas implementing $\log(MT/\delta)$ confidence intervals may yield asymptotic regret that scales as $\log^2T$ (see Remark~\ref{rem:anytime_log}).

The $\euler$ algorithm proceeds by standard optimistic value iteration, with carefully chosen exploration bonuses, and keeps track of various variance-related quantities:
\begin{algorithm}[h]
\textbf{Input: }\\
\textbf{Initialized: } For each $a\in\actions$ $x,x' \in \states$,
 $n_{1}(x,a) = 0$, $n_{1}(x'\mid x,a) = 0$, $\rsum_1 = 0 $, $\rsumsq_{1} = 0$, $\phat_1(x,a) = 0$, $\Varhat_1[R(x,a)] = 0$\\
\For {$k = 1,2,\dots$}
{
	$\vup_{k,H+1} \leftarrow 0$\\
	\For{$h = H,H-1,\dots,1$}
	{

		\For{$x \in \states$}
		{
			\For{$a \in \actions$}
			{
				Call $\mathsf{ConstructBonuses}$.\\
				$\qupkh(x,a) \leftarrow \min\{H - h+1,\rhat\xa +  \phat\kh\xa^\top \vupkhpl + $\\
				\Indp \qquad\qquad\qquad $\bonusprob\kh\xa + \bonusrew\subk\xa + \bonusstrong\kh\xa\}$

			}
			$\pikh(x) := \argmax_{a} \qupkh(x,a)$, $\ahat \leftarrow \pikh(x)$\\
			$\vupkh(x) := \qupkh(x,\ahat) $\\
			$\vlowkh(x) = \max\{0,\rhat(x,\ahat) - \bonusrew\kh(x,\ahat) + \phat\kh(x,\ahat)^\top \vlowkhpl - \bonusprob\kh(x,\ahat) - \bonusstrong\kh(x,\ahat)\}$. 
		}
	}
	\textbf{Call} $\mathsf{RolloutAndUpdate}(k)$.

	}

\caption{$\strongeuler$}\label{alg:strong_euler}
\end{algorithm}

The $\mathsf{RolloutAndUpdate}$ function (Algorithm~\ref{alg:rollout_update} below) executes one trajectory according to the policy $\pik$, and records all count- and variance- data regarding the relevant rewards and transition probabilities. Finally, the bonuses are are defined in Algorithm~\ref{alg:bonuses}.

~\begin{algorithm}[h]
\textbf{Input: } Global current episode $k$, global counts and empirical probabilities.
Initialize $k+1$-th episode counts: $n_{k+1}(\cdot,\cdot) \leftarrow n_{k}(\cdot,\cdot)$, $n_{k+1}(\cdot\mid \cdot, \cdot) \leftarrow n_{k}(\cdot\mid\cdot,\cdot)$, $\rsum_{k+1}(\cdot,\cdot) \leftarrow \rsum_{k}(\cdot,\cdot) $, $\rsumsq_{k+1}(\cdot,\cdot) \leftarrow \rsumsq_{k}(\cdot,\cdot) $.\\
	\For{$h = 1,\dots,H$}
	{
		Observe state $x_h$, play $a_h = \pikh(x_h)$, recieve reward $R$ and view next state $x_{h+1}$.\\
		$\nk(x_{h},a_h) \pluseq 1$, $\nk(x_{h'}| x_{h},a_h) \pluseq 1$, $\rsum(x,a) \pluseq R$, $\rsum(x,a) \pluseq R^2$
	}
	\For{$a \in \actions,x \in \states$}
	{
		\For{$x' \in \states$}
		{
			$\phat_{k+1}(x'|x,a) = \frac{\nk(x_{h'}| x_{h},a_h)}{\nk(x_{h},a_h)}$
		}
		$\rbar_{k+1}(x,a) = \frac{\rsum_{k+1}}{\nk(x_{h},a_h)}$, $\Varhat_{k+1}[R(x,a)] = \frac{\rsumsq_{k+1}}{\nk(x_{h},a_h)}- \rbar_{k+1}(x,a)^2$.
	}

	, 
	\caption{$\mathsf{RolloutAndUpdate}(k)$\label{alg:rollout_update}}
\end{algorithm}

\begin{algorithm}
\textbf{Bonuses: }
\begin{align}
 \bonusrew\subk\xa &:= 1 \wedge \left(\sqrt{\frac{2\Varhat_{k}[ R(x,a)]\Lfactor(\nk\xa)}{\nk\xa}} + \frac{8\Lfactor(\nk\xa)}{3(\nk\xa - 1)}\right) \label{eq:bonusrew_def}\\
 \bonusprob\kh\xa &:= H \wedge \Bigg{(}\sqrt{\frac{2\Var_{\phatk\xa}[ \vupkhpl]\Lfactor(\nk\xa)}{\nk\xa}} + \frac{8H\Lfactor(\nk\xa)}{3(\nk\xa - 1)} \nonumber\\
&\qquad+ \sqrt{\frac{2\Lfactor(\nk\xa)\|\vlowkhpl - \vupkhpl\|_{2,\phatk\xa}^2}{\nk\xa}}\Bigg{)}.\label{eq:bonusprob_def} \\
\bonusstrong\kh\xa &:=  \|\vupkhpl -  \vlowkhpl\|_{2,\phatk\xa}\sqrt{\frac{S\Lfactor(\nk\xa)}{\nk(x,a)}}    + \frac{8}{3}\frac{S H\Lfactor(\nk\xa)}{\nk\xa} \label{eq:bonusstrong_def} 
\end{align}
\caption{$\mathsf{ConstructBonuses}$\label{alg:bonuses}}
\end{algorithm}

\pagebreak

\section{Analysis of $\strongeuler$: Proof of Proposition~\ref{prop:surplus}\label{sec:prop_surplus_proof}}
Proposition~\ref{prop:surplus} requires demonstrating a lower bound on the surplus, $0 \le \Ekh\xa$, thereby establishing strong optimism, as well as an upper bound on the surplus, which we shall use to analyze the same complexity. We address strong optimism first in the next subsection, and then the upper bound in the following subsection. Throughout, we will assume that a good event $\goodconcentration$ holds. To keep the proofs modular, the event $\goodconcentration$ will only appear as an assumption in the supporting lemmas used in Sections~\ref{sec:optimism_proof} and~\ref{sec:surplus_upper_bound_proof}. Then, in Section~\ref{sec:concentration_sec}, we formally define $\goodconcentration$ in terms of 6 constituent events, establish $\Pr[\goodconcentration] \ge 1 - \frac{\delta}{2}$, and conclude with proofs of the supporting lemmas which rely on $\goodconcentration$. We remark that many of the arguments in this section are similar to those from \cite{zanette2019tighter}, with the main differences being strong optimism and the additional care paid to log-factors, necessary for $\log T$ regret. Again, recall the definition $\Lfactor(u) := \sqrt{ 2\log(10 M^2\max\{u,1\}/\delta)}$. 

\subsection{Proof of Optimism\label{sec:optimism_proof}}
Here we establish the optimism of $\strongeuler$, and in particular, the bound $\Ekh\xa \ge 0$. 
\begin{prop}\label{prop:main_euler_properties} Under the good event $\goodconcentration$,
\begin{enumerate}
	\item[(a)] $\strongeuler$ is \emph{optimistic}: $\pikh(x) = \argmax_{a}\qupkh(x,a)$, where $\qupkh(x,a) \ge \qsth(x,a)$ for all $h,x,a$. In particular, $\vupkh(x) \ge \vsth(x)$ for $h \in [0:H]$.
	\item[(b)] $\strongeuler$ is \emph{strongly optimistic} $\Ekh(x,a) := \qupkh(x,a) - r(x,a) - p(x,a)^\top \vupkhpl(x) \ge 0 $. 
	\item[(c)] $\vlowkh \le \vpikh \le \vsth \le \vupkh$
\end{enumerate}
\end{prop}
\begin{proof}
The policy choice $\pikh(x) = \argmax_{a}\qupkh(x,a)$ holds by definition of the algorithm. We now give the remainder of the argument by inducting backwards on $h$. For $h = H+1$, $\vup_{k,H+1} = \vlow_{k,H+1} = \vst_{k,H+1} = \vpikhpl = 0$. Now, suppose as an inductive hypothesis that $\vupkhpl \ge \vsthpl \ge \vpikhpl \ge \vlowkhpl$, and $\Ekhpl(x,a) \ge 0$ for all $x,a$. 

First, we shall show that $\Ekh(x,a) \ge 0$ for all $x,a$. This will establish the induction for point $b$. It also establishes $(a)$, since then $\qupkh(x,a) \ge r(x,a) + p(x,a)^\top \vupkhpl(x) \ge r(x,a) + p(x,a)^\top \vsthpl = \qsth(x,a)$, proving optimism. To this end, note that
\begin{align*}
\Ekh(x,a)&:= \qupkh(x,a) - r(x,a) - p(x,a)^\top \vupkhpl(x)\\
&:=  \min\{H - h+1,\rhat\xa +  \phat\kh\xa^\top \vupkhpl + \bonusprob\kh\xa + \bonusrew\subk\xa + \bonusstrong\kh\xa\} \\
&\qquad- r(x,a) - p(x,a)^\top \vupkhpl(x).
\end{align*}
Since $ r(x,a) + p(x,a)^\top \vupkhpl(x) \le H-h+1$, it suffices to show that
\begin{align*}
\rhat\xa +  \phat\kh\xa^\top \vupkhpl + \bonusprob\kh\xa + \bonusrew\subk\xa + \bonusstrong\kh\xa  - r(x,a) - p(x,a)^\top \vupkhpl(x) \ge 0.
\end{align*}
Grouping the terms, it suffices to show that $\rhat\xa - r(x,a) + \bonusrew\subk\xa  \ge 0$, and that
\begin{align*}
0 &\le (\phat\kh\xa^\top - p(x,a))^\top \vupkhpl(x)+ \bonusprob\kh\xa  + \bonusstrong\kh\xa\\
&= \left\{(\phat\kh\xa^\top - p(x,a))^\top \vsthpl(x)+ \bonusprob\kh\xa \right\} \\
&\qquad+ \left\{\phat\kh\xa^\top - p(x,a))^\top (\vupkhpl(x)-\vsthpl(x))  + \bonusstrong\kh\xa\right\}.
\end{align*}
We lower bound $\rhat\xa - r(x,a) + \bonusrew\subk\xa $ and  $(\phat\kh\xa^\top - p(x,a))^\top \vsthpl(x)+ \bonusprob\kh\xa$ by zero with the following lemma:
\begin{lem}\label{lem:empirical_diff_bound} On the good concentration event $\goodconcentration$, it holds that
\begin{align*}
|\rhat\xa - r\xa| &\le \bonusrew\subk\xa, \\
|(\phat(x,a) - \p(x,a))^\top\vsthpl| &\le \bonusprob\kh\xa  \quad \text{if }   \vlowkhpl\le\vsthpl \le \vupkhpl
\end{align*}
\end{lem}
We conclude the proof of (b) with the following lemma, which lets us bound 
\begin{align*} (\phat\kh\xa^\top - p(x,a))^\top (\vupkhpl(x)-\vsthpl(x))  + \bonusstrong\kh\xa\ge 0
\end{align*}
Precisely we apply the following lemma with $V_2 = \vupkhpl$ and $V_1 = \vsthpl$:
\begin{lem}\label{lem:strongbonus} Suppose that $\goodprob \supset \goodconcentration$ holds, and suppose that $V_1,V_2:\states \to \R$ satisfies $\vlowkhpl \le V_1 \le V_2 \le \vupkhpl$. Then, 
\begin{align*}
\left| (\phat(x,a) - \p(x,a))^\top(V_1 - V_2)\right| \le \bonusstrong\kh\xa 
\end{align*}
\end{lem}

This finally establishes (b). We conclude by establishing (c). Here, we note that by definition $\vpikh \le \vsth $, and $\vsth \le \vupkh$ as show above. Hence, it suffices to show $\vlowkh \le \vpikh$. We begin with the inequality
\begin{align*}
  \vpikh(x) &= \p\xast^\top\vpikhpl + r\xast \\
  &=   \phat\xast^\top \vpikhpl + \rhat\xast +  (r\xast - \rhat \xast) +  (\p\xast^\top - \phat\xast^\top)\vpikhpl\\
   &= \phat\xast^\top \vpikhpl + \rhat\xast +  (r\xast - \rhat \xast) +  (\p\xast^\top - \phat\xast^\top)\vsthpl \\
   &\qquad+ (\p\xast - \phat\xast)^\top(\vpikhpl - \vsthpl)\\
   &\ge \phat\xast^\top \vpikhpl + \rhat\xast - \bonusrew\kh\xa - \bonusprob\kh\xa - \bonusstrong\kh\xa,
\end{align*}
where the last inequality uses the bounds $(r\xast - \rhat \xast)  \ge - \bonusrew\kh\xa$ and $(\p\xast^\top - \phat\xast^\top)\vsthpl \ge -  \bonusprob\kh\xa$ on $\goodconcentration$ due to Lemma~\ref{lem:empirical_diff_bound}, and bounds $(\p\xast - \phat\xast)^\top(\vpikhpl - \vsthpl) \ge - \bonusstrong\kh\xa$ by applying Lemma~\ref{lem:strongbonus} with $V_1 = \vpikhpl$ and $V_2 = \vsthpl$, which satisfy $\vlowkhpl \le V_1 \le V_2 \le \vupkhpl$ by our inductive hypothesis (namely, $\vupkhpl \ge \vsthpl \ge \vpikhpl \ge \vlowkhpl$). Since $ \vpikh(x) \ge 0$ as well, and since $\vpikhpl \ge \vlowkhpl$ by our inductive hypothesis, we therefore have 
\begin{align*}
  \vpikh(x) &\ge 0 \vee \phat\xast^\top \vlowkhpl + \rhat\xast - \bonusrew\kh\xa - \bonusprob\kh\xa - \bonusstrong\kh\xa = \vlowkh(x),
\end{align*}
This completes the induction. 
\end{proof}

\subsection{Proof of Surplus Bound Upper Bound\label{sec:surplus_upper_bound_proof}}

Throughout, we assume the round $k$ is fixed, and suppress the dependence of $\phat$, $\Varhat$, and $\rhat$ on $k$. We use the shorthand $p = p(x,a)$ and $\phat = \phat(x,a)$, where the pair $(x,a)$ are clear from context. 
\begin{align*}
\Ekh(x,a) &= \vupkh(x,a) - r(x,a) - p(x,a)^\top \vupkhpl \\
&\le \bonusrew_k\xa + \rhat\xa + \phat(x,a)^\top \vupkhpl  + \bonusprob\kh\xa - r(x,a)  - p(x,a)^\top \vupkhpl\\
&= (\rhat\xa - r(x,a) + (\phat(x,a) - p\xa)^\top\vsthpl \\  
&\qquad+ \bonusrew_k\xa + \bonusprob\kh\xa +  (\phat - \p)^\top (\vupktpl -  \vsthpl)\\
&\le 2\bonusrew_k\xa + 2\bonusprob\kh\xa +  \bonusstrong\kh\xa.
\end{align*}
where the last line is by Lemmas~\ref{lem:empirical_diff_bound} and~\ref{lem:strongbonus}. Next, we state a standard lemma that lets us swap out the empirical variance for the true variance in upper bounding $\bonusrew\subk\xa$:
\begin{lem}\label{lem:reward_empirical_variance_bound}
Under the event $\goodconcentration$, $ \bonusrew\subk\xa \lesssim\sqrt{\frac{\Var[R\xa]\Lfactor(\nk\xa)}{\nk\xa}}  + \frac{\Lfactor(\nk\xa)}{\nk\xa}$.
\end{lem}

 Next, we recall from the definition of $\bonusprob$,
\begin{align*}
 \bonusprob\kh\xa &\lesssim \sqrt{\frac{\Var_{p\xa}[ \vsthpl]\Lfactor(\nk\xa)}{\nk\xa}} + \frac{H\Lfactor(\nk\xa)}{\nk\xa} + \sqrt{\frac{\Lfactor(\nk\xa)\|\vlowkhpl - \vupkhpl\|_{2,\phat}^2}{\nk\xa}}.
\end{align*}
where we replaced $\nk\xa - 1$ by $\nk\xa$ in the deminator of one of the terms by taking advantage of the `$H\wedge$'.  Furthermore, we can bound
\begin{align*}
\sqrt{\frac{\Var_{p\xa}[ \vsthpl]\Lfactor(\nk\xa)}{\nk\xa}} &\le \sqrt{\frac{\min\{\Var_{p\xa}[ \vsthpl], \Var_{p\xa}[ \vpikhpl]\}  \Lfactor(\nk\xa)}{\nk\xa}} \\
&+ \left|\sqrt{|\Var_{p\xa}[\vsthpl]} - \sqrt{\Var_{p\xa}[ \vpikhpl]}\right| \sqrt{\frac{\Lfactor(\nk\xa)}{\nk\xa}}.
\end{align*}
We can control the difference $|\sqrt{|\Var_{p\xa}[\vsthpl]} - \sqrt{\Var_{p\xa}[ \vpikhpl]}|$ using the following lemma:
\begin{lem}\label{lem:variance_fundamental} Let $X,Y$ be two real valued random variables, and let $\|\cdot\|_{p,2} := \sqrt{\Exp[(\cdot)^2]}$. Then $|\sqrt{\Var[X]} - \sqrt{\Var[Y]}| \le \sqrt{\Var[X-Y]} \le \|X-Y\|_{2,p}$.
\end{lem}
\begin{proof} The inequality $\Var[X-Y] \le \Exp[(X-Y)^2] = \|X-Y\|_{2,p}^2$ follows since $\Var[Z] \le \Exp[Z^2]$ for any random variable $Z$. For the first inequality, we can assume WLOG that $X,Y$ are mean zero, in which case $\sqrt{\Var[X]} = \|X\|_{2,p}$, and similarly for $Y$ and $X-Y$. The result now follows from the fact that the norm $\|\cdot\|_{p,2}$ satisfies the triangle inequality.
\end{proof}
We shall also need the following simple fact:
\begin{fact}\label{fact:monotone_fact} If $V_1(x) \le V_2(x) \le V_3(x) \le V_4(x)$ for all $x \in \states$, then $\|V_2 - V_3\|_{2,p} \le \|V_1 - V_4\|_{2,p}$.
\end{fact}
Since $\vlowkhpl \le \vpikhpl \le \vsthpl \le \vupkhpl$ by Proposition~\ref{prop:main_euler_properties}, Lemma~\ref{lem:variance_fundamental}  and Fact~\ref{fact:monotone_fact} above yield
\begin{align*}
|\sqrt{\Var_{p\xa}[\vsthpl]} - \sqrt{\Var_{p\xa}[ \vpikhpl]}| \le \|\vsthpl - \vpikhpl\|_{2,p} \le \|\vupkhpl - \vlowkhpl\|_{2,p}. 
\end{align*}
Together the with the elementary inequality, $\sqrt{a+b} \le \sqrt{a} + \sqrt{b} \lesssim \sqrt{a+b}$, this in turn yields
\begin{align*}
 &\bonusprob\kh\xa +   \bonusrew\kh\xa \\
 &\lesssim  \bonusrew\kh\xa +  \sqrt{\frac{\min\{\Var_{p\xa}[ \vsthpl], \Var_{p\xa}[\vpikhpl]\}\Lfactor(\nk\xa)}{\nk\xa}} + \frac{H\Lfactor(\nk\xa)}{\nk\xa} \\
 &\qquad+ \sqrt{\frac{\Lfactor(\nk\xa)\|\vlowkhpl - \vupkhpl\|_{2,\phat + p}^2}{\nk\xa}}\\
 &\lesssim  \sqrt{\frac{\Var[R\xa] + \min\{\Var_{p\xa}[ \vsthpl], \Var_{p\xa}[\vpikhpl]\}\Lfactor(\nk\xa)}{\nk\xa}} + \frac{H\Lfactor(\nk\xa)}{\nk\xa} \\
 &\qquad+ \sqrt{\frac{\Lfactor(\nk\xa)\|\vlowkhpl - \vupkhpl\|_{2,\phat + p}^2}{\nk\xa}},\\
 &\lesssim  \sqrt{\frac{\Varkxah\Lfactor(\nk\xa)}{\nk\xa}} + \frac{H\Lfactor(\nk\xa)}{\nk\xa} + \sqrt{\frac{\Lfactor(\nk\xa)\|\vlowkhpl - \vupkhpl\|_{2,\phat + p}^2}{\nk\xa}},
\end{align*}
where we use the shorthand $\|V\|_{2,\phat + p} = \sqrt{\|V\|^2_{2,p} + \|V\|_{2,\phat}^2}$, and where in the last, we recall that $\Varkxah = \min\{\Varpikxah,\Varxah\}  = \Var[R\xa] + \min\{\Var_{p\xa}[ \vsthpl], \Var_{p\xa}[\vpikhpl]\}$. 

Next, substituing in $\bonusstrong\kh\xa :=  \|\vupkhpl -  \vlowkhpl\|_{2,\phat\xa}\sqrt{\frac{S\Lfactor(\nk\xa)}{\nk(x,a)}}    + \frac{8}{3}\frac{S H\Lfactor(\nk\xa)}{\nk\xa}$, we obtain
\begin{align*}
&\bonusrew_k\xa + \bonusprob\kh\xa + \bonusstrong\kh\xa \\
&\lesssim \sqrt{\frac{\Varkxah \Lfactor(\nk\xa)}{\nk\xa}} + \frac{H\Lfactor(\nk\xa)}{\nk\xa} \\
&\qquad+ \sqrt{\frac{\Lfactor(\nk\xa)\|\vlowkhpl - \vupkhpl\|_{2,\phat+p}^2}{\nk\xa}} + \sqrt{\frac{S\|\vupkhpl -  \vlowkhpl\|_{2,\phat}^2\Lfactor(\nk\xa)}{\nk(x,a)}}   + \frac{S H\Lfactor(\nk\xa)}{\nk\xa}\\
&\lesssim \sqrt{\frac{\Varkxah \Lfactor(\nk\xa)}{\nk\xa}} + \frac{SH\Lfactor(\nk\xa)}{\nk\xa} + \sqrt{\frac{S\|\vupkhpl -  \vlowkhpl\|_{2,p+\phat}^2\Lfactor(\nk\xa)}{\nk(x,a)}}   \numberthis\label{eq:intermediate_bound}\\
&\overset{(i)}{\le} \sqrt{\frac{\Varkxah \Lfactor(\nk\xa)}{\nk\xa}} + \frac{SH\Lfactor(\nk\xa)}{\nk\xa} + \frac{S\Lfactor(\nk\xa)}{\nk\xa} + \|\vupkhpl -  \vlowkhpl\|_{2,\phat + p}^2\\
&\overset{(ii)}{=} \sqrt{\frac{\Varkxah \Lfactor(\nk\xa)}{\nk\xa}} + \frac{SH\Lfactor(\nk\xa)}{\nk\xa}  + 2\|\vupkhpl -  \vlowkhpl\|_{2,p}^2 + (\phat - p)(\vupkhpl - \vlowkhpl)^2\\
&\overset{(iii)}{\le} \sqrt{\frac{\Varkxah \Lfactor(\nk\xa)}{\nk\xa}} + \frac{SH\Lfactor(\nk\xa)}{\nk\xa}  + 2\|\vupkhpl -  \vlowkhpl\|_{2,p}^2 + H(p - \phat)^\top (\vupkhpl - \vlowkhpl),
\end{align*}
where $(i)$ uses the inequality $a/b \le a^2 + \frac{1}{b^2}$, and $(ii)$ uses the facts that $\|V\|_{2,\phat+p}^2 = \|V\|_{2,p}^2 + \|V\|_{2,\phat}^2$  and $\|V\|_{2,\phat}^2 = \langle \phat, V^2 \rangle = \langle \phat, V^2 \rangle + \langle \phat - p , V^2 \rangle =  \|V\|_{2,\phat}^2 + \langle \phat - p , V^2 \rangle $. Lastly, inequality $(iii)$ uses $0 \le \vlowkhpl \le \vupkhpl \le H$.

We continue bounding $H(p - \phat)^\top (\vupkhpl - \vlowkhpl)$ in much the same way that we bounded the term in Lemma~\ref{lem:strongbonus} in terms of $\bonusstrong$, with the exception that we seek a term which depends on the true transition probability $p(x,a)$, and not the empirical $\phat(x,a)$:
\begin{lem}\label{lem:strongbonus2} Under $\goodconcentration$,
\begin{align*}
\left| (\phat(x,a) - \p(x,a))^\top(\vupkhpl - \vlowkhpl)\right| \lesssim \|\vupkhpl -  \vlowkhpl\|_{2,p(x,a)}\sqrt{\frac{S\Lfactor(\nk\xa)}{\nk(x,a)}}    + \frac{S H\Lfactor(\nk\xa)}{\nk\xa}.
\end{align*}
\end{lem}
The proof of the above lemma is ommitted for the sake of brevity, and follows from a simplified version of the proof of Lemma~\ref{lem:strongbonus} where we need not pass through an empirical variance.  Applying the bound in Lemma~\ref{lem:strongbonus2}, we have
\begin{align*}
 H(p - \phat)^\top (\vupkhpl - \vlowkhpl) &\lesssim  H\|\vupkhpl -  \vlowkhpl\|_{2,p}\sqrt{\frac{S\Lfactor(\nk\xa)}{\nk(x,a)}}  + \frac{S H^2\Lfactor(\nk\xa)}{\nk\xa}\\
 &\lesssim  \|\vupkhpl -  \vlowkhpl\|_{2,p}^2 + \frac{H^2S\Lfactor(\nk\xa)}{\nk(x,a)}  + \frac{S H^2\Lfactor(\nk\xa)}{\nk\xa},
\end{align*}
where the last line uses the inequality $ab \le (a^2 + b^2)/2$. Finally, combining the above with our previous bound, we arrive at 
\begin{align*}
\Ekh(x,a) &\lesssim \bonusrew_k\xa + \bonusprob\kh\xa + \bonusstrong\kh\xa   \\
&\lesssim\sqrt{\frac{\Varkxah \Lfactor(\nk\xa)}{\nk\xa}} + \frac{SH^2\Lfactor(\nk\xa)}{\nk\xa}  + \|\vupkhpl -  \vlowkhpl\|_{2,p\xa}^2.
\end{align*}
From first principles, it is straightforward to show that $\Ekh(x,a) \lesssim H$, which implies that
\begin{align}\label{eq:min_bonus_with_h}
\Ekh(x,a) &\lesssim H \wedge \sqrt{\frac{\Varkxah \Lfactor(\nk\xa)}{\nk\xa}}\\
&\qquad + \left(H \wedge \frac{SH^2\Lfactor(\nk\xa)}{\nk\xa}\right) +  \|\vupkhpl -  \vlowkhpl\|_{2,p\xa}^2.
\end{align}

To conclude the proof, it remains to unravel the term $\|\vupkhpl -  \vlowkhpl\|_{2,p\xa}^2$.
\begin{lem}\label{lem:future_bound} Define the term
\begin{align*}
\matZk\xa &=  H^2 \wedge H^2\left(\sqrt{\frac{S\Lfactor(\nk \xa)}{\nk\xa}} + \frac{S\Lfactor(\nk\xa)}{\nk\xa} \right)^2,
\end{align*}
Then, we have the bound
\begin{align}
\vupkh(x) - \vlowkh(x) \lesssim \Exppik\left[\sum_{t=h}^H \sqrt{ \matZ(x_t,a_t)} \mid x_h = x\right].
\end{align}
\end{lem}
As a consequence, we can compute
\begin{align*}
(\vupkh(x) - \vlowkh(x) )^2 \lesssim \Exppik\left[\left(\sum_{t=h}^H \sqrt{ \matZ_k(x_t,a_t)}\right)^2 \mid x_h = x\right] \le H\Exppik\left[\sum_{t=h}^H \matZ_k(x_t,a_t) \mid x_h = x\right].
\end{align*}
Hence, we have
\begin{align*}
 \|\vupkhpl -  \vlowkhpl\|_{2,p}^2 &= \Exp_{x'\sim p\xa}(\vupkhpl(x) - \vlowkhpl(x) )^2 \\
 &\lesssim H\Exppik\left[\sum_{t=h+1}^H \matZ_k(x_t,a_t) \mid \xhah = \xa\right].
\end{align*}
Since $H\matZk(x,a) \ge H \wedge \frac{SH^2\Lfactor(\nk\xa)}{\nk\xa}$. Hence, we can bound via \eqref{eq:min_bonus_with_h}
\begin{align*}
\Ekh(x,a) \lesssim \underbrace{H \wedge \sqrt{\frac{\Varkxah \Lfactor(\nk\xa)}{\nk\xa}}}_{:=\boundlead\kh\xa} + \Exppik\left[\sum_{t=h}^H\underbrace{ H\matZ_k(x_t,a_t) }_{:=\boundfutk(x_t,a_t)}\mid \xhah = \xa\right]
\end{align*}
where we note that the summation in the expectation now begins at $t = h$ to account for the term $ H \wedge \frac{SH^2\Lfactor(\nk\xa)}{\nk\xa}$ in \eqref{eq:min_bonus_with_h}, and recal that 
\begin{align*}
\boundfut_k(x,a) := H^3 \wedge H^3\left(\sqrt{\frac{S\Lfactor(\nk \xa)}{\nk\xa}} + \frac{S\Lfactor(\nk\xa)}{\nk\xa} \right)^2 = H\matZk\xa
\end{align*}
To conclude, we recall the definitions,
\begin{align*}
\Lfactor(u) := \sqrt{ 2\log(10 M^2\max\{u,1\}/\delta)}.
\end{align*}
so that, for $u \ge 1$, $\Lfactor(u) \lesssim \log \frac{Mu}{\delta}$. 

\subsection{Definition of $\goodconcentration$, and proofs of supporting lemmas \label{sec:concentration_sec}}

Before proving the lemmas above, we formally express the good event $\goodconcentration$ as a list of constituent concentration events, and verify that it occurs with probability at least $1-\delta/2$:
\begin{prop}\label{prop:conc_event} The event $\goodconcentration := \goodreward \cap \goodprob  \cap \goodval \cap \goodvarval \cap \goodvarrew$ occurs with probability $1 - \delta/2$, where each of the  constituent events occurs with probability at least $1-\delta/12$:
\begin{align*}
\goodreward &:= \left\{\forall k,x,a,h: |\rhat_k\xa - r\xa| \le \sqrt{\Var[R(x,a)]\frac{2\Lfactor(\nk\xa)}{\nk\xa}} + \frac{2\Lfactor(\nk\xa)}{3\nk\xa}\right\} \\
\goodprob &:= \left\{\forall k,x,x',a,h: |\phat(x'\mid x,a) - \p(x'\mid x,a)| \le \sqrt{\p(x'\mid x,a)(1-\p(x'\mid x,a))\frac{2\Lfactor(\nk\xa)}{\nk\xa}} + \frac{2\Lfactor(\nk\xa)}{3\nk\xa}\right\} \\
\goodval &:= \left\{\forall k,x,a,h: |(\phat(x,a) - \p(x,a))^\top\vsthpl| \le \sqrt{\Varpxa[\vsthpl]\frac{2\Lfactor(\nk\xa)}{\nk\xa}} + \frac{2H\Lfactor(\nk\xa)}{3\nk\xa}\right\} \\
\goodvarprob &:= \left\{\forall k,h,x,a: |\phat(x'\mid x,a) - \p(x'\mid x,a)| \le \sqrt{\frac{2\Lfactor(\nk\xa)}{\nk\xa}}\right\}.\\
\goodvarval &:= \left\{\forall k,h,x,a: \left|\|\vst_h\|_{2,\phat(x,a)} - \|\vst_h\|_{2,p\xa}\right|  \le H\sqrt{\frac{2\Lfactor\nk\xa)}{\nk\xa - 1}}\right\}\\
\goodvarrew &:= \left\{\forall k,h,x,a: \left|\sqrt{\Varhat(R(x,a))} - \sqrt{\Var(R(x,a))}\right|  \le \sqrt{\frac{2\Lfactor(\nk\xa)}{\nk\xa -1 }}\right\}
\end{align*}
\end{prop}
\begin{proof}
The proof of these the first four events follows from standard applications of Bernstein's and Hoeffding's inequality, and the last two from \iftoggle{nips}{Theorem 10 in~\cite{maurer2009empirical}}{\citet[Theorem 10]{maurer2009empirical}}. Similar proofs can be found in~\citep{zanette2019tighter,azar2017minimax,dann2017unifying}. As in those works, the only subtlety is to use the appropriate concentration inequality with respect to an appropriate filtration to attain bounds that depend on $\Lfactor (\nk\xa)$, rather than on $\Lfactor (T)$.

Let's prove $\goodreward$ as an example. We it suffices to only consider rounds for which $\nk\xa \ge 1$, for otherwise the bound is vacuous. Fix an action $\xa$, and let $\tau_i \in \{1,2,\dots\} \cup \{\infty\}$ denote the round $k+1$ immediately after the $i$-th round $k$ at which a pair $\xa$ is observed at least once during the rollout, and define a sub-filtration $\{\calG_i\}$ via $\calG_i = \calF_{\tau_i}$. Then, for any given $i$, a martingale analogue of Bernstein's inequality yields 
\begin{align*}
\Pr\left[ |\rhat_{\tau_i}\xa - r\xa|  \I(\tau_n < \infty) \ge\sqrt{ \frac{2\Var[R(x,a)]\log(2 /\eta)}{n_{\tau_i\xa}}} + \frac{2\log(2 /\eta)}{3n_{\tau_i\xa}}\right] \le \eta. 
\end{align*}
Now fix an $i \ge 1$. Since $\rhat_k\xa$ and $\nk\xa$ are constant for  $k \in \{\tau_i,\dots,\tau_{i+1} - 1\}$, we have 
\begin{align*}
\forall n, ~\Pr\left[ \exists k \in \{\tau_i,\dots,\tau_{i+1} - 1\}  : |\rhat_{k}\xa - r\xa|  \ge \sqrt{ \frac{2\Var[R(x,a)]\log(2 /\eta)}{\nk\xa}} + \frac{2\log(2/\eta)}{3\nk\xa}\right] \le \eta,
\end{align*}
Applying the above with $\eta \leftarrow 2\eta/i^2$ and union bounding over $n$, we have
\begin{align*}
\Pr\left[ \exists i,k: k \in \{\tau_i,\dots,\tau_{i+1} - 1\} ,  |\rhat_{k}\xa - r\xa|  \ge \sqrt{ \frac{2\Var[R(x,a)]\log(4i^2 /\eta)}{\nk\xa}} + \frac{2\log(4i^2/\eta)}{3\nk\xa}\right] \le \eta.
\end{align*}
Since $\nk\xa$ increments by at least one for each $\tau_i$, we have $i \le \nk\xa$ for $k \in \{\tau_i,\dots,\tau_{i+1} - 1\}$. Thus,
\begin{align*}
\Pr\left[ \exists i,k: k \in \{\tau_i,\dots,\tau_{i+1} - 1\} ,  |\rhat_{k}\xa - r\xa|  \ge \sqrt{ \frac{2\Var[R(x,a)]\log(4\nk\xa^2/\eta)}{\nk\xa}} + \frac{2\log(4\nk\xa^2/\eta)}{3\nk\xa}\right] \le \eta.
\end{align*}
Lastly, since for any $k$, there always exist some $i$ for which $k \in \{\tau_i,\dots,\tau_{i+1} - 1\}$, we have 
\begin{align*}
\Pr\left[ k:  |\rhat_{k}\xa - r\xa|  \ge \sqrt{ \frac{2\Var[R(x,a)]\log(4\nk\xa^2 /\eta)}{\nk\xa}} + \frac{2\log(4\nk\xa^2/\eta)}{3\nk\xa}\right] \le \eta.
\end{align*}
We then conclude by union bounding over $SA$, and letting $\eta = \delta/12SA$, yielding the following log factor: $\log(48SA\nk\xa^2/\delta) \le \Lfactor(\nk\xa)$, where we recall $\Lfactor(u) = \sqrt{ 2\log(10 M^2\max\{u,1\}/\delta)}$ for $M = SAH$. The proof for $\goodprob$ is analogous, the proof for $\goodval$ requires union bounding over states $x'$, incuring a log factor $\log(4S^2A\nk\xa^2/\delta) \le \Lfactor(\nk\xa)$.
\end{proof}

\begin{proof}[Proof of Lemma~\ref{lem:empirical_diff_bound}]
We prove the bound $|(\phat(x,a) - \p(x,a))^\top\vsthpl| \le \bonusprob\kh\xa $; the analogous bounds for rewards is similar. Note that since $\phat(x,a)^\top \vsthpl \in [0,H]$ and $p(x,a)^\top \vsthpl \in [0,H]$, $|(\phat(x,a) - \p(x,a))^\top\vsthpl| \in [0,H]$. This takes care of the first '$H\wedge$' in $\bonusprob\kh\xa$. Next, on $\goodval$ and $\goodvarval$,
\begin{align*}
&|(\phat(x,a) - \p(x,a))^\top\vsthpl| \\
&\le \sqrt{\frac{2\Var_{\p\xa}[ \vsthpl]\Lfactor(\nk\xa)}{\nk\xa}} + \frac{2H\Lfactor(\nk\xa)}{3\nk\xa} \tag*{(on $\goodval$)}\\
&\le \sqrt{\frac{2\Varphatxa[\vsthpl]\Lfactor(\nk\xa)}{\nk\xa}} + \frac{8H\Lfactor(\nk\xa)}{3(\nk\xa - 1)} \tag*{(on $\goodvarval$)}\\
&= \sqrt{\frac{2\Varphatxa[\vupkhpl]\Lfactor(\nk\xa)}{\nk\xa}} + \frac{8H\Lfactor(\nk\xa)}{3(\nk\xa - 1)}\\
&+  (\sqrt{\Varphatxa[\vsthpl]} - \sqrt{\Varphatxa[\phat\xa^\top \vupkhpl]}) \sqrt{\frac{2\Lfactor(\nk\xa)}{\nk\xa}}.
\end{align*}
Lastly, by Lemma~\ref{lem:variance_fundamental}, we have the bound
\begin{align*}
\left|\sqrt{\Varphatxa[ \vsthpl]} - \sqrt{\Varphatxa[ \vupkhpl]}\right| \le  \|\vlowkhpl - \vupkhpl\|_{2,\phat\xa}.
\end{align*}
\end{proof}

\begin{proof}[Proof of Lemma~\ref{lem:strongbonus}]
Summing up the condition of event $\goodprob$ over states $x' \in \states$, and then applying event $\goodvarprob$ to control $|p(x'|x,a) - \phat(x'|x,a)|$: 
\begin{align*}
(\phat(x,a) - \p(x,a))^\top(\vupkhpl - \vsthpl) &\le \sum_{x'} \sqrt{ 2\frac{ \Lfactor(\nk\xa)p(x'|x,a)(1-p(x'|x,a))}{\nk\xa} }|V_2(x') -  V_1(x')| \\
&\quad + \frac{2}{3}\sum_{x'}\frac{\Lfactor(\nk\xa)}{\nk\xa} |V_2(x') - V_1(x')|,\\
&\overset{(i)}{\le} \sum_{x'} \sqrt{ 2\frac{ \Lfactor(\nk\xa)p(x'|x,a)}{\nk\xa} }|V_2(x') - V_1(x')| \\
&\quad + \frac{2}{3}\sum_{x'}\frac{\Lfactor(\nk\xa)}{\nk\xa} |V_2(x') - V_1(x')|,\\
&\overset{(ii)}{\le} \sum_{x'} \sqrt{ 2\frac{ \Lfactor(\nk\xa)\phat(x'|x,a)}{\nk\xa} }|V_2(x') - V_1(x')| \\
&\quad + \underbracer{\le \frac{8}{3}\frac{HS\Lfactor(\nk\xa)}{\nk\xa}}{\frac{8}{3}\sum_{x'}\frac{\Lfactor(\nk\xa)}{\nk\xa} |V_2(x') - V_1(x')|},
\end{align*}
where $(i)$ uses $p(x'|x,a)(1-p(x'|x,a)) \le p(x'|x,a)$,  $(ii)$ uses event $\goodvarprob$, and where bound in the bracket is because there $|\vupktpl(x') -  \vsthpl(x')| \le H$ by Proposition~\ref{prop:main_euler_properties} part (b), and there are at most $S$ terms in the summation. To bound the first term, we have
\begin{align*}
&\sum_{x'}  \sqrt{ 2\Lfactor(\nk\xa)\frac{\phat(x'|x,a)}{\nk(x,a)} }|V_2(x') - V_1(x')| \\
&=\sqrt{2\frac{\Lfactor(\nk\xa)}{\nk(x,a)}} \sum_{x'} \sqrt{ \phat(x'|x,a) } |V_2(x') - V_1(x')|  \\
&\overset{(i)}{\le} \sqrt{2\frac{\Lfactor(\nk\xa)}{\nk(x,a)}}  \sqrt{ S\|V_2 -  V_1\|_{2,\phat}^2} \\
&\overset{(ii)}{\le}\|\vupkhpl -  \vlowkhpl\|_{2,\phat}\sqrt{2\frac{S\Lfactor(\nk\xa)}{\nk(x,a)}}  
\end{align*}
 $(i)$ bounds uses Cauchy-Schwartz, and $(ii)$ uses Proposition~\ref{prop:main_euler_properties} part (c) to bound $\|V_2 - V_1\|_{2,p} \le \|\vupkhpl -  \vlowkhpl\|_{2,\phat}$ for $\vlowkhpl \le V_1 \le V_2 \le \vupkhpl$, in light of Fact~\ref{fact:monotone_fact}.

\end{proof}

\begin{proof}[Proof of Lemma~\ref{lem:reward_empirical_variance_bound}]
Under the event $\goodconcentration$ we have
\begin{align*}
& \bonusrew\subk\xa \\
&\lesssim 1 \wedge \left(\sqrt{\frac{\Varhat[R\xa]\Lfactor(\nk\xa)}{\nk\xa}} + \frac{\Lfactor(\nk\xa)}{\nk\xa - 1} \tag*{(definition)}\right)\\
 &\le 1 \wedge \left(\sqrt{\frac{\Var[R\xa]\Lfactor(\nk\xa)}{\nk\xa}} + \left|\sqrt{\Var[R\xa]} - \sqrt{\Varhat[R\xa]} \right| \sqrt{\frac{\Lfactor(\nk\xa)}{\nk\xa}} + \frac{\Lfactor(\nk\xa)}{\nk\xa - 1}\right)\\
 &\lesssim 1 \wedge \left(\sqrt{\frac{\Var[R\xa]\Lfactor(\nk\xa)}{\nk\xa}}  + \frac{\Lfactor(\nk\xa)}{\nk\xa -1}\right)~\lesssim\sqrt{\frac{\Var[R\xa]\Lfactor(\nk\xa)}{\nk\xa}}  + \frac{\Lfactor(\nk\xa)}{\nk\xa},
 \end{align*}
where in the second-to-last inequality, we used the event $\goodvarrew$ to control $\left|\sqrt{\Var[R\xa]} - \sqrt{\Varhat[R\xa]} \right|  \lesssim \sqrt{ \frac{\Lfactor(\nk\xa)}{\nk\xa - 1 }}$. 
 \end{proof}

 \begin{proof}[Proof of Lemma~\ref{lem:future_bound}]
Let $a = \pikh(x)$. Then, by definition of $\vupkh(x), \vlowkh(x)$
\begin{align*}
\vupkh(x) - \vlowkh(x) &= (H - h + 1) \wedge (\phat\kh\xa^\top \vupkhpl + \bonusrew\subk(x,a) + \bonusprob\kh\xa + \bonusstrong\kh\xa) \\
&\qquad - 0 \vee ( \phat\kh\xa^\top \vlowkhpl - \bonusrew\subk(x,a) - \bonusprob\kh\xa)\\
&\le H \wedge \left\{\phat\kh\xa^\top (\vupkhpl - \vlowkhpl) + 2\bonusrew\subk(x,a) + 2\bonusprob\kh\xa + \bonusstrong\kh\xa \right\}\\
&= H \wedge \Bigg\{ p\xa^\top (\vupkhpl - \vlowkhpl) + (\phat\xa - p\xa)^\top (\vupkhpl -  \vlowkhpl)  \\
&\qquad+ 2\bonusrew\subk(x,a) + 2\bonusprob\kh\xa + \bonusstrong\kh\xa\Bigg\} \\
&\le H \wedge \left\{p\xa^\top (\vupkhpl - \vlowkhpl) + 2\bonusrew\subk(x,a) + 2\bonusprob\kh\xa + 2\bonusstrong\kh\xa\right\} \tag*{(Lemma~\ref{lem:strongbonus})}, \\
&\le  p\xa^\top (\vupkhpl - \vlowkhpl) + H \wedge  \left\{2\bonusrew\subk(x,a) + 2\bonusprob\kh\xa + 2\bonusstrong\kh\xa\right\},
\end{align*}
where the last line uses the fact that $p\xa^\top (\vupkhpl - \vlowkhpl) \ge 0$ on $\goodconcentration$ (Proposition~\ref{prop:main_euler_properties}, part (c)). Unfolding the above expression inductively, we then find that
\begin{align*}
\vupkh(x) - \vlowkh(x) \le \Exppik\left[\sum_{t=h}^H  H \wedge  \left\{2\bonusrew\subk(x_t,a_t) + 2\bonusprob\kh(x_t,a_t) + 2\bonusstrong\kh(x_t,a_t)\right\} \mid x_t = x\right].
\end{align*}
To conclude, it suffices to check that $H \wedge  \left\{2\bonusrew\subk(x,a) + 2\bonusprob\kh(x,a) + 2\bonusstrong\kh(x,a)\right\} \lesssim \sqrt{\matZk(x,a)}$, for any triple $x,a,h$. To check that this bound holds, we have from~\eqref{eq:intermediate_bound} that
\begin{align*}
&2\bonusrew\subk\xtat + 2\bonusprob\kt\xtat + 2\bonusstrong\kt\xtat \\
&\lesssim \sqrt{\frac{\Varxat \Lfactor(\nk\xtat)}{\nk\xtat}} + \frac{SH\Lfactor(\nk\xtat)}{\nk\xtat} + \sqrt{\frac{S\|\vupktpl -  \vlowktpl\|_{2,(\phat_k +p)\xtat}^2\Lfactor(\nk \xtat)}{\nk\xtat}}\\
&\lesssim \sqrt{\frac{H \Lfactor(\nk\xtat)}{\nk\xtat}} + \frac{SH\Lfactor(\nk\xtat)}{\nk\xtat} + H\sqrt{\frac{S\Lfactor(\nk \xtat)}{\nk\xtat}},
\end{align*}
where we recall the notation $\|V\|_{2,\phat+p} = \sqrt{\|V\|_{2,p}^2+\|V\|_{2,\phat}^2}$, and thus the final bound holds since $\Varxat \le H$ implying that $\|\vupktpl -  \vlowktpl\|_{2,(\phat_k + p)\xtat}^2 \le 4H$ for $0 \le \vlowktpl \le \vupktpl \le H$.  Consolidating the terms, we have $2\bonusrew\subk\xtat + 2\bonusprob\kt\xtat + 2\bonusstrong\kt\xtat$ is at most $ \lesssim \left(H\sqrt{\frac{S\Lfactor(\nk \xa)}{\nk\xa}} + \frac{SH\Lfactor(\nk\xa)}{\nk\xa} \right)$, and thus $H \wedge 2\bonusrew\subk\xtat + 2\bonusprob\kt\xtat + 2\bonusstrong\kt\xtat$ is $\lesssim H \wedge \left(H\sqrt{\frac{S\Lfactor(\nk \xa)}{\nk\xa}} + \frac{SH\Lfactor(\nk\xa)}{\nk\xa} \right) 
:= \sqrt{\matZk\xa}.$
\end{proof}

\newpage
\part{Lower Bounds\label{part:lb}}
\newcommand{\cbonus}{c_{\mathrm{bon}}}
\newcommand{\cbonusup}{\overline{c}_{\mathrm{bon}}}
\newcommand{\bonusrewk}{\bonusrew_k}

\newcommand{\Sminus}{\calS_-}
\newcommand{\Splus}{\calS_+}
\newcommand{\Eagg}{\calE^{\mathrm{agg}}}
\newcommand{\Eover}{\calE^{\Delta}}

\newcommand{\ntauj}{n_{\tau_j}}
\newcommand{\Eopt}{\mathcal{E}^{\mathrm{opt}}}

\section{Min-Gap Lower Bound for Optimistic Algorithms (Theorem~\ref{thm:lower_bound_informal})\label{sec:min-gap-lower-bound}}

\subsection{Formal Statement}
We begin a formal version of the lower bound, Theorem~\ref{thm:lower_bound_informal}.
\begin{thm}\label{thm:lower_mingap}
Let $c_1,c_2,c_3$ be absolute constants that may depend on the constants defined in Section~\ref{sec:alg_clas}.
Let $\Alg$ denote an algorithm in the class described in Section~\ref{sec:alg_clas} run with confidence parameter $\delta \in (0,1/8)$. 
For any $S\geq 1$ and $\epsilon \le 1/\ceil{c_1 S \log(S/ \delta))}$, fix any MDP in the class described in Section~\ref{sec:lower_bound_mingap_instance} so that $|\states| = 2S+1$, $|\actions|=2$, $H=2$, and exactly one state has a sub-optimality gap of $\gapmin = \epsilon$ and all other states have a minimum sub-optimality gap of at least $1/2$.
Then $\sum_{h,x,a:\gaph\xa > 0}\frac{1}{\gaphxa} \lesssim S + \frac{1}{\gapmin}$ but $\Alg$ for all sufficiently large $K$ suffers a regret \begin{align*}
\regret_K \geq  \frac{c_2 S}{\gapmin}\log (1/\delta) \gtrsim \sum_{h,x,a:\gaph\xa > 0}\frac{1}{\gaphxa} + \frac{S}{\gapmin}
\end{align*}
with probability at least $1 -  c_2 S \epsilon^{-2} \log(1/\delta) e^{-c_3 S} -3\delta$.
\end{thm}

In particular, for any $\epsilon \in (0,c)$ for some constant $c$, if $\log(\epsilon^{-1}/\delta) \lesssim S \lesssim \epsilon^{-1} / \log(\epsilon^{-1}/\delta)$ then the above regret lower bound  holds with probability $1-O(\delta)$. 

\subsection{Algorithm Class \label{sec:alg_clas}} 
\textbf{Optimistic Q-functions:} We consider algorithms where the optimistic Q-function is constructed as follows: given a reward bonus function $\bonusrew_k\xa \ge 0$ and an additional nonnegative stage-dependent bonus $\bonus_{k,h}(x,a)$, and empirical estimates $\rhatk\xa$ of the reward and $\phatk\xa = (\phat(x'|x,a))$ of the transition probabilities. We set the Q-function at stage $H$ as $\qup_{k,H}\xa = \rhat_k\xa + \bonusrew_k\xa$, where $\rhat_k\xa$, and for $h \in \{1,\dots,H-1\}$,
\begin{align}
\vup\khpl(x)&:= \max_{a} \qup_{k,h+1}(x',a)\nonumber\\
\qup_{k,h}(x,a) &:= \rhat_k\xa + \bonusrew_k\xa + \phat_k(x,a)^\top \vup\khpl + \bonus_{k,h}(x,a)\label{eq:model_based_class}
\end{align}
Lastly, suppose that $\bonusrew_k\xa$ depends only on rewards collected when the state $\xa$ is visited. 

Note that this template subsumes the model-based approaches of \cite{azar2017minimax,zanette2019tighter,dann2018policy}, and if $\bonusrew(x,a)$ is made to be time dependent, captures the approach of \cite{dann2017unifying} as well. For the specific lower bound instance we consider, each stage $x \in \states$ can only be visited at a single stages $h \in [2]$, so $\bonusrew$ may be chosen to be time dependent without loss of generality. In order to capture the ``model-free'' methods based on Q-learning due to \cite{jin2018q}, we can instead mandate that 
\begin{align*}
\qup_{k,h}(x,a) &:= \rhat_k\xa + \bonusrew_k\xa + \widehat{\left(p(x,a)^\top \vsthpl\right)} + \bonus_{k,h}(x,a),
\end{align*}
where $\widehat{\left(p(x,a)^\top \vsthpl\right)}$ is a generalized estimate of $p(x,a)^\top \vsthpl$, and such that $\widehat{\left(p(x,a)^\top \vsthpl\right)}$ is nonnegative. In  Lemma 4.2 in \cite{jin2018q}, one can see that we can take
\begin{align*}
\widehat{\left(p(x,a)^\top \vsthpl\right)} = \sum_{s=1}^{\nk\xa} \alpha_{s}\widehat{\Pr}_{k_s,h}(x,a)^\top \vup_{k_s,h+1}(x,a),
\end{align*}
where $k_s$ is the round at which $(x,a)$ was selected for the $s$-th time, $\alpha_s$ is an appropriate weight, $\widehat{\Pr}_{k_s,h}(x,a)$ is the empirical probability estimate $\widehat{\Pr}_{k_s}(x,a)[x'] = \I(x' = x_{k_s,h+1})$ equal to indicator at the state $x_{k_s,h+1}$ visited after playing $a$ at $x$ at round $k_s$,  and where $\vup_{k_s,h+1}$ is an optimistic estimate of $\vsthpl$ at round $k_s$. 

For simplicity, we shall work with the model based formulation~\eqref{eq:model_based_class}, though the lower bound can be extended to this more general class. 

\textbf{Confidence Interval Assumptions: } Our class of algorithms takes in a confidence parameter $\delta \in (0,1/8)$. We shall also assume that there exists consants $\cbonus,\cbonusup$ such that, when the algorithm is run with parameter $\delta$, the bonuses $\bonusrew$ and $\bonusrew_k$ satisfy\footnote{The quantity $\Var[R(x,a)]$ below can also be replaced with an empirical variance, but we choose the true variance for simplicity.}
\begin{align*}
\bonus_{k,h}\xa \ge \frac{\cbonus}{1 \vee \nk \xa}, ~ \cbonus\sqrt{\frac{\Var[R(x,a)] \log(1/\delta)}{1 \vee \nk\xa}} \le \bonusrew_k\xa  \le \cbonusup \sqrt{\frac{\log(M(1 \vee \nk(x,a))/\delta)}{1 \vee \nk\xa}}
\end{align*}
We further assume that $\bonusrew\xa$ is $\delta$-correct, in the sense that, 
\begin{align*}
\Pr[\forall x,a,k: \bonusrew_k\xa + \rhat_k\xa \ge r\xa] \ge 1 - \delta.
\end{align*}
Lastly, we shall assume that the optimistic overestimate is \emph{consistent} in the sense that for any MDP $\calM$ with optimal value $\mathbf{V}_0^{*,\calM}$, for any $\epsilon, \delta > 0$ there exists a function $f_{\calM}$ such that
\begin{align*}
\Pr[ \forall k \ge f_{\calM}(\epsilon,\delta),~\vup_{k,0} - \mathbf{V}^{*,\calM}_0 \le \epsilon] \ge 1 -\delta.
\end{align*}
Intuitively, this condition states that with high probability, the optimistic over-estimate of the value estimate approaches the expected reward under the optimal policy. 
Note that this does \emph{not} assume uniform convergence of the entire value function itself, just the expected reward with respect to the initial state distribution $p_0$ on the optimal policy.

\begin{rem} Note that we do not require that our algorithm's confidence  intervals are ``inflated'', in the sense that, with high probability, $\rhatk\xa + \bonusrewk\xa - r(x,a) \ge c \bonusrewk\xa$, for a universal constant $c$. With this stronger assumption, we note that the proof of the lower bound can be simplified, and some restrictions on $S,\epsilon$ removed. In the interest of generality, we refrain from making this assumption. 
\end{rem}

\subsection{Formal Lower Bound Instance }\label{sec:lower_bound_mingap_instance}

Consider the following simple game with $H = 2$, $\actions = \{-1,+1\}$ and $\calS = \{-S,\dots,-1,0,1,\dots,S\} = \Sminus \cup \{0\} \cup \Splus$, where $\calS_- = -[S]$ and $\Splus = [S]$ (note $|\calS|= 2S+1)$. The game always begins at state $x_1 = 0$ with two available actions, $a \in \{-1,+1\}$. Then, $x_2 | (x_1 = 0,a_1 = +1) \unifsim \Splus$, and $x_2 | (x_1  = 0,a_1 = -1) \unifsim \Sminus$. Lastly, let $\calD$ denote \emph{any} symmetric distribution on $[-1,1]$ with $\Omega(1)$ variance. For $\epsilon \in (0,1/8)$, we formally define the reward distributions
\begin{align*}
R(x,a) \sim \begin{cases} 0 & x = 0 \text{ or } a = -1\\
\frac{1}{2} +\epsilon + \frac{1}{4}\calD & (x,a) = (s,1), s \in [S]\\
\frac{1}{2}  +  \frac{1}{4}\calD & (x,a) = (-s,1), s \in [S]
\end{cases}~.
\end{align*}
It is straightforward to verify the following fact
\begin{fact} The optimal action is always $a = 1$. Moreover, $\gap_1(0,-1) = \gapmin = \epsilon$, whereas $\gap_2(x,-1) \ge \frac{1}{2}$ for $x \ne 0$.  
\end{fact}
In other words, all the gaps for suboptimal arms are  $\Omega(1)$, except for the gap at state $x = 0$, which means for this instance with $H=2$ and $A=2$ we have $\sum_{x,a,h} \frac{1}{\gaph(x,a)} \eqsim S + \frac{1}{\epsilon}$. Nevertheless, we shall show that any algorithm in the class above suffers regret
\begin{align*}
\gtrsim \frac{S}{\epsilon}\log(1/\delta) = \frac{S}{\gapmin}\log(1/\delta).
\end{align*}

\subsection{The Lower Bound: }

\textbf{The Lower Bound:} We first show that the optimistic Q-function relative to the optimal value at $(0,1)$ decays at a rate of at least $\sqrt{S\log(1/\delta)/\nk(0,1)}$. This will ultimately lead to incurring a regret of $\frac{S\log(1/\delta)}{\epsilon}$, despite the fact that all but one of the Q-function gaps are $\Omega(1)$.
\begin{prop}\label{prop:q_lower_bound} Let $\Alg$ denote an algorithm in the class described in Section~\ref{sec:alg_clas} run with confidence parameter $\delta \in (0,1/8)$. Then there exists constants $c_1,c_2,c_3$, depending only on the constants described in Section~\ref{sec:alg_clas}, such that the following holds. For any $\epsilon \le 1/\ceil{c_1 S \log(S/ \delta))}$ and for $N = \floor{c_2 S \log(1/\delta) /\epsilon^2}$,
\begin{align*}
\Pr\left[ \forall k: \nk(0,-1) \le N, ~\qup_{k,1}(0,-1) - \vst_1(0) \ge \epsilon\right] \ge 1 - Ne^{-c_3 S} -2\delta
\end{align*}
\end{prop}

We now use Proposition~\ref{prop:q_lower_bound} to prove Theorem~\ref{thm:lower_mingap}.
Note that $\vst_1(0) = \vst_0$. 
By assumption, with probability $1 - \delta$, $\vbar_0 \le \vst_0 + \eta $ after $f(\eta,\delta)$ rounds.
Fix an appropriate $\epsilon$ and $N$ in Proposition~\ref{prop:q_lower_bound} and let $K \geq f_{\calM}(\epsilon/2, \delta) + N$.
If $n_{K}(0,-1) > N$ times, then we have 
\begin{align*}
\regret_K > \epsilon N \gtrsim \frac{S \log(1/\delta)}{\epsilon}
\end{align*}
and the theorem is proved. 
Thus, suppose not so that $n_{K}(0,-1) \leq N$. 
Then by Proposition~\ref{prop:q_lower_bound} we have with high probability that 
\begin{align*}
\vbar_0 - \vst_1(0) = \max_{a \in \{-1,1\}} \qup_{k,1}(0,a) - \vst_1(0) \geq \qup_{k,1}(0,-1) - \vst_1(0) \geq \epsilon
\end{align*}  
However, by assumption $K \geq f_{\calM}(\epsilon/2, \delta)$ which means that on an event that holds with probability at least $1-\delta$, we have $\vbar_0 - \vst_1(0) = \max_{a \in \{-1,1\}} \qup_{k,1}(0,a) - \vst_1(0) \le \epsilon/2$, a contradiction.
 
\subsubsection{Proof of Proposition~\ref{prop:q_lower_bound}}

Throughout, we will use upper case $C_1,C_2,\dots$ to do denote possibly changing numerical constants that depend on the the constants in the definition of $\Alg$, as set in Section~\ref{sec:alg_clas}. The lower cast constants $c_1,c_2$ will be coincide with those in Proposition~\ref{prop:q_lower_bound}.

Since $\qst_{1}(0,1) = \frac{1}{2} + \epsilon$, it suffices to show that
\begin{align*}
\Pr\left[ \forall k: \nk(0,-1) \le N, ~\qup_{k,1}(0,-1) - \frac{1}{2} \ge 2\epsilon\right] \ge 1 - Ne^{-C_3 S} - \delta
\end{align*}

Fix an $n_0 =  \ceil{c_1S/\log(S/\delta)}$ for a constant $c_1$ be  specified later, and let
\begin{align*}
\Eopt := \left\{\forall k \ge 1, x \in \Sminus,~ \rhat_k(x,1) + \bonusrew_k\xa \ge r(x,1) = \frac{1}{2}\right\}.
\end{align*}
By the optimism assumption, $\Eopt$ holds with probability at least $1-\delta$.  First we verify that $\qup_{k,1}(0,-1) - \frac{1}{2} \ge 2\epsilon$ for $0 \le \nk(0,-1) \le n_0$, provided that $\epsilon$ is sufficiently small:
\begin{claim} Suppose that $\epsilon \le \frac{\cbonus}{2 n_0}$. Then, with probability $1-\delta$, $\qup_{k,1}(0,-1) - \frac{1}{2} \ge 2\epsilon$ whenever $0 \le \nk(0,-1) \le n_0$:
\end{claim}
\begin{proof}
We have that
\begin{align*}
\qup_{k,1}(0,-1) = \bonus_{k,1}\xa + \sum_{x' \in \calS}\phatk(x'|0,-1)\vup_{k,2}(x').
\end{align*}
Since $p(x|0,-1) = 0$ for $x \notin \Sminus$, the empirical probability $\phat(x|0,-1)$ is also $0$, and thus 
\begin{align}
\qup_{k,1}(0,-1) - \frac{1}{2} &= \bonus_{k,1}\xa + \sum_{x' \in \Sminus}\phatk(x'|0,-1)(\vup_{k,2}(x') - \frac{1}{2}) \label{value_decomp}\\
&\ge \bonus_{k,1}\xa + \min_{x' \in \Sminus}(\vup_{k,2}(x')-\frac{1}{2}) \ge \bonus_{k,1}\xa \nonumber,
\end{align}
where the first equality and first inequality use $\sum_{x' \in \Sminus}\phat(x|0,-1) = 1$, and the second uses the optimistic event $\Eopt$ to show that $\vup_{k,2}(x') \ge  \rhatk(x',1) +\bonusrewk(x',1) \ge r(x',1) = \frac{1}{2}$ for $x' \in \Sminus$. Using the  assumption that $\bonus_{k,1}\xa \ge  \frac{\cbonus}{1 \vee \nk\xa}$, we see that if $\nk\xa \le n_0$ and $\epsilon \le \frac{\cbonus}{2 n_0}$, then $\bonus_{k,1}\xa \ge  \frac{\cbonus}{n_0} \ge 2\epsilon$, as needed.
\end{proof}

Now, we turn to the case where $\nk\xa \in \{n_0,\dots,N\}$ for some $N = \floor{c_2 S\log(1/\delta)/\epsilon^2}$. It light of~\eqref{value_decomp}, it suffices to show that for $\nk \le N$, 
\begin{align}
 \sum_{x' \in \Sminus}\phatk(x'|0,-1)(\vup_{k,2}(x') - \frac{1}{2}) \ge 2\epsilon.
\end{align}
By the definition of our algorithm class, the optimistic Q-function at stage $h = 2$ and pair $\xa$ depend only at rewards collected at $\xa$, and the construction of our MDP, pairs $\xa$ for $x \in \Sminus$ are only accessible by playing $(0,-1)$. 
Hence, to analyze $\qup_{k,1}(0,-1)$, for $n_0 \le \nk(0,-1) \le N$, it suffices to prove our described lower bound on $\qup_{k,1}(0,-1)$ in the simplified game, where at each round $k = 1,2,\dots$, the algorithm \emph{always} selects $(0,-1)$, and show that for this algorithm
\begin{align*}
\Delta_0(k) :=  \sum_{x \in \Sminus}\phatk(x|0,-1)(\vup_{k,2}(x) - \frac{1}{2}) \ge 2\epsilon, \forall k \in\{n_0,\dots,N\}.
\end{align*}
Turning our attention to this simplified game,for $x \in \Sminus$ let $\nk(x)$ denote the number of times $x$ has been visited up to round $k$, and recall $\nk(x,a)$ is the number of times action $a$ is played at stage $s$. Further, set
\begin{align*}
\Delta(x,k) :=  \vup_{k,2}(x) - \frac{1}{2}
\end{align*}
We now make a couple of observations
\begin{enumerate}
\item[(a)] The vector $(\nk(x))_{x \in \Sminus}$ is a uniform multinomial on the states in $\Sminus$. 
\item[(b)] Conditioned on $(\nk(x))_{x \in \Sminus}$, we can see that the values of $\vup_{k,2}(x)$ are independent, because for each $x \in \Sminus$, the game decouples into $\nk(x)$ rounds of a two arm bandit game on actions $a \in \{-1,1\}$.
\end{enumerate}
Using these observations, we prove the following claim:
\begin{claim}\label{claim:reward_overestimate} There exists constants $C_1,C_2$ such that for any $x \in \Sminus$, if $\delta \leq 1/8$ and $\nk(x) \ge C_1 \log(M/\delta)$, then conditioned on the history $(n_j(x'))_{x' \in \Sminus, j \ge 1}$, the following event holds with probability at least $1/4$:
\begin{align*}
\Eover_{k}(x) := \left\{\Delta(k,x) := \vup_{k,2}(x) - \frac{1}{2} \ge C_2\sqrt{\log(1/\delta)/n_{k}(x)}\right\},
\end{align*}
and the events $\{\Eover_{j}(x): x \in \Sminus\}$ are mutually independent (again, given $(n_j(x'))_{x' \in \Sminus, j \ge 1}$).
\end{claim}
Therefore, on the optimistic event $\Eopt$, where $\{ \Delta(k,x) \ge 0\}$, we can lower bound (again,  in the simplified game where we always select action $(0,-1)$),
\begin{align*}
\Delta_0(k) &\ge  \sum_{x \in \Sminus}\phat(x'|0,-1) \Delta(x,k)\\
&\ge \sum_{x \in \Sminus}\phat(x'|0,-1)\I(\Eover_{k}(x))C_2\sqrt{\frac{\log(1/\delta)}{\nk(x)}}\\
&\overset{(i)}{=} \sum_{x \in \Sminus}\frac{\nk(x)}{k} \I(\Eover_{k}(x))C_2\sqrt{\frac{\log(1/\delta)}{\nk(x)}}\\
&=  \frac{C_2\sqrt{\log(1/\delta)}}{k}\sum_{x \in \Sminus}\I(\Eover_{k}(x))\sqrt{\nk(x)}
\end{align*}
where $(i)$ uses the fact that for $x\in \Sminus$ is only accessible through $(0,-1)$, and that $(0,-1)$ is always selected in the simplified game. Next, observe that in the simplified game, $\nbark(x) = k/S$, so that if $n_0/S \ge C_3\log(1/\delta)$ for some constant $C_3$, it holds by an argument similar to Lemma~\ref{lem:sample} that with probability $1 - \delta$, the event $\calE_1 := \{\forall x \in \Sminus, \forall k \ge n_0, ~ \nk(x) \ge \nbark(x)/4 = k/4S\}$ holds, yielding 
\begin{align*}
\Delta_0(k) &\ge \frac{C_2\sqrt{\log(1/\delta)}}{2k}\sqrt{k/S}\cdot\sum_{x \in \Sminus}\I(\Eover_{k}) \\
&= \frac{C_2}{2}\sqrt{{S\log(1/\delta)}/{k}} \cdot \left(\frac{1}{S}\sum_{x \in \Sminus}\I(\Eover_{k})\right).
\end{align*}
Finally, if in addition $n_0/4S \ge C_1 \log(1/\delta)$, where $C_1$ is the constant from claim~\ref{claim:reward_overestimate}, then on $\calE_1$, it holds that for $k \ge n_0$, $\nk(x) \ge C_1 \log(1/\delta)$. We then set the constant $c_1$ so that $ n_0/S \ge C_3\log(1/\delta)$  and  $n_0/4S \ge C_1 \log(1/\delta)$ hold.

Lastly, since (a) $\calE_1$ is measurable with respect to the counts $(n_j(x'))_{x' \in \Sminus, j \ge 1}$, (b) since $\Eover_k(x)$ are independent given these counts, and (c) $\Exp[\I(\Eover_{k})] \ge 1/4$, a Chernoff bound shows that for $k \ge n_0$,  the event $\calE_2(k) :=\{\left(\frac{1}{S}\sum_{x \in \Sminus}\I(\Eover_{k})\right) \ge 1/8\}$ holds with probability at least $e^{-C_5 S}$ conditioned on $\calE_1$. Hence, on $\Eopt \cap \calE_1 \cap \bigcup_{k=n_0}^N\calE_2(k)$, we have 
\begin{align*}
\Delta_0(k) &\ge  \frac{C_2}{16}\sqrt{{S\log(1/\delta)}/{k}} \ge \frac{C_2}{16}\sqrt{{S\log(1/\delta)}/{N}}, \quad\forall k \in \{n_0,\dots,N\}.
\end{align*}
Hence, if $N \le c_2 \frac{S\log(1/\delta)}{\epsilon^2}$ for some constant $c_2$, we see that $\Delta_0(k) \ge 2\epsilon$ for all $k \in \{n_0,\dots,N\}$. Lastly, we see that 
\begin{align*}
\Pr[(\Eopt \cap \calE_1 \cap \bigcup_{k=n_0}^N\calE_2(k))^c] &\le \Pr[(\Eopt)^c] + \Pr[\calE_1^c] +  \Pr[(\bigcup_{k=n_0}^N\calE_2(k))^c\wedge  \calE_1]\\
&\le \Pr[(\Eopt)^c] + \Pr[\calE_1^c] +  N\max_{k \ge n_0}\Pr[\calE_2(k)^c \mid  \calE_1] \le 2\delta + Ne^{-C_4 S}.
\end{align*}
Translating to the non-simplified game, we have therefore established that
\begin{align*}
\Pr[\forall k: n_0 \le \nk(0,-1) \le N, \qup_{k,1}(-1,1) - \frac{1}{2} \ge 2\epsilon] \ge 1 - 2\delta + Ne^{-C_4 S}.
\end{align*}
Combining with the additional probability of error $\delta$ for the case $\nk(0,-1) \le n_0$ concludes the proof.

\subsection{Proof of Claim~\ref{claim:reward_overestimate}}

	We observe that conditioned on the vector $(n_{j}(x'))_{x' \in \Sminus, j \ge 1}$, the games at states $x$ and round $k$ are equivalent to $S$ independent two-arm bandit games with  $\nk(x)$ rounds.  Note moreover that $\Delta(x,k) =  \vup_{k,2}(x) - \frac{1}{2} \ge \bonusrew_k(x,1) + \rhat_k(x,1) - \frac{1}{2}$. Hence, restricting to a single state $x$ (and dropping the dependence on $x$ for simplicity), it suffices to show that for $k$ rounds of an appropriate two-arm bandit game with $a \in \{-1,1\}$ with empirical rewards $\rhat_k(a)$ and bonuses $\bonusrew_k(a)$, $R(-1) = 0$ and $R(1) \sim \frac{1}{2} + \frac{1}{4}\calD$, that 
	\begin{align*}
	\forall k \ge \const_1\log(S/\delta), \quad \Pr[\bonusrew_k(1) + \rhat_k(1) - \frac{1}{2} \ge \sqrt{\log(1/\delta)/k}] \ge \frac{1}{4}
	\end{align*}
	 where we have dropped the dependence on $x$ for simplicity. Throughout, we will also use the notation $\const_1,\const_2,\const_3$ to denote constants specific to the proof of Claim~\ref{claim:reward_overestimate}, and reserve $C_1,C_2$ for the constants in the claim statement.

	If $\delta \le 1/8$, then a standard argument shows that for some constant $\const_1$ (depending on $\cbonusup$), $n_k(-1) \le \const_1 \log(S/\delta)$. Indeed, define the event $\calE_0 := \{\forall k \ge 1: \bonusrew_k(1) + \rhat_k(1) \ge r(1) = \frac{1}{2}\}$; by assumption on our confidence intervals, complement of this event occurs with probability at most $\delta \le 1/8$. Note also that on $\calE_0$, since $R(-1) = 0$ with probability $1$, it holds that for any $j \le k$ with $n_j(-1) \ge \const_1 \log(S/\delta)$
	\begin{align*}
	\rhat_j(-1) + \bonusrew_j\xa  = \bonusrew_j\xa &\overset{(i)}{\le} \cbonusup\sqrt{\frac{\log(S n_j(-1)/\delta)}{n_j(-1)}} \\
	&\overset{(ii)}{\le}\frac{1}{2} = r(1) \le  \rhat_j(1) + \bonusrew_j(1),
	\end{align*}
	where in $(i)$ we used the definition of the confidence interval with $M \lesssim S$, and in $(ii)$ we used $n_j(-1) \ge \const_1 \log(S/\delta)$ for an appropriately tuned constant $\const_1$. Since $a_j := \argmax_{a} \rhat_j(a) + \bonusrew_j(a)$, we have $a_j = 1$. This implies that $n_k(-1)  \le \max_{j \ge 1} n_j(-1)\le \const_1 \log(M/\delta)$.

	Next, set $k_0 =  \const_1 \log(M/\delta)$. We wish to show that for $k \ge k_0$,
	\begin{align*}
	\rhat_{k}(1) + \bonusrew_k(1)\gtrsim \frac{\log(1/\delta)}{k}
	\end{align*}
	There are two technical challenges: first, the confidence interval $\bonusrew_k(1)$ might be nearly tight, so that we cannot show that with high probability, $\rhat_{k}(1) + \bonusrew_k(1) \gtrsim \bonusrew_k(1) $. Second, because the algorithm adaptively chooses to sample actions $a \in \{-1,1\}$, $\rhat_{k}(1)$ \emph{does not} have the distribution of $n_{k}(1)$ i.i.d. samples from $R(1)$. 

	We can get around this as follows. We can imagine all rewards sampled from action $1$ as being drawn at the start of the game, and constituting a sequence $R^{(1)}(1),R^{(2)}(1),\dots$ and so on. Then, $\rhat_k(1)$ is the average of the samples $1,\dots,\nk(1)$, where $\nk(1) \le k$. Therefore	
	\begin{align*}
	n_k(1)(\rhat_{k}(1) - \frac{1}{2}) = \sum_{i=1}^{n_k(1)}(R^{(i)}(1) - \frac{1}{2}) &= \sum_{i=1}^k (R^{(i)}(1) -\frac{1}{2}) -  \sum_{i= n_k(1)+1}^{k} (R^{(i)}(1) -\frac{1}{2}).\\
	&= \sum_{i=1}^k (R^{(i)}(1) -\frac{1}{2}) -  \sum_{i= k - n_k(-1) + 1}^{k} (R^{(i)}(1) -\frac{1}{2}),
	\end{align*}
	where the last line uses $n_k(1)  + n_k(-1) = k$. 

	Now consider the event $\calE_1(\delta) := \{n_k(-1) \le k_0\}$, where we recall $k_0 =  \const_1 \log(M/\delta)$ was our $1-\delta$-probability upper bound on $n_k(-1)$. On $\calE_1(\delta)$, $n_k(-1) = j$ for some $j \in \{0,1,\dots,k_0\}$, and we can lower bound the above expression by
	\begin{align*}
	\ge \sum_{i=1}^k (R^{(i)}(1) -\frac{1}{2}) -  \max_{j = 0,\dots,k_0}\sum_{i= k - j + 1}^{k} (R^{(i)}(1) -\frac{1}{2}).
	\end{align*}
	Observe now that we have lower bounded $n_k(1)(\rhat_{k}(1) - \frac{1}{2})$ in terms of quantities depending only on the i.i.d. reward sequence $(R^{(i)}(1))$, and \emph{not} on the quantities $\nk(-1),\nk(1)$. 

	Moreover, a standard maximal inequality implies that the following event $\calE_2(\delta)$ holds for an appropriate constant $\const_2$ with probability $1-\delta$:
	\begin{align}
	\calE_2(\delta) := \left\{\max_{j = 0,\dots,k_0}\sum_{i= k - j + 1}^{k} (R^{(i)}(1) -\frac{1}{2}) \le \const_2\sqrt{k_0 \log(1/\delta) }\right\}
	\end{align}
	Lastly,  since $R^{(i)}$ is symmetric, we have that the following event $\calE_3$ holds with probability $1/2$:
	\begin{align*}
	\calE_3 := \left\{\sum_{i=1}^k (R^{(i)}(1) - \frac{1}{2}) \ge 0\right\}.
	\end{align*}
	Hence, on $\calE_1(\delta) \cap \calE_2(\delta) \cap \calE_3$, 
	\begin{align*}
	n_k(1)(\rhat_{k}(1) - \frac{1}{2}) &\ge \underbracer{\ge 0}{\sum_{i=1}^k (R^{(i)}(1) -\frac{1}{2})} -  \max_{j = 0,\dots,k_0}\sum_{i= k - j + 1}^{k} (R^{(i)}(1) -\frac{1}{2})\\
	&\ge - \const_2\sqrt{\log(1/\delta) k_0}.
	\end{align*}
	If we further assume that $k \ge 2\const_1\log(M/\delta)$, then $n_k(-1) \le k_0 \le k/2$, so that $\calE_1(\delta)$ implies $n_k(1) \ge k/2$. Dividing both sides of the above by $k$ and bringing $1/k$ into the square root yields (again on $\calE_1(\delta) \cap \calE_2(\delta) \cap \calE_3$)
	\begin{align*}
	(\rhat_{k}(1) - \frac{1}{2}) &\ge  - \frac{\const_2}{2}\sqrt{\frac{\log(1/\eta)}{k} \cdot \frac{k_0}{k}}. \numberthis\label{eq:rhat_lb}
	\end{align*}

	Moreover, by the lower bound assumption on $\bonusrew$ and the fact that $R(1)$ has $\Omega(1)$ variance, there exists some constant $\const_3$ such that
	\begin{align*}
	\bonusrew(\nk(1)) \ge \cbonus\sqrt{\frac{\Var[R(x,a)] \log(1/\delta)}{n_k(1)}} \ge \const_3 \sqrt{\frac{\log(1/\delta)}{k}},
	\end{align*}
	where again we use $\nk(1) \le k$. Combining with~\eqref{eq:rhat_lb}, we have on $\calE_1(\delta) \cap \calE_2(\delta) \cap \calE_3$ that
	\begin{align*}
	(\rhat_{k}(1) - \frac{1}{2}) + \bonusrew(\nk(1)) &\ge  \const_3\sqrt{\frac{\log(1/\delta)}{k}} - \frac{\const_2}{2}\sqrt{\frac{\log(1/\delta)}{k} \cdot \frac{k_0}{k}}.
	\end{align*}
	Hence, if $k_0/k \le (\const_3/\const_2)^2$, or equivalently if $k \ge \const_1(\const_3/\const_2)^{-2} \log (M/\delta)$, then 
	\begin{align*}
	(\rhat_{k}(1) - \frac{1}{2}) + \bonusrew(\nk(1)) \ge \frac{\const_3}{2}\sqrt{\frac{\log(1/\delta)}{k}} \ge \frac{\const_3}{2}\sqrt{\frac{\log(1/\delta)}{2 \nk(1)}}
	\end{align*}
	on the event $\calE_1(\delta) \cap \calE_2(\delta) \cap \calE_3$. Lastly, for $\delta \le 1/8$, we note $\Pr[\calE_1(\delta) \cap \calE_2(\delta) \cap \calE_3] \ge \frac{1}{2} - 2\delta \ge 1/4$. Recalling our earlier condition $k \ge 2\const_1\log(M/\delta)$, the claim now holds with by setting the constant $C_1$ in the claim statement to be $\const_1\max\left\{2,(\const_3/\const_2)^{-2}\right\}$, and $C_2$ to be  $\frac{\const_3}{2\sqrt{2}}$.


\newpage
\section{Information Theoretic Lower Bound (Proposition~\ref{prop:info_th_lower_bound})\label{sec:inf_th_lb}}
In this section we construct give a proof of the information theoretic lower bound Proposition~\ref{prop:info_th_lower_bound}, as well as a non-asymptotic bound that holds even for non-uniformly good algorithms. 
\subsection{Construction of the hard instance} 
Our construction mirrors the lower bounds due to~\cite{dann2015sample}, but with specific and non-uniform gaps. We define $\calM$ as an MDP on state space $\states = [S+2]$, with actions $\actions = [A]$, and horizon $[H]$. We will first state the construction for $H \ge 2$, and then remark on the modification for $H = 1$ at the end of the section. For $a \in [A], x\in [S]$, we set
\begin{align*}
p(x' = S+1 | x,a) = \frac{3}{4} - \frac{2}{H - 1}\Delta_{x,a}, \quad p(x' = S+2 | x,a)= 1 - p(x' = S+1 | x,a).
\end{align*}
Furthermore, we set the initial state to have the distribution $x_1 \unifsim [S]$, and set
\begin{align*}
p(x' = S+1| x = S+1,a) = 1,~~p(x' = S+2| x = S+2,a) = 1 \forall a \in [A].
\end{align*}
Finally, the rewards are set deterministically as 
\begin{align*}
R(x,a) := \begin{cases} 0 & x \in [S]\\
0 & x \in \{S+1,S+2\}, a > 1 \\
1 & (x,a) = (S+1,1)  \\
\frac{1}{2} & (x,a) = (S+2,1)  \\
\end{cases}
\end{align*}
We may then verify that $\vst_{h}(S+1) = (H - h+1)$ and $\vst_h(S+1) = (h-H+1)/2$, which implies that that for $x \in [S]$,
\begin{align*}
\gaph(x,a) &= \left(\max_{a'}\sum_{x'}p(x'|x,a')\vst_{h+1}(x')\right) \\
&=  \max_{a'} (p(S+1|x,a') - p(S+1|x,a))(H- h) + (p(S+2|x,a') - p(S+2|x,a))\frac{(H- h)}{2}\\
&=  \max_{a'} (p(S+1|x,a') - p(S+1|x,a))(H- h) - (p(S+1|x,a') - p(S+1|x,a))\frac{(H- h)}{2} \\
&= \max_{a'} (p(S+1|x,a') - p(S+1|x,a))\frac{H- h}{2}\\
&= \frac{2\Delta_{x,a}}{(H-1)} \cdot \frac{H- h}{2},
\end{align*}
and in particular that $\gap_1\xa = \Delta_{x,a}$. For $H = 1$, the construction is modified so that $\states = [S]$, and 
\begin{align*}
R(x,a) \sim \Bern(\frac{3}{4} - \Delta_{x,a}), \text{ and }  x_1 \unifsim [S].
\end{align*}
Then, we see that $\gap_1\xa = \Delta_{x,a}$. In what follows, we will adress the $H \ge 2$ case; the case $H = 1$ will follow from similar, but simpler arguments. 

\subsection{Regret Lower Bound Decomposition} We can now lower bound the expected regret as 
\begin{align*}
\Exp^\calM[\regret_K] &:= \Exp^\calM[\sum_{k=1}^K \vst_0 - \vpik_0]\\
&:= \Exp[\sum_{k=1}^K \sum_{x}p(x_1 = x) \{\vst_1(x)-\vpik_1(x)\}]\\
&\overset{(i)}{\ge} \Exp^\calM\left[\sum_{k=1}^K \sum_{x}p(x_1 = x) \{\vst_1(x) - \qst_1(x,\pi_{k,1}(x))\}\right]\\
&=\Exp^\calM\left[\sum_{k=1}^K \sum_{x}p(x_1 = x) \gap_{1}(x,\pi_{k,1}(x))\right]\\
&=\Exp^\calM\left[\sum_{k=1}^K \sum_{x,a}p(x_1 = x) \I(\pi_{k,1}(x) = a) \gap_{1}\xa\right]\\
&= \sum_{x,a}\Exp^\calM[\nbarK(x,a)]\gap_{1}\xa\numberthis\label{eq:regret-lb-decomp},
\end{align*}
where inequality $(i)$ follows since $\vpik_1(x)= \trueq^{\pik}_1(x,\pi_{k,1}(x)) \le \qst_1(x,\pi_{k,1}(x))$.  We now show that for all sufficiently large $K \ge K_0(\calM)$, any uniformly correct algorithm must have 
\begin{multline}\label{eq:nbarK_lb}
\forall \xa: x \in [S], ~\gap_1\xa > 0, \forall K \ge K_0(\calM) \\
\quad \Exp^\calM[\nbarK((x,a)] \gtrsim (\frac{2}{H - 1}\Delta_{x,a})^{-2} \log K \gtrsim \frac{H^2}{\Delta_{x,a}^{-2}}\log K = \frac{H^2}{\gap_1\xa^2} \log K,
\end{multline}
 which concludes the proof since
\begin{align*}
\Exp^\calM[\regret_K] \gtrsim \sum_{x,a: \gap_1\xa > 0}\frac{H^2}{\gap_1\xa^2} \log K, \cdot \gap_{1}\xa = \sum_{x,a: \gap_1\xa > 0}\frac{H^2}{\gap_1\xa} \log K.
\end{align*}

We further note that this argument can also show that, for all $K$ sufficiently large and all $h \in [H-1]$
\begin{align}
\Exp^\calM[\regret_K] \gtrsim \sum_{x,a: \gap_h\xa > 0}\frac{(H-h)^2}{\gap_h\xa} \log K \label{eq:general_stage_gap}. 
\end{align}
as well. 

\subsection{Proof of Equation~\eqref{eq:nbarK_lb}}
\newcommand{\Deltil}{\bm{\Delta}}

Throughout, we fix a state $x \in [S]$, and an action $a: \gap_1\xa > 0$. We shall further introduce the shorhand
\begin{align}
\Deltil_{x,a}:= \frac{2\Delta_{x,a}}{H-1} \in (0,1/2),
\end{align} 
where the bound on $\Deltil_{x,a}$ follows from $\Delta_{x,a} \in (0,H/8)$.

To lower bound Equation~\eqref{eq:nbarK_lb}, we follow steps analogues to standard information theoretic lower bounds. Our exposition will follow~\cite{garivier2018explore}. First, we state a lemma which is the MDP analogue of 
\iftoggle{nips}{Equation (6) in~\cite{garivier2018explore}}{\citet[Equation (6)]{garivier2018explore}}. Its proof is analogous, and omitted for the sake of brevity:
\begin{lem} Let $\calM= (\states,\actions, H,r,p^{\calM},p_0,R^{\calM})$ and $\calM'= (\states,\actions, H,r,p^{\calM'},p_0,R^{\calM'})$ denote two episodic MDPs with the same state space $\states$, action space $\calA$ and horizon $h$, and initial state distribution $p_0$.
For any $\xa \in \states \times \actions$, let $\nu^{\calM}\xa$ denote the law of the joint distribution of $(X',R)$ where $X' \sim p^{\calM}(\cdot|x,a)$  and $R \sim R^{\calM}(x,a)$; define the law $\nu^{\calM}\xa$ analogously. Finally, fix a horizon $K \ge 1$, and let $\calF_K$ denote the filtration generated by all rollouts up to episode $K$. Then, for any $\calF_K$-measurable random variable $Z \in [0,1]$,
\begin{align*}
\kl(\Exp^{\calM}[Z],\Exp^{\calM'}[Z]) \le \sum_{x,a}\Exp^{\calM}[\nbarK(x,a)]\KL(\nu^{\calM}\xa,\nu^{\calM'}\xa),
\end{align*}
where $\kl(x,y) = x\log\frac{x}{y} + (1-x) \log \frac{1-x}{1-y}$ denotes the binary KL-divergence, and $\KL(\cdot,\cdot)$ denotes the KL-divergence between two probability laws. 
\end{lem}

We apply the above lemma as follows. For our fixed pair $\xa$, define an alternate $\calM'$ to be the MDP which coincides with $\calM$ except that 
\begin{align*}
p(x'|x,a) = \frac{3}{4} + \eta,~\eta = \min\{7/8,\frac{3}{4} + \Deltil_{x,a}\}
\end{align*}
By construction, $\calM$ and $\calM'$ differ only at their law at $\xa$. Thus, 
\begin{align*}
 \kl(\Exp^{\calM}[Z],\Exp^{\calM'}[Z]) \le \Exp^{\calM}[\nbarK(x,a)]\KL(\nu^{\calM}\xa,\nu^{\calM'}\xa).
\end{align*}
 We the following lower bound controls the $\KL$ divergence between the laws $\nu^{\calM}\xa,\nu^{\calM'}\xa$:
 \begin{claim} There exists a universal constant $c$ such that
\begin{align*}
\KL(\nu^{\calM}\xa,\nu^{\calM'}\xa) \lesssim c\Deltil_{x,a}^2.
\end{align*}
\end{claim}
\begin{proof}
At $\xa$, $R(x,a) = 0$ with probability under both $\calM,\calM'$. Moreover, recall that under $\calM$, $\xa$ transition to state $S+1$ with probability $\frac{3}{4} - \Deltil_{x,a}$, and to $S+2$ with probability $1 - (\frac{3}{4} - \Deltil_{x,a},)$. On the other hand, $\calM'$ transtion to $S+1$ with probability $\frac{3}{4} +\eta$, and $S+2$ with probability $1 - (\frac{3}{4}+\eta)$. Consequently both laws are equivalent to Bernoulli distributions with parameters $\frac{3}{4} - \Deltil_{x,a}$ and $\frac{3}{4} +\eta$, respectively. Since $\kl(x,y)$ is precisely $\KL(\Bern(x),\Bern(y))$ for $x,y \in (0,1)$,
\begin{align*}
\KL(\nu^{\calM}\xa,\nu^{\calM'}\xa) = \kl\left(\frac{3}{4} - \Deltil_{x,a}, \frac{3}{4} + \eta\right). 
\end{align*}

Lastly, set $x = \frac{3}{4} - \Deltil_{x,a}$ and $y =  \frac{3}{4} + \min\{\frac{7}{8},\Deltil_{x,a}\}$.We  $y-x \le 2\Deltil_{x,a}$, and by assumption on $\Deltil_{x,a} \le 1/2$, Thus, $1/4 \le x \le y \le 7/8$. Hence, a standard Taylor expansion (e.g. \citet[Lemma E.1]{simchowitz2016best}) shows that there exists a universal constant $c$ such that $\kl(x,y) \le \frac{c}{(x-y)^2} \le \frac{4c}{\Deltil_{x,a}^2}$, as needed.
\end{proof}
As a consequence, we see that for any $\calF_K$-measurable $Z \in [0,1]$, we find
\begin{align*}
\Exp^{\calM}[\nbarK(x,a)]\gtrsim \Deltil_{x,a}^{-2}\kl(\Exp^{\calM}[Z],\Exp^{\calM'}[Z]) \gtrsim \frac{H^2}{\Delta_{x,a}^2}\kl(\Exp^{\calM}[Z],\Exp^{\calM'}[Z]),
\end{align*}
where the last inequality uses that $\Deltil_{x,a} \lesssim \Delta_{x,a}/H$. 

To conclude, it suffices to exhibit a random variable $Z_K$ such that, for $K$ sufficiently large,
\begin{align*}
\kl(\Exp^{\calM}[Z],\Exp^{\calM'}[Z])  \gtrsim (1-\alpha) \log K.
\end{align*}
To this end, consider $Z_K = \frac{S\nbarK\xa}{K}$. Note that since $x$ is only visited with probability at most $1/S$ at stage $h=1$, and with probability $0$ for stages $h \ge 2$, we have  
\begin{align*}
\nbarK\xa =\sum_{k=1}^K \Exp^{\calM}[\Pr(x_1 = x)\I(\pi_{k,1}(x_1) = 1)] = \frac{1}{S}\sum_{k=1}^K \Exp^{\calM}[\I(\pi_{k,1}(x_1) = 1)] \le K/S,
\end{align*} 
which implies that,  $Z_K \in [0,1]$ with probability one. Moreover, note that by an argument similar to that of~\eqref{eq:regret-lb-decomp}, that under the MDP $\calM'$, 
\begin{align*}
\Exp^{\calM'}[\regret_K] \ge \eta \Exp^{\calM'}[\sum_{a' \ne a} \nbarK(x,a')] =  \eta (\frac{K}{S} - \Exp^{\calM'}\nbarK(x,a')]) = \frac{\eta K}{S}(1 - \Exp^{\calM'}[Z_K]).
\end{align*}
Hence, if $\Alg$ is $\alpha$-uniformly good, then there existsa  constant $C_{\calM'}$ such that
\begin{align*}
1 - \Exp^{\calM'}[Z_K] \le \frac{C_{\calM'}\eta}{S} K^{\alpha - 1}. 
\end{align*}
By the same token, there exists a constant $C_{\calM}$ such that
\begin{align*}
C_{\calM} K^{\alpha} \ge \Exp^{\calM}[\regret_K] \ge \gap_1\xa \Exp^{\calM}[\nbarK(x,a)] = \frac{K\gap_1\xa}{S} \Exp^{\calM}[Z_K]\gap_1\xa/S.
\end{align*}
which implies that $\Exp^{\calM}[Z_K] \le \frac{S C_{\calM} K^{\alpha - 1}}{\gap_1\xa}$. Furthermore, by \iftoggle{nips}{Inequality (11) in~\cite{garivier2018explore}}{\citet[Inequality (11)]{garivier2018explore}}, it holds that
\begin{align*}
\kl(x,y) \ge (1- x)\log \frac{1}{1-y} - \log 2
\end{align*}
which implies that for $K$ sufficiently large, 
\begin{align*}
\kl(\Exp^{\calM}[Z_K],\Exp^{\calM'}[Z_K]) &\ge (1 -\frac{S C_{\calM} K^{\alpha - 1}}{\gap_1\xa}) \left\{(1-\alpha)\log K - \log \frac{C_{\calM'}\eta}{2S}\right\} \gtrsim (1-\alpha)\log K.
\end{align*}

\end{document}